\documentclass{article} 

\usepackage{amssymb, amsmath, amsfonts, mathrsfs, mathtools}
\usepackage[nice]{nicefrac}
\usepackage{url}
\usepackage{enumitem}

\usepackage{etoolbox}
\newcommand{\arxiv}[1]{\iftoggle{colt}{}{#1}}
\newcommand{\colt}[1]{\iftoggle{colt}{#1}{}}
\newtoggle{colt}
\global\togglefalse{colt}

\colt{
	\usepackage{times}
}

\addtocontents{toc}{\protect\setcounter{tocdepth}{0}}

\arxiv{
\usepackage{geometry}
\geometry{left = 1in, right = 1in, top = 1in, bottom = 1.5in}
}

\newcommand{\EE}{\mathbb{E}}
\newcommand{\RR}{\mathbb{R}}
\newcommand{\calA}{\mathcal{A}}
\newcommand{\calX}{\mathcal{X}}
\newcommand{\calV}{\mathcal{V}}
\newcommand{\calF}{\mathcal{F}}
\newcommand{\calK}{\mathcal{K}}

\newcommand{\calG}{\mathcal{G}}
\newcommand{\calN}{\mathcal{N}}

\newcommand{\calH}{\mathcal{H}}
\newcommand{\calR}{\mathcal{R}}
\newcommand{\calS}{\mathcal{S}}

\newcommand{\calU}{\mathcal{U}}
\newcommand{\calY}{\mathcal{Y}}
\newcommand{\calQ}{\mathcal{Q}}
\newcommand{\tO}{\tilde{\mathcal{O}}}
\newcommand{\bs}{\mathbf{s}}
\newcommand{\by}{\mathbf{y}}
\newcommand{\bw}{\mathbf{w}}
\newcommand{\ba}{\mathbf{a}}
\newcommand{\bx}{\mathbf{x}}
\newcommand{\bp}{\mathbf{p}}
\newcommand{\bz}{\mathbf{z}}
\newcommand{\bv}{\mathbf{v}}
\newcommand{\bu}{\mathbf{u}}
\newcommand{\bq}{\mathbf{q}}

\newcommand{\ty}{\tilde{y}}
\newcommand{\tx}{\tilde{x}}

\newcommand{\bty}{\tilde{\by}}
\newcommand{\btx}{\tilde{\bx}}

\newcommand{\ber}{\mathrm{Ber}}
\newcommand{\unif}{\mathrm{Unif}}
\newcommand{\argmin}{\mathop{\arg\min}} 
\newcommand{\argmax}{\mathop{\arg\max}}
\newcommand{\bepsilon}{{\pmb{\epsilon}}}

\renewcommand{\epsilon}{\varepsilon}
\newcommand{\simiid}{\stackrel{\mathrm{i.i.d.}}{\sim}}
\newcommand{\TV}{D_{\mathrm{TV}}}
\newcommand{\hp}{\widehat{p}}
\newcommand{\bhp}{\mathbf{\hp}}

\newcommand{\frakd}{\mathfrak{D}}

\newcommand{\calNsq}{\mathcal{N}_{\mathsf{sq}}}
\newcommand{\calHsq}{\mathcal{H}_{\mathsf{sq}}}
\newcommand{\baru}{\bar{u}}
\newcommand{\bard}{{\bar{d}}}

\usepackage{makecell}

\usepackage{prettyref}
\newcommand{\pref}[1]{\prettyref{#1}}

\newcommand{\savehyperref}[2]{\texorpdfstring{\hyperref[#1]{#2}}{#2}}
\newrefformat{eq}{\savehyperref{#1}{Eq.~\textup{(\ref*{#1})}}}
\newrefformat{eqn}{\savehyperref{#1}{Equation~\ref*{#1}}}
\newrefformat{lem}{\savehyperref{#1}{Lemma~\ref*{#1}}}
\newrefformat{lemma}{\savehyperref{#1}{Lemma~\ref*{#1}}}
\newrefformat{def}{\savehyperref{#1}{Definition~\ref*{#1}}}
\newrefformat{line}{\savehyperref{#1}{Line~\ref*{#1}}}
\newrefformat{thm}{\savehyperref{#1}{Theorem~\ref*{#1}}}
\newrefformat{corr}{\savehyperref{#1}{Corollary~\ref*{#1}}}
\newrefformat{cor}{\savehyperref{#1}{Corollary~\ref*{#1}}}
\newrefformat{sec}{\savehyperref{#1}{Section~\ref*{#1}}}
\newrefformat{subsec}{\savehyperref{#1}{Section~\ref*{#1}}}
\newrefformat{app}{\savehyperref{#1}{Appendix~\ref*{#1}}}
\newrefformat{ass}{\savehyperref{#1}{Assumption~\ref*{#1}}}
\newrefformat{ex}{\savehyperref{#1}{Example~\ref*{#1}}}
\newrefformat{fig}{\savehyperref{#1}{Figure~\ref*{#1}}}
\newrefformat{alg}{\savehyperref{#1}{Algorithm~\ref*{#1}}}
\newrefformat{rem}{\savehyperref{#1}{Remark~\ref*{#1}}}
\newrefformat{conj}{\savehyperref{#1}{Conjecture~\ref*{#1}}}
\newrefformat{prop}{\savehyperref{#1}{Proposition~\ref*{#1}}}
\newrefformat{proto}{\savehyperref{#1}{Protocol~\ref*{#1}}}
\newrefformat{prob}{\savehyperref{#1}{Problem~\ref*{#1}}}
\newrefformat{claim}{\savehyperref{#1}{Claim~\ref*{#1}}}
\newrefformat{que}{\savehyperref{#1}{Question~\ref*{#1}}}
\newrefformat{op}{\savehyperref{#1}{Open Problem~\ref*{#1}}}
\newrefformat{fn}{\savehyperref{#1}{Footnote~\ref*{#1}}}
\newrefformat{tab}{\savehyperref{#1}{Table~\ref*{#1}}}
\newrefformat{fig}{\savehyperref{#1}{Figure~\ref*{#1}}}
\newrefformat{proc}{\savehyperref{#1}{Procedure~\ref*{#1}}}
\newrefformat{property}{\savehyperref{#1}{Property~\ref*{#1}}}

\arxiv{
\usepackage{xcolor}
\usepackage[colorlinks=true, linkcolor=blue, citecolor=blue]{hyperref}

\usepackage{amsthm}
\newtheorem{theorem}{Theorem}
\newtheorem{proposition}{Proposition}
\newtheorem{definition}{Definition}
\newtheorem{corollary}{Corollary}
\newtheorem{lemma}{Lemma}
\newtheorem{remark}{Remark}
}

\newcommand{\numberthis}{\addtocounter{equation}{1}\tag{\theequation}}

\allowdisplaybreaks

\usepackage{color-edits}
\addauthor{sr}{red}
\addauthor{zj}{green}

\title{On the Minimax Regret of Sequential Probability Assignment via Square-Root Entropy}
\colt{
	\coltauthor{\Name{Zeyu Jia} \Email{zyjia@mit.edu}\\
	\Name{Yury Polylanskiy} \Email{yp@mit.edu}\\
	\Name{Alexander Rakhlin} \Email{rakhlin@mit.edu}\\
	\addr Massachusetts Institute of Technology}
}

\arxiv{
	\author{
		Zeyu Jia
		\\
		\normalsize
        {\texttt{zyjia@mit.edu}}
		\and
		Yury Polyanskiy
		\\
		\normalsize
        {\texttt{yp@mit.edu}}
		\and
		Alexander Rakhlin
		\\
		\normalsize
        {\texttt{rakhlin@mit.edu}}
		\and
	}
}

\begin{document}
\maketitle

\begin{abstract}
	We study the problem of sequential probability assignment under  logarithmic loss, both with and without side information. Our objective is to analyze the \emph{minimax regret}---a notion extensively studied in the literature---in terms of geometric quantities, such as covering numbers and scale-sensitive dimensions.  We show that the minimax regret for the case of no side information (equivalently, the Shtarkov sum) can be upper bounded in terms of \emph{sequential square-root entropy}, a  notion closely related to Hellinger distance. For the problem of sequential probability assignment with side information, we develop both upper and lower bounds based on the aforementioned entropy. The lower bound matches the upper bound, up to log factors, for classes in the Donsker regime (according to our definition of entropy). 
\end{abstract}

\section{Introduction}

We consider the problem of sequential probability assignment under logarithmic loss. This framework has been studied extensively over the decades in fields such as information theory—where it relates to sequence compression—in gambling and sequential investment—where it is linked to wealth growth—and in online learning \cite{cesa2006prediction}. In its more recent incarnation, next-token prediction has emerged as a central challenge in training large language models, where the goal is to minimize the logarithmic loss (commonly referred to as cross-entropy loss) on nearly all available data.

Let us now describe the formal setup. On each round $t=1,\ldots, n$, the forecaster chooses a distribution $\hp_t$ over the finite alphabet $\calY$, observes $y_t\in\calY$, and incurs a loss of $-\log \hp_t(y_t)$. Over the $n$ rounds, the cumulative cost is $\sum_{t=1}^n -\log \hp_t(y_t)$. Since the distribution $\hp_t$ is chosen based on the previous outcomes $y_1,\ldots,y_{t-1}$, we associate $\hp_t$ with a conditional distribution $\hp(\cdot|y_1,\ldots,y_{t-1})$ and write the cumulative loss succinctly as $-\log \bhp(\by)$, where $\bhp$ is the corresponding joint distribution over sequences  $\by=(y_1,\ldots,y_n)$. 

The cumulative loss of the forecaster can be compared to that of the best ``expert'' in a class $\calQ \subseteq \Delta(\calY^n)$, each  identified with a joint probability distribution $\bq\in\calQ$. The forecaster aims to minimize regret
\begin{align}
    \label{eq:reg_def_non-contextual}
    \sum_{t=1}^n -\log \hp_t(y_t) - \inf_{\bq\in\calQ} \sum_{t=1}^n -\log q_t(y_t|y_1,\dots,y_{n-1}) = \sup_{\bq\in\calQ} \log\left(\frac{\bq(\by)}{\bhp(\by)}\right)
\end{align}
for \textit{any} sequence $y_1,\ldots,y_n$. As such, the problem falls under the umbrella of worst-case prediction (also known as individual sequence prediction).

In the more general problem of \textit{prediction with side information} (or, contextual prediction), the forecaster observes additional covariates $x_t\in\calX$ prior to making the probabilistic forecast $\hp_t\in\Delta(\calY)$ on round $t$. In this case, the regret expression becomes 
\begin{align}
\label{eq:regret_with_side_info}
    \sum_{t=1}^n -\log \hp_t(y_t) - \inf_{\bq\in\calQ} \sum_{t=1}^n -\log q_t(y_t|x_t)
\end{align}
and $\bq=(q_1,\ldots,q_n)$ is a sequence of conditional distributions $q_t:\calX\to \Delta(\calY)$. Of course, if $x_t=(y_1,\ldots,y_{t-1})$ and $\calX=\calY^*$, the problem reduces to the non-contextual version in \eqref{eq:reg_def_non-contextual}. 

The intrinsic difficulty of the prediction problem in the non-contextual case is
\begin{align*}
	\calR_n(\calQ) \coloneqq \inf_{\bhp\in \Delta(\calY^n)}\sup_{\by\in \calY^n} \sup_{\bq\in \calQ} \log\left(\frac{\bq(\by)}{\bhp(\by)}\right), \numberthis \label{eq: minimax-regret-Q}
\end{align*}
a quantity referred to as the \textit{worst-case redundancy}, or \textit{minimax regret}. A similar notion can be defined for the contextual version of the problem, when $(x_1,\ldots,x_n)$ also form an individual (i.e. arbitrary) sequence; however, for brevity, we defer this definition to \pref{sec: sequential-probability}.

The goal of this paper is to analyze the behavior of $\calR_n(\calQ)$, for both contextual and non-contextual cases, in terms of geometric concepts—such as covering numbers (or entropy) and scale-sensitive dimensions—analogous to how sample complexity is quantified in statistical learning through the complexity measures of the function class. This objective is not new; over the past several decades, numerous seminal ideas have been developed to address this question \cite{cover1974universal,rissanen1983universal,shtar1987universal,cover1991universal,merhav1993universal,merhav1998universal, cesa1999minimax}, and more recently in \cite{bilodeau2020tight,rakhlin2015sequential}, among many others.

In particular, the classical result of \cite{shtar1987universal} states that in the non-contextual case, $\calR_n(\calQ)$ has the following closed form:
\begin{align*}
	\calR_n(\calQ) = \log\sum_{\by\in \calY^n} \sup_{\bq\in \calQ} \bq(\by), \numberthis \label{eq: shtarkov-sum-Q}
\end{align*}
and the optimal strategy in  \pref{eq: minimax-regret-Q} is attained by the Shtarkov distribution $\bp^*(\by)\propto \sup_{\bq\in \calQ} \bq(\by)$, also known as the normalized maximum likelihood. While \eqref{eq: shtarkov-sum-Q} is more succinct than \eqref{eq: minimax-regret-Q}, it is still not amenable to analysis with standard tools, except for special cases \cite{cesa2006prediction}.

\subsection{Towards a General Result}

To the best of our knowledge, the first analysis of minimax regret for non-parametric (but iid) class $\calQ$ was proposed in~\cite{opper1999worst}, who presented an upper bound involving a Dudley integral. This work was extended to a general $\calQ$ in \cite{cesa1999minimax}, who observed that, owing to the equalizing property of the optimal strategy $\bp^*$, the Shtarkov sum \eqref{eq: shtarkov-sum-Q} can be expressed as
$
	\calR_n(\calQ) = \EE_{\by\sim \bp^*}\left[ \sup_{\bq\in\calQ} \log \frac{\bq(\by)}{\bp^*(\by)}\right]
$
and further upper bounded by the expected supremum of a subgaussian process indexed by the collection $\calQ$, much in the spirit of the empirical process theory approach in statistics and learning theory \cite{geer2000empirical}. Notably, the process was shown to be subgaussian with respect to a pseudometric $d(f,g)=\left(\sum_{t=1}^n \max_{y_{1:t}} (\log f(y_t|y_{1:t-1})-\log g(y_t|y_{1:t-1}))^2\right)^{1/2},$
where $y_{1:t} :=(y_1,\ldots,y_t)$.  \cite{cesa1999minimax} subsequently developed a Dudley-integral-style bound for the Shtarkov sum; however, the induced covering numbers are difficult to control due to the unbounded nature of the logarithm for small values, ultimately leading to generally suboptimal upper bounds on $\calR_n(\calQ)$ as a consequence of clipping probabilities away from $0$. Remarkably, retaining the logarithm in the definition of the pseudometric yielded an interesting consequence: the main upper bound in \cite[Theorem 3]{cesa1999minimax} assumes a form typical of bounds obtained via localized or \textit{offset Rademacher} complexities, as encountered in square loss regression \cite{geer2000empirical,liang2015learning,mourtada2023universal}. 

Several subsequent attempts have been made to derive tighter upper bounds on minimax regret, with the focus shifting toward the contextual case. In an effort to obtain, as an upper bound on minimax regret, a stochastic process that is subgaussian with respect to a pseudometric on the \textit{values} of the distributions $q_t$ rather than on their logarithms, \cite{rakhlin2015sequential} employed a first-order expansion of the logarithmic loss. This approach upper bounded the minimax regret by a version of sequential offset Rademacher complexity. Unfortunately, despite their efforts to tame the explosive nature of the derivatives, the authors were unable to derive upper bounds on the offset process that were independent of the clipping range, even in the finite case \cite[Lemma 2]{rakhlin2015sequential}. 

The important work of \cite{bilodeau2020tight} leveraged the self-concordance properties of the logarithm to upper bound the minimax value in the contextual case by an offset-like process that offered clear advantages over earlier approaches. In particular, for a finite collection, the resulting process could be controlled without resorting to clipping. However, the process did not exhibit a subgaussian nature, which prevented the authors from employing chaining arguments. This issue arises from the presence of linear terms of the form $q_t(y_t)/p_t(y_t)$, which are small in expectation over $y_t\sim p_t$ (thus permitting single-scale discretization) but become uncontrolled when the ratio is squared.

The Hellinger distance has long been recognized as a convenient metric on the space of distributions \cite{lecam73,haussler1997mutual,yang1999information,geer2000empirical, bilodeau2023minimax}. In particular, as an $\ell_2$-distance between the square roots of distributions, it offers the possibility of combining the benefits of offset-based analysis with those of multi-scale chaining. This is the approach we adopt in this paper. Specifically, we employ an approximation $\log x\le \zeta(x) - \frac{1}{4\log (n|\calY|)}\cdot \zeta(x)^2,$
which holds over an appropriate range of $x$, and where $\zeta(x)$ behaves as $2(\sqrt{x}-1)$ for $x\leq 1$. Applied, roughly speaking, to $x=q_t(y_t)/p_t(y_t)$, this inequality allows us to leverage symmetrization and chaining techniques while also capitalizing on the fast rates provided by the offset sequential Rademacher process. Our approach, therefore, appears to resolve the technical issues encountered by the various techniques, starting with \cite{cesa1999minimax}, at least in the so-called Donsker regime (with respect to our entropy definition), where chaining provides an advantage.

To demonstrate the sharpness of our results---again, in the Donsker regime---for the contextual version of the problem, we develop new lower bound techniques that build upon \cite[Lemma 10]{rakhlin2014online}. In particular, we introduce a novel sequential scale-sensitive dimension, prove a combinatorial result that controls the size of the sequential cover in terms of this dimension, and employ this new notion to derive nearly matching lower bounds for any $\calQ$ (in the contextual case). This approach significantly strengthens the earlier work, which only guaranteed lower bounds for a modified function class. Our techniques will be presented in full detail in the companion paper \cite{JiaPolRak25modified}.

We now summarize our contributions.

\subsection{Summary of Main Results}

We study minimax regret in both non-contextual (\pref{sec: shtarkov}) and contextual (\pref{sec: sequential-probability}) settings. Our results below are stated with respect to sequential square-root entropy, $\calHsq(\calQ, \alpha, n)$, defined formally in \pref{sec: preliminaries}.

\paragraph{An upper bound on minimax regret for the non-contextual case:} For any class of distributions $\calQ\subseteq \Delta(\calY^n)$, with sequential square-root entropy $\calHsq(\calQ, \alpha, n)$ at scale $\alpha$, the minimax regret \eqref{eq: minimax-regret-Q} (and, hence, the Shtarkov sum \eqref{eq: shtarkov-sum-Q}) has the following upper bound:
\begin{align*}
	\calR_n(\calQ)\lesssim 1 + \inf_{\gamma > \delta > 0}\left\{n\delta\sqrt{|\calY|} + \sqrt{n|\calY|}\int_{\delta}^\gamma\sqrt{\calHsq(\calQ, \alpha, n)}d\alpha + \calHsq(\calQ, \gamma, n)\right\},
\end{align*}
where we use $\lesssim$ to hide constants and $\log(n|\calY|)$ factors.

\paragraph{Tight characterization of contextual sequential probability assignment:} 
Focusing on the binary alphabets for simplicity, we provide both an upper bound and a lower bound for the minimax regret, defined below in \eqref{eq: def-minimax-regret} and again denoted here as $\calR_n(\calQ)$ .
The following upper bound holds in terms of sequential square-root entropy:
\begin{align*}
	\calR_n(\calQ)\lesssim 1 + \inf_{\gamma > \delta > 0}\left\{n\delta + \sqrt{n}\int_{\delta}^\gamma\sqrt{\calHsq(\calQ, \alpha, n)}d\alpha + \calHsq(\calQ, \gamma, n)\right\}.
\end{align*} 
According to this upper bound, for any nonparametric function class $\calQ$ which satisfies $\calHsq(\calQ, \alpha, n) = \mathcal{O}\left(\alpha^{-p}\right)$, the minimax regret  is upper bounded as 
\begin{align*} 
	\calR_{n}(\calQ) = \begin{cases}
		\tilde{\mathcal{O}}\left(n^{\frac{p}{p+2}}\right) &\quad \text{if } 0 < p\le 2,\\
		\tilde{\mathcal{O}}\left(n^{\frac{p-1}{p}}\right) &\quad \text{if } p > 2.
	\end{cases} \numberthis \label{eq: R-n-nonparametric}
\end{align*}

In addition, we establish a lower bound demonstrating the tightness of \eqref{eq: R-n-nonparametric} for $0\le p\le 2$. Hence, for nonparametric classes with parameter $p\le 2$, our results offer a tight characterization of the minimax regret in terms of the sequential square-root entropy. Our upper bound further yields an $\tilde{\mathcal{O}}\left(\sqrt{n}\right)$ bound for the Hilbert ball problem, thereby answering a question posed in \cite{rakhlin2015sequential}.

Our contributions are also technical. The proof of the upper bound introduces a novel approach to analyzing the expectation of the offset Rademacher process, enabling us to handle cases with unbounded coefficients. We adopt a chaining argument alongside the analysis of offset Rademacher processes in our proof. On the lower bound side, as mentioned, our techniques involve a new definition of a scale-sensitive dimension and a novel argument for lower-bounding the sequential offset Rademacher complexity that is applicable beyond this paper.

Overall, our results largely resolve the open problem stated in \cite{rakhlin2015sequential} by tightly characterizing the minimax regret of contextual probability assignment for any class of conditional probability distributions in terms of entropic quantities, at least in the Donsker regime (according to the our definition of entropy).

\subsection{Notation}\label{sec: preliminaries}

Given $\bq\in \Delta(\calY^n)$, we write $q_t(y_t\mid \by) = q_t(y_t\mid y_{1:t-1})$ to denote the conditional probability for any length-$n$ sequence $\by = (y_1, \cdots, y_n)\in \calY^n$. 
A $\{0,1\}$-path $\bw$ of depth $n$ is a tuple $(w_1, \ldots, w_n)\in\{0,1\}^n$. For any set $\calX$, a depth-$n$ $\calX$-valued binary tree (or, simply, `a tree') $\bx$ has $2^n-1$ nodes, where each node takes value in $\calX$. Formally, $\bx=(x_1,\ldots,x_n)$ with $x_i:\{0,1\}^{i-1}\to\calX$. We write $x_i(\bw)=x_i(w_{1:i-1})$ for brevity. For a depth-$n$ $\calX$-valued tree $\bx$ and function $f: \calX\to [0, 1]$, we use $f\circ \bx$ to denote the depth-$n$ $[0, 1]$-valued tree whose value at depth $t$ on path $\bw$ equals to $f(x_t(\bw))$. We write $\calF\circ \bx = \{f\circ \bx: f\in \calF\}$.

Additionally, we use the following asymptotic notation: for positive sequence $\{a_n\}$ and $\{b_n\}$ (or functions $f(\alpha), g(\alpha): (0, 1)\to \mathbb{R}_+$, we use $a_n = \mathcal{O}(b_n)$ (or $f(\alpha) = \mathcal{O}(g(\alpha))$) if there exists a positive constant $c$ such that $a_n\le c\cdot b_n$ for any $n$ (or $f(\alpha)\le c\cdot g(\alpha)$ for any $\alpha$), and we use $a_n = \tO(b_n)$ if there exists a positive constant $c$ and positive integer $r$ such that $a_n\le c\cdot(\log n)^r \cdot b_n$ (or $f(\alpha)\le c\cdot (\log(1/\alpha))^r\cdot g(\alpha)$ ). We use notation $a_n = \Omega(b_n)$ (or $f(\alpha) = \Omega(g(\alpha))$ if and only if $b_n = \mathcal{O}(a_n)$ (or $g(\alpha) = \mathcal{O}(f(\alpha))$, and $\tilde{\Omega}$ is defined similarly. The notation $a_n = \Theta(b_n)$ (or $f(\alpha) = \Theta(g(\alpha))$) is used if and only if both $a_n = \mathcal{O}(b_n)$ and $a_n = \Omega(b_n)$ hold (or both $f(\alpha) = \mathcal{O}(g(\alpha))$ and $f(\alpha) = \Omega(g(\alpha))$ hold). The notation $\tilde{\Theta}$ is defined similarly.

\subsection{Organization}
In \pref{sec: shtarkov}, we present our results in upper bounding the Shtarkov sum using sequential square-root entropy. In \pref{sec: sequential-probability}, we revisit the problem of contextual sequential probability assignment, and provide upper and lower bounds for the minimax regret in terms of sequential square-root entropy. Finally, in \pref{sec: proof-sketch-main} we provide a proof sketch of our main result, \pref{thm: general-upper-bound}. All the technical proofs are deferred to the appendix.

\section{Upper Bound for Shtarkov Sum through Sequential Square-Root Covering}\label{sec: shtarkov}

In this section, we upper bound the minimax regret \pref{eq: minimax-regret-Q} or Shtarkov sum \pref{eq: shtarkov-sum-Q} in terms of the $\ell_\infty$ sequential square-root covering defined as follows. \begin{definition}[sequential square-root cover and entropy]\label{def: sequential-covering}
	Let $\calY$ be a finite alphabet. For a class of joint distributions $\calQ$ over $\calY^n$, we say that a finite class $\calV$ of joint distributions over $\calY^n$ is a sequential square-root cover (in the $\ell_\infty$ sense) of $\calQ$ at scale $\alpha$ if 
	\begin{align}
    \label{eq:sqrt_cover_def}
    \sup_{\bq\in \calQ} \max_{\bw\in \calY^n}\min_{\bv\in \calV} \max_{t\in [n]} \max_{y\in \calY} \left|\sqrt{q_t(y \mid \bw)} - \sqrt{v_t(y \mid \bw)}\right|\le \alpha.
    \end{align}
	We use $\calNsq(\calQ, \alpha, n)$ to denote the size of the smallest cover of class $\calQ$, and we use $\calHsq(\calQ, \alpha, n) = \log \calNsq(\calQ, \alpha, n)$ to denote the sequential square-root entropy of $\calQ$.
\end{definition}
In words, the requirement placed on $\calV$ is that for any joint distribution $\bq\in\calQ$ and any sequence $\bw\in\calY^n$, there exists a ``representative'' joint distribution $\bv$ in $\calV$ that is close to $\bq$ in terms of the difference of square roots of the conditional probabilities $\bq$ and $\bv$ assign to any outcome $y$, uniformly for all time steps. 

In this definition, $\ell_\infty$ refers to the maximum over $t\in[n]$, which is consistent with prior uses of such sequential and empirical notions of a cover. We also remark that $\max_{y\in \calY} \left|\sqrt{q_t(y \mid \bw)} - \sqrt{v_t(y \mid \bw)}\right|$ is within $\sqrt{|\calY|}$ of the Hellinger distance between these two conditional distributions (which is the $\ell_2$ version with respect to the $y\in\calY$). If scaling with $|\calY|$ is not of interest, we can instead think of the sequential square-root cover as a \textit{sequential Hellinger cover}.

\begin{theorem}\label{thm: general-upper-bound}
	For any $n\ge 7$ and class $\calQ\subseteq \Delta(\calY^n)$, we have 
	\begin{align*}
		\calR_n(\calQ) = \tilde{\mathcal{O}}\left(1 + \inf_{\gamma > \delta > 0}\left\{n\delta\sqrt{|\calY|} + \sqrt{n|\calY|}\int_{\delta}^\gamma\sqrt{\calHsq(\calQ, \alpha, n)}d\alpha + \calHsq(\calQ, \gamma, n)\right\}\right),
	\end{align*}
	where $\tilde{\mathcal{O}}$ hides constants and logarithmic factors of $n$ and $|\calY|$. 
\end{theorem}
The proof of \pref{thm: general-upper-bound} is deferred to \pref{sec: proof-upper-bound-app}, and we provide a sketch of the proof in \pref{sec: proof-sketch-main}. The theorem immediately implies an upper bound on $\calR_n(\calQ)$ whenever the sequential square-root entropy scales with $\alpha^{-p}$, as shown in the following corollary. 
\begin{corollary}
	When $\calHsq(\calQ, \alpha, n) = \tO\left(\alpha^{-p}\right)$ for some $p\ge 0$, it holds that
	$$\calR_n(\calQ) = \begin{cases}
		\tO\left(n^{\frac{p}{p+2}}\right) &\quad \text{if }p \le 2,\\
		\tO\left(n^{\frac{p-1}{p}}\right) &\quad \text{if }p > 2.
	\end{cases}$$
\end{corollary} 

\subsection{Comparison with Previous Results} 
We compare our results with \cite{cesa1999minimax, cesa2006prediction}, which also provide an upper bound on the minimax regret using entropy. Taking 
\begin{equation}\label{eq: CBL-d}
d(f,g)=\left(\frac{1}{n}\sum_{t=1}^n \max_{y_{1:n}} (\log f(y_t|y_{1:t-1})-\log g(y_t|y_{1:t-1}))^2\right)^{1/2},
\end{equation}
as the (pseudo)metric, the authors define a notion of entropy $\calH_{\sf log}(\calQ, \alpha, n)$ as the logarithm of the size of the smallest covering at scale $\alpha$ under $d$. \cite{cesa1999minimax} establish that 
\begin{equation}\label{eq: regret-cbl}
    \calR_n(\calQ)\lesssim \inf_{\gamma > 0}\left\{ \sqrt{n}\int_{0}^\gamma\sqrt{\calH_{\sf log}(\calF, \epsilon, n)}d\epsilon + \calH_{\sf log}(\calF, \gamma, n)\right\}.
\end{equation}
The form of the bound appears frequently in the literature on prediction with square loss, in both fixed design regression and online regression. Writing the definition of the above covering notion in the form of \eqref{eq:sqrt_cover_def}, we have
\begin{align}
    \label{eq:CBL_cover_def}
    \sup_{\bq\in \calQ} \min_{\bv\in \calV} 
    \max_{\bw\in \calY^n}
    \max_{t\in [n]} \max_{y\in \calY} \left|\log q_t(y \mid \bw) - \log v_t(y \mid \bw)\right|\le \alpha.
\end{align}
with the only difference that we opted for $\max_{t\in[n]}$ instead of the $\ell_2$ version employed above. Modulo this difference, the requirement \eqref{eq:CBL_cover_def} is clearly more stringent than \eqref{eq:sqrt_cover_def} as the representative $\bv$ has to be chosen irrespective of the data $\bw$, making the notion of the cover similar to the (often prohibitively large) $\sup$-norm cover. The line of work on sequential complexities addresses this shortcoming via symmetrization, an approach we also take in this paper. Finally, we note that 
    $\sup_{y\in \calY} |\sqrt{p(y)} - \sqrt{q(y)}| \leq \sup_{y\in \calY} |\log p(y) - \log q(y)|$
and, thus, we expect $\calH_{\log}$ to be larger (and often much larger) than $\calHsq$.

Note that while the sequential square-root entropy is an improvement over the entropy in \cite{cesa1999minimax}, there are still interesting distribution classes where it does not yield the correct bound. For example, consider the renewal process class (definition is included in \pref{sec: renewal-app}). It is known from \cite{csiszar1996redundancy} that minimax regret is $\Theta(\sqrt{n})$. In \pref{sec: renewal-app}, however, we show that the sequential square-root entropy is always lower bounded by $\Omega(n)$.

\section{Binary Contextual Sequential Probability Assignment}\label{sec: sequential-probability}

In this section, we connect the problem of contextual sequential probability assignment to the non-contextual case discussed in the previous section. Application of the general bound of Theorem~\ref{thm: general-upper-bound} will then lead to the main results of our paper.

 For simplicity of presentation, and to make our results more directly comparable to prior work, we focus on the binary alphabet $\calY=\{0,1\}$. With some abuse of notation we let $\hp_t \in [0, 1]$ denote the probability of the outcome $1$. The loss incurred on round $t$ after making the prediction $\hp_t$ can thus be written as 
\begin{equation}\label{eq: logloss}
    \ell(\hp_t, y_t) := -y_t\log \hp_t - (1-y_t)\log (1-\hp_t).
\end{equation}
Similarly, we re-parametrize conditional distributions $\calQ$ by instead working with a class $\calF$ of experts, mapping $\calX$ to $[0,1]$. This re-parametrization is consistent with other prior works. With this notation, the cumulative loss of an expert $f$ is $\sum_{t=1}^n \ell(f(x_t), y_t)$, and regret is defined (in a form that is more explicit than \eqref{eq:regret_with_side_info}) as
\begin{equation}\label{eq: def-R-n}
	\calR_n(\calF, \hp_{1:n}, x_{1:n}, y_{1:n})\coloneqq \sum_{t=1}^n \ell(\hp_t, y_t) - \inf_{f\in \calF} \sum_{t=1}^n \ell(f(x_t), y_t).
\end{equation}
Recall that $x_t$ may depend arbitrarily on the history 
$$\calH_t=\{x_1, \hp_1, y_1, \ldots, x_{t-1}, \hp_{t-1}, y_{t-1}\},$$
and $y_t$ may depend arbitrarily on $\calH_t, x_t$ and $\hp_t$. Based on the order of making predictions and observing outcomes, we define the minimax regret as
\begin{equation}\label{eq: def-minimax-regret}
	\calR_{n}(\calF) = \sup_{x_{1}\in\calX}\inf_{\hp_{1}\in [0, 1]}\sup_{y_{1}\in \{0, 1\}}\cdots\sup_{x_{n}\in\calX}\inf_{\hp_{n}\in [0, 1]}\sup_{y_{n}\in \{0, 1\}}\calR(\calF, \hp_{1:n}, x_{1:n}, y_{1:n}),
\end{equation}
or, more succinctly, as
$$\calR_{n}(\calF) = \left\{\sup_{x_{t}\in\calX}\inf_{\hp_{t}\in [0, 1]}\sup_{y_{t}\in \{0, 1\}}\right\}_{t=1}^n\calR(\calF, \hp_{1:n}, x_{1:n}, y_{1:n}).$$
Here, the curly braces indicated a repeated application of the operators.

The above expression indeed matches the aforementioned dependencies. To make the connection to the previous section, we start with the following observation. Consider an adversary that is not allowed to adapt the sequence of $x$'s to the past predictions made by the forecaster and instead has to fix ahead of time a strategy for choosing $x$'s based only on the outcomes $y$'s; this is equivalent to fixing an $\calX$-valued tree $\bx=(x_1,\ldots,x_n)$ and presenting $x_t(y_{1:t-1})$ to the forecaster at the beginning of round $t$. Notably, the sequence of $y_t$'s can still be adapted to the predictions of the forecaster. The following lemma states that such an adversary is just as powerful as the fully-adaptive one, even if the forecaster knows the strategy $\bx$:
\newcommand{\multiminimax}[1]{\left\langle #1\right\rangle}
\begin{lemma}\label{lem: tree-transductive}
    For any $\ell:\calY\times\calY\to\RR$ which is convex in its first argument, and for any $\phi:\calY^n\times \calX^n\to\RR$, 
    \begin{align*} 
        &
        \left\{
            \sup_{x_t} \inf_{\hp_t}\sup_{y_t}
            \right\}_{t=1}^n\left[\sum_{t=1}^n \ell(\hp_t, y_t) - \phi(y_{1:n},x_{1:n}) \right]\\
        & = \sup_{\bx} \left\{\inf_{\hp_t}\sup_{y_t}\right\}_{t=1}^n \left[\sum_{t=1}^n \ell(\hp_t, y_t) - \phi(y_{1:n},x_1,x_2(y_1),\ldots,x_n(y_{1:n-1}))\right]
    \end{align*}
    where the supremum in the last expression is over all depth-$n$ $\calX$-valued trees $\bx$.
    In particular, for $\phi(y_{1:n},x_{1:n})=\inf_{f\in \calF}\sum_{t=1}^n \ell(f(x_t), y_t)$ and logarithmic loss discussed in this paper, $$\calR_n(\calF) = \sup_{\bx}\calR_n(\calQ_{\bx}) $$ 
    for the set $\calQ_{\bx}=\calF \circ \bx = \{f\circ \bx: f\in \calF\}$ with $(f\circ \bx) (\by) = \prod_{t=1}^n \{\mathbb{I}[y_t = 1]f(x_t(\by)) + \mathbb{I}[y_t = 0](1 - f(x_t(\by))\}$ for any $\by\in \{0, 1\}^n$.
\end{lemma}
For the logarithmic loss, this result was proved in \cite[Theorem 3.2]{liu2024sequential}. The proof of \pref{lem: tree-transductive} is deferred to \pref{sec: upper-bound-proof}.

The importance of this proposition is two-fold. First, it shows a possibly counter-intuitive property that regret is unchanged if the adversary's $x$'s are not allowed to depend on the actions of the forecaster, but only on the $y$'s. In other words, there exists a best possible adversarial tree $\bf x$ that saturates regret for all possible strategies of the forecaster.\footnote{Note that the optimal learning algorithm for this $\bf x$ tree is not guaranteed to be optimal for the actual problem ~\eqref{eq: def-minimax-regret}.} Second, note that when the tree $\bx$ is fixed ahead of time, the resulting problem corresponds to the problem discussed in Theorem~\ref{thm: general-upper-bound} with $\calQ_{\bx} = \calF \circ \bf x$.

\subsection{Upper Bound with Sequential Square-Root Entropy}\label{sec: upper-bound}

We now repeat the definition \pref{def: sequential-covering}, adapting it to the case of binary alphabet and the real-valued re-parametrization of probabilities:
\begin{definition}[sequential square-root cover and entropy]\label{def: sequential-covering-f}
	Suppose $\calV$ and $\calA$ are two sets of $[0, 1]$-valued binary trees of depth $n$. We say $\calV$ is a sequential square-root cover (in the  $\ell_\infty$ sense) of $\calA$ at scale $\alpha$ if
	$$\max_{\by\in \{0, 1\}^n} \sup_{\ba\in \mathcal{A}} \inf_{\bv\in\calV} \max_{t\in [n]} \max\left\{\left|\sqrt{a_t(\by)} - \sqrt{v_t(\by)}\right|, \left|\sqrt{1 - a_t(\by)} - \sqrt{1 - v_t(\by)}\right|\right\}\le \alpha.$$
	We use $\calNsq(\calA, \alpha, n)$ to denote the size of the smallest sequential square-root cover at scale $\alpha$. For any set $\calX$ and function class $\calF\subseteq\{f:\calX\to [0, 1]\}$, the sequential square-root entropy of function class $\calF$ on an $\calX$-valued tree $\bx$ (of depth $n$) at scale $\alpha$ is defined as 
$$\calHsq(\calF, \alpha, n, \bx) = \log\calNsq(\calF\circ \bx, \alpha, n).$$
\end{definition}
With the above definition of sequential square-root entropy, we have the following theorem:
\begin{theorem}\label{thm: upper-bound}
	For any $\calX$ and function class $\calF\in [0, 1]^{\calX}$, we have 
	\begin{align*}
		\calR_n(\calF) = \tO\left(\sup_{\bx}\left\{1 + \inf_{\gamma > \delta > 0}\left\{n\delta + \sqrt{n}\int_{\delta}^\gamma\sqrt{\calHsq(\calF, \alpha, n, \bx)}d\alpha + \calHsq(\calF, \gamma, n, \bx)\right\}\right\}\right),
	\end{align*}
	where the supremum is over all depth-$n$ $\calX$-valued trees $\bx$.
\end{theorem}
This theorem has the following direct corollary, which provides explicit upper bounds on the minimax regret whenever the growth of sequential square-root entropy at scale $\alpha$ is bounded by $\tO(\alpha^{-p})$ for some $p\ge 0$.
\begin{corollary}\label{corr: entropy}
	For any function class $\calF\subseteq\{f:\calX\to [0, 1]\}$, suppose the sequential square-root entropy $\calHsq(\calF, \alpha, n, \bx)$ at scale $\alpha$ satisfies $\sup_{\bx}\calHsq(\calF, \alpha, n, \bx) = \tO\left(\alpha^{-p}\right)$ for some $p\ge 0$. The minimax regret $\calR_n(\calF)$ is upper bounded by
	\begin{align*} 
		\calR_{n}(\calF) = \begin{cases}
			\tilde{\mathcal{O}}\left(n^{\frac{p}{p+2}}\right) &\quad \text{if } 0 \le p\le 2,\\
			\tilde{\mathcal{O}}\left(n^{\frac{p-1}{p}}\right) &\quad \text{if }p > 2.
		\end{cases}
	\end{align*}
\end{corollary}
The proofs of \pref{thm: upper-bound} and \pref{corr: entropy} are deferred to \pref{sec: upper-bound-proof}.

\subsection{Comparison with Previous Upper Bounds}\label{sec: entropy-connection}
We compare our results to those of \cite{rakhlin2015sequential}  and \cite{bilodeau2020tight}. These two works use the sequential entropy $\calH_{\infty}(\calF, \alpha, n, \bx)$, defined as 
\begin{equation}\label{eq: def-H-infty-F}
	\calH_\infty(\calF, \alpha, n, \bx) = \log \calN_\infty(\calF\circ \bx, \alpha, n),
\end{equation}
with $\calN_\infty(\calF\circ \bx, \alpha, n)$ being the size of smallest cover $V$ such that 
$$\max_{\by\in \{0, 1\}^n} \sup_{f\in\calF}\min_{\bv\in V}\max_{t\in [n]}\left|f(x_t(\by)) - v_t(\by)\right|\le \alpha.$$

With proper choice of parameters, our results can  recover the upper bounds in \cite[Theorem 1 and Theorem 4]{rakhlin2015sequential}. This follows from the next result relating sequential square-root entropy and the sequential entropy defined above.
\begin{proposition}\label{prop: entropy-connection-1}
	Suppose $\delta > 0$. For any $\calF\subseteq\{f:\calX\to [\delta, 1-\delta]\}$, $\alpha>0$, and depth-$n$ $\calX$-valued tree $\bx$, 
	$$\calHsq(\calF, \alpha /\sqrt{\delta}, n, \bx)\le \calH_{\infty}(\calF, \alpha, n, \bx).$$
\end{proposition}

For nonparametric class $\calF$ which satisfies $\sup_{\bx}\calH_\infty(\calF, \alpha, n, \bx)\asymp \alpha^{-q}$ for some $q > 0$, \cite[Theorem 2]{bilodeau2020tight} proves the following upper bound for the minimax regret: 
\begin{equation}\label{eq: r-n-p-p+1}
	\calR_n(\calF) = \mathcal{O}\left(n^{\frac{q}{q+1}}\right).
\end{equation}
\pref{corr: entropy} recovers this result for $0 < q \le 1$, up to logarithmic factors, via the following proposition.
\begin{proposition}\label{prop: entropy-connection}
	For any function class $\calF\subseteq\{f:\calX\to [0, 1]\}$, $\alpha > 0$, and depth-$n$ $\calX$-valued tree $\bx$, we have 
	$$\calHsq(\calF, 2\alpha, n, \bx)\le  \calH_\infty(\calF, \alpha^2, n, \bx).$$
\end{proposition}
The proofs of \pref{prop: entropy-connection-1} and \pref{prop: entropy-connection} are deferred to \pref{sec: entropy-connection-proof}.

\subsection{Lower Bound with Sequential Square-Root Entropy}\label{sec: lower-bound}
In this section, we provide lower bounds on the minimax regret $\calR_n(\calF)$ defined in \pref{eq: def-minimax-regret} via sequential square-root entropy. 

\begin{theorem}\label{thm: lower-bound-entropy}
	Suppose function class $\calF\subseteq [0, 1]^{\calX}$ satisfies 
	\begin{align}
    \label{eq:entropy_poly_growth}
	    \sup_{\bx}\calHsq(\calF, \alpha, n, \bx) = \tilde{\Omega}\left(\alpha^{-p}\right),
	\end{align}
	where the supremum is over all depth-$n$ $\calX$-valued trees. Then we have the following lower bound on the minimax regret:
	$$\calR_n(\calF) = \tilde{\Omega}\left(n^{\frac{p}{p+2}}\right).$$
\end{theorem}
The proof of \pref{thm: lower-bound-entropy} rests on a definition of a new type of sequential scale-sensitive dimension of the function class $\calF$. We further relate the sequential square-root entropy and the minimax regret to this dimension. The details are deferred to \pref{sec: lower-bound-proof}. Notice that according to \pref{corr: entropy} and \pref{thm: lower-bound-entropy}, if the sequential square-root entropy of a function class $\calF$ satisfies 
$$\sup_{\bx} \calHsq (\calF, \alpha, n, \bx) = \tilde{\Theta}(\alpha^{-p}),$$
for some $0\le p\le 2$, then we have the following tight characterization of the minimax regret up to log factors:
$$\calR_n(\calF) = \tilde{\Theta}\left(n^{\frac{p}{p+2}}\right).$$
However, when $p > 2$, there exists a gap between the lower bound in \pref{thm: lower-bound-entropy} and the upper bound in \pref{corr: entropy}. Indeed, the following result shows that the upper bound $\tO\left(n^{\frac{p-1}{p}}\right)$ is not improvable in general.
\begin{theorem}\label{thm: lower-bound-large-p}
	Suppose the function class $\calF\subseteq\{f:\calX\to[7/16, 9/16]\}$ satisfies 
	\begin{equation}\label{eq: condition-lower-bound-large-p}
		\sup_{\bx} \calHsq(\calF, \alpha, n, \bx) = \tilde{\Omega}\left(\alpha^{-p}\right).
	\end{equation}
	Then we have the following lower bound on the minimax regret
	$$\calR_n(\calF) = \Omega\left(n^{\frac{p-1}{p}}\right).$$
	Additionally, for any integer $p > 2$, there exists a class $\calF\subseteq\{f:\calX\to[7/16, 9/16]\}$ such that \pref{eq: condition-lower-bound-large-p} holds.
\end{theorem}
The proof of \pref{thm: lower-bound-large-p} is deferred to \pref{sec: lower-bound-proof}. Comparing \pref{thm: lower-bound-large-p} and \pref{corr: entropy}, we see a dichotomy between the regime of $0\le p\le 2$ and the regime of $p > 2$, where the rates of minimax regret $\calR_n(\calF)$ have different behaviors. Such a dichotomy is analogous to the one for online regression \cite{rakhlin2014online} and to  misspecified regression with i.i.d. data \cite{rakhlin2017empirical}.

\subsection{Examples}\label{sec: examples}
In this section, we provide several examples to illustrate \pref{thm: upper-bound} and \pref{corr: entropy}. We consider the example of linear class (Hilbert ball class) and the class of one-dimensional Lipschitz Functions.

\paragraph{Hilbert Ball}\label{sec: hilbert ball}
Consider $\calX = B_2(1)$ to be the infinite dimensional unit ball, and function class $\calF\subseteq [0, 1]^{\calX}$ defined as 
\begin{equation}\label{eq: def-f-hilbert}
	\calF = \left\{f: f(x) = \frac{1 + \langle w, x\rangle}{2}\text{ for some } w\in B_2(1)\right\}.
\end{equation}
This class is generally viewed as `hard case' in existing literature. \cite{rakhlin2015sequential} proposed an ad hoc  follow-the-regularized-leader (FTRL) algorithm with log-barrier regularizer, which achieves the optimal regret $\tilde{\mathcal{O}}(\sqrt{n})$. In terms of entropy characterizations, the same paper provided a loose upper bound of $\mathcal{O}(n^{3/4})$ and \cite{bilodeau2020tight,wu2022precise} provided an upper bound of $\mathcal{O}(n^{2/3})$ using their versions of sequential entropies. The present work is the first to define an appropriate version of sequential entropy (and a corresponding regret bound) to derive a matching $\tilde O(\sqrt{n})$ regret bound. 

We first truncate the function class $\calF$ as follows:
\begin{equation}\label{eq: def-f-n-hilbert}
	\calF_{1/n} = \left\{f: f(x) = \frac{1 + \langle w, x\rangle}{2} \text{ for some } w\in B_2(1-1/n)\right\}.
\end{equation}
The following lemma indicates that minimax regret of $\calF_{1/n}$ is similar to that of $\calF$.
\begin{lemma}\label{lem: truncation-hilbert}
	For $\calF$ and $\calF_{1/n}$ defined in \pref{eq: def-f-hilbert} and \pref{eq: def-f-n-hilbert}, the minimax regret satisfies
	$$\calR_n(\calF)\le \calR_n(\calF_{1/n}) + 2.$$
\end{lemma}
The proof of \pref{lem: truncation-hilbert} is deferred to \pref{sec: example-hilbert-app}. Equipped with this lemma, we only need to bound the sequential square-root entropy of function class $\calF_{1/n}$.
\begin{proposition}\label{prop: tree-fat}
	It holds that
	$$\sup_{\bx} \calHsq(\calF_{1/n}, \alpha, n, \bx) = \mathcal{O}\left(\frac{\log n}{\alpha^2}\cdot \log\left(\frac{n}{\alpha}\right)\right).$$
\end{proposition}
The proof of \pref{prop: tree-fat} is deferred to \pref{sec: example-hilbert-app}. As a consequence, in view of \pref{corr: entropy}, we conclude:
\begin{corollary}
	The minimax regret $\calR_n(\calF)$ of Hilbert ball class $\calF$ satisfies
	$$\calR_n(\calF) = \tilde{\mathcal{O}}\left(\sqrt{n}\right).$$
\end{corollary}

\paragraph{One-Dimensional Lipschitz Function Class}
We consider the example of one-dimensional Lipschitz function class, which has been studied in \cite{bilodeau2020tight, wu2022precise, foster2021efficient}, among others. In this case, the context set is $\calX=[0, 1]$, and the function class $\calF$ is defined to be 
\begin{equation}\label{eq: def-Lipschitz}
	\calF = \{f: [0, 1]\to [0, 1], f\text{ is }1\text{-Lipschitz}\}.
\end{equation}
In \cite{bilodeau2020tight}, the minimax regret is shown to be upper bounded by $\tilde{\mathcal{O}}(\sqrt{n})$, which matches the lower bound. We now recover this rate using sequential square-root entropy. According to \pref{prop: entropy-connection}, the characterization $\sup_{\bx}\calH_\infty(\calF, \alpha, n, \bx) = \Theta(\alpha^{-1})$ for one-dimensional Lipschitz function class in \cite[Theorem 3]{bilodeau2020tight} directly indicates that for square-root entropy,
$$\sup_{\bx}\calHsq(\calF, \alpha, n, \bx) = \mathcal{O}\left(\alpha^{-2}\right).$$
Similarly to the Hilbert ball example, we conclude:
\begin{corollary}
	For $\calX = [0, 1]$ and Lipschitz function class $\calF$ defined in \pref{eq: def-Lipschitz}, 
	$$\calR_n(\calF) = \tilde{\mathcal{O}}(\sqrt{n}).$$
\end{corollary}

\section{Proof Sketch of \pref{thm: general-upper-bound}}\label{sec: proof-sketch-main}
In this section, we sketch the proof of \pref{thm: general-upper-bound}.  The detailed proof is deferred to \pref{sec: proof-upper-bound-app}. We break up the proof into the following key steps: 

\paragraph{Transform minimax regret into the dual form. } Our first step of analyzing the minimax regret $\calR_n(\calQ)$ is to transform it into the dual form \cite{bilodeau2020tight,rakhlin2015sequential}:  
$$\calR_n(\calQ) = \sup_{\bp} \EE_{\by\sim \bp}\calR_n(\calQ, \bp, \by), ~~\text{where}~~\calR_n(\calQ, \bp, \by) := \sup_{\bq\in \calQ} \log\left(\frac{\bq(\by)}{\bp(\by)}\right).$$
where the first supremum is over all joint distributions $\bp\in \Delta(\calY^n)$. In the following, we upper bound $\EE_{\by\sim \bp}\calR_n(\calQ, \bp, \by)$ for any $\bp$.

\paragraph{Truncating the distributions. } We first show that by truncating the distribution $\bp$ and every $\bq\in\calQ$ so that all conditional probabilities $p_t(y_t\mid \bw)$ and $q_t(y_t\mid \bw)$  take values in the interval $[\delta, 1 - \delta]$, for an appropriate $\delta$, we pay an additional constant factor in regret.

\paragraph{Construct offset Rademacher processes. } After truncation, we proceed to introduce the offset Rademacher processes through a symmetrization argument. To do this, we define $\zeta: (0, \infty)\to \RR$ satisfying the following three properties: for some appropriately chosen positive number $c$,
\begin{enumerate}[label=(\roman*)]
	\item \textbf{Transformation of logarithm: } $\log x\le \zeta(x) - c\cdot \zeta(x)^2$ for any $\delta\le x\le 1/\delta$, where $\delta$ is the truncation scale. This property is inspired by the transformation in \cite{bilodeau2020tight}.
	\item \textbf{Nonnegativity of divergence: } 
    $\EE_{y\sim p} \left\{-\zeta(f(y)/p(y)) -c\cdot  \zeta(f(y)/p(y))^2\right\}\geq 0$ for any $f, p\in \Delta(\calY)$.
    This property is inspired by the proof of \cite{cesa1999minimax} where nonnegativity of KL was used. Here we ensure that $- \zeta(x) - c\cdot \zeta(x)^2$ is convex with respect to $x$ and takes on the value $0$ at $x=1$, inducing an $f$-divergence.
	\item \textbf{Lipschitz property: } for any $p, q\in [0, \infty)$, $|\zeta(p) - \zeta(q)|\le 2|\sqrt{p} - \sqrt{q}|$.
\end{enumerate}
The explicit form of $\zeta$ with these three conditions is given in \pref{sec: proof-upper-bound-2}. As a consequence, we obtain the following inequality after a sequential symmetrization argument:
\begin{align*}
	\EE_{\by\sim \bp}\calR_n(\calQ, \bp, \by) & = \EE_{\bw\sim \bp}\left[\sup_{\bq\in\calQ}\sum_{t=1}^{n}\left(\log q_t(w_{t}\mid \bw) - \log p_{t}(w_t\mid \bw)\right)\right]\\
		&\le \EE_{\bw\sim \bp}\left[\sup_{\bq\in\calQ}\sum_{t=1}^{n}\left\{\zeta\left(\frac{q_t(w_{t}\mid \bw)}{p_{t}(w_t\mid \bw)}\right) - c\cdot\zeta\left(\frac{q_t(w_{t}\mid \bw)}{p_{t}(w_t\mid \bw)}\right)^2\right\}\right]\\
     &\le \EE\left[\sup_{\bq\in\calQ}\sum_{t=1}^{n}\epsilon_{t}\zeta\left(\frac{q_t(y_{t}\mid \bw)}{p_t(y_t\mid \bw)}\right) - c\cdot \zeta\left(\frac{q_t(y_{t}\mid \bw)}{p_t(y_t\mid \bw)}\right)^{2}\right]\\
	&\quad + \EE\left[\sup_{\bq\in\calQ}\sum_{t=1}^{n} (- \epsilon_{t})\zeta\left(\frac{q_t(z_t\mid \bw)}{p_t(z_t\mid \bw)}\right) - c\cdot \zeta\left(\frac{q_t(z_t\mid \bw)}{p_t(z_t\mid \bw)}\right)^{2}\right], \numberthis \label{eq: proof-sketch-symmetrization}
\end{align*}
where $\epsilon_t$ are Rademacher random variables, i.e. $\epsilon_{1:n}\simiid \unif\{-1, 1\}$, and $\by = (y_{1:n}), \bz = (z_{1:n}), \bw = (z_{1:n})$ have a specific coupling:  $y_t, z_t\simiid p_t(\cdot\mid w_{1:t-1})$, and $w_t = y_t$ if $\epsilon_t = 1$ or $w_t = z_t$ if $\epsilon_t = -1$. The scheme that $w_t$ chooses $y_t$ or $z_t$ based on the value of $\epsilon_t$ is a variant of the ``selectors'' approach of \cite{rakhlin2011online}.

\paragraph{Analysis through chaining technique} Finally, to upper bound the right hand side of \pref{eq: proof-sketch-symmetrization}, we adopt the chaining technique \cite{dudley1978central,rakhlin2014online,rakhlin2015sequentialcomplexity,rakhlin2015sequential}. We sketch the beginning of the argument. The first term (and, analogously, the second term) in \eqref{eq: proof-sketch-symmetrization} can be decomposed through a chain of $N$ approximating representatives as
\begin{align}
	& \EE\left[\sup_{\bq\in\calQ}\sum_{t=1}^{n}\epsilon_{t}\left\{\zeta\left(\frac{q_t(y_{t}\mid \bw)}{p_t(y_t\mid \bw)}\right) - c\cdot \zeta\left(\frac{q_t(y_{t}\mid \bw)}{p_t(y_t\mid \bw)}\right)^{2}\right\}\right]\\
	& \le \EE\left[\sup_{\bq\in \calQ}\sum_{t=1}^n\epsilon_{t}\left\{\zeta\left(\frac{q_t(y_{t}\mid \bw)}{p_t(y_t\mid \bw)}\right) - \zeta\left(\frac{v_{t}[\bq, \bp, \bw, \by, \alpha_1](y_{t}\mid \bw)}{p_t(y_t\mid \bw)}\right)\right\}\right] \notag\\
	&\qquad + \sum_{i=1}^{N-1} \EE\left[\sup_{\bq\in \calQ}\sum_{t=1}^n\epsilon_{t}\left\{\zeta\left(\frac{v_{t}[\bq, \bp, \bw, \by, \alpha_i](y_{t}\mid \bw)}{p_t(y_t\mid \bw)}\right) - \zeta\left(\frac{v_{t}[\bq, \bp, \bw, \by, \alpha_{i+1}](y_{t}\mid \bw)}{p_t(y_t\mid \bw)}\right)\right\}\right]\notag\\
	& \qquad + \EE\left[\sup_{\bv\in\calV(\alpha_N)\cup\{\bp\}}\sum_{t=1}^n\left\{\epsilon_{t}\zeta\left(\frac{v_t(y_{t}\mid \bw)}{p_t(y_t\mid \bw)}\right) - \frac{c}{4}\cdot \zeta\left(\frac{v_t(y_{t}\mid \bw)}{p_t(y_t\mid \bw)}\right)^{2}\right\}\right]. \notag
\end{align}
where $\bv[\bq, \bp, \bw, \by, \alpha_i]$ is an element of an $\alpha_i$-cover $\calV(\alpha_i)$ of $\calQ$.  The three terms in the above decomposition give rise to the corresponding three terms in the bound of \pref{thm: general-upper-bound}: the approximation at the finest scale (term 1), the Dudley-style term (term 2), and the finite cover at the coarsest scale (term 3).

We will use Cauchy-Schwarz inequality to bound the first term. The second term is a form of sequential Rademacher process. The third term is an offset sequential Rademacher process. However, the key difficulty in dealing with the second and third terms is that the coefficients of the Rademacher random variables are not uniformly bounded by a constant, and directly applying prior techniques does not provide the desired upper bounds. To overcome this issue, we establish upper bounds on offset and non-offset sequential Rademacher processes with unbounded coefficients (\pref{lem: offset-lemma}, \pref{lem: non-offset-lemma}), heavily relying on the properties of the function $\zeta$. Since the latter has $\sqrt{p_t}$-type terms in the denominator in the relevant range of behavior, the squared increments of the process, \emph{under the expectation over $p_t$}, are controlled.  The formal proofs are deferred to the appendix.

\section*{Acknowledgements}
ZJ and AR acknowledge support of the Simons Foundation and the NSF through awards DMS-2031883 and PHY-2019786, as well as from the ARO through award W911NF-21-1-0328.

\arxiv{
	\bibliographystyle{alpha}
}
\bibliography{References.bib}

\newpage
\appendix
\renewcommand{\contentsname}{}
\addtocontents{toc}{\protect\setcounter{tocdepth}{2}}
{\hypersetup{hidelinks}
\tableofcontents
}

\section{Finite Class Lemmas}
We first provide a version of \cite[Lemma 10]{rakhlin2014online}. 
\begin{lemma}\label{lem: offset-lemma}
	Suppose $\epsilon_{1:n}$ are $n$ i.i.d. Rademacher random variables, i.e. $\epsilon_{1:n}\stackrel{i.i.d.}{\sim} \unif(\{-1, 1\})$, and $\calG_{1:n}$ is a filtration which satisfies that $\EE[\epsilon_t\mid \calG_{t}] = 0$ for any $t\in [n]$. Given $n$ sets $\calS_1, \ldots, \calS_n$, we suppose $s_1, s_2, \ldots, s_n$ are $\calS_1, \calS_2, \ldots, \calS_n$-valued random variables such that $s_t$ is $\calG_t$-measurable, i.e. $\sigma$-algebra $\sigma(s_t)\subseteq \calG_t$. For class $\calA$ of tuples $\ba = (a_1, a_2, \ldots, a_n)$ with $a_t: \calS_t\to \RR$ for all $t\in [n]$, we have for any $\lambda > 0$,
	$$\left\{\EE_{s_t}\EE_{\epsilon_t}\right\}_{t=1}^n\left[\sup_{\ba\in \mathcal{A}}\sum_{t=1}^{n}a_{t}(s_t)\epsilon_{t} - \lambda a_{t}(s_t)^{2}\right]\le \frac{\log|\calA|}{2\lambda},$$
	where we denote $\ba = (a_1, a_2, \ldots, a_n)$.
\end{lemma}
\begin{proof}
	We observe that
	\begin{align*} 
		&\hspace{-0.5cm} \left\{\EE_{s_t}\EE_{\epsilon_t}\right\}_{t=1}^n\left[\sup_{\ba\in \mathcal{A}}\sum_{t=1}^{n} a_{t}(s_t)\epsilon_{t} - \lambda a_{t}(s_t)^{2}\right]\\
		& \stackrel{(i)}{\le} \frac{1}{2\lambda}\log\left\{\EE_{s_t}\EE_{\epsilon_t}\right\}_{t=1}^n\sup_{\ba\in \mathcal{A}}\left[\exp\left(2\lambda\sum_{t=1}^{n}a_{t}(s_t)\epsilon_{t} - 2\lambda^2 a_{t}(s_t)^{2}\right)\right]\\
		& \stackrel{(ii)}{\le} \frac{1}{2\lambda}\log\sum_{\ba\in \mathcal{A}}\left\{\EE_{s_t}\EE_{\epsilon_t}\right\}_{t=1}^n \exp\left(2\lambda\sum_{t=1}^{n}a_{t}(s_t)\epsilon_{t} - 2\lambda^2 a_{t}(s_t)^{2}\right)\\
		& = \frac{1}{2\lambda}\log\sum_{\ba\in \mathcal{A}}\left\{\EE_{s_t}\EE_{\epsilon_t}\right\}_{t=1}^{n-1}\Bigg[\exp\left(2\lambda\sum_{t=1}^{n-1}a_{t}(s_t)\epsilon_{t} - 2\lambda^2 a_{t}(s_t)^{2}\right)\\
		&\qquad \cdot \EE_{s_n}\left[\exp\left(-2\lambda^2a_n(s_{n})^2\right)\left(\frac{\exp(2\lambda a_n(s_n))}{2} + \frac{\exp(-2\lambda a_n(s_n))}{2}\right)\mid \calG_n\right]\Bigg]\\
		&\stackrel{(iii)}{\le} \frac{1}{2\lambda}\log\sum_{\ba\in \mathcal{A}}\left\{\EE_{s_t}\EE_{\epsilon_t}\right\}_{t=1}^{n-1}\exp\left(2\lambda\sum_{t=1}^{n-1}a_{t}(s_t)\epsilon_{t} - 2\lambda^2 a_{t}(s_t)^{2}\right),
	\end{align*}
	where in $(i)$ we use the Jensen's inequality, in $(ii)$ we use replace the $\sup$ by the sum since the terms inside $\sup$ are always positive, and in $(iii)$ we use the inequality $\exp(x^2/2)\ge \exp(x)/2 + \exp(-x)/2$ for any $x\in\RR$. 
    By repeating the argument $n$ times we obtain that
	$$\left\{\EE_{s_t}\EE_{\epsilon_t}\right\}_{t=1}^n\left[\sup_{\ba\in \mathcal{A}}\sum_{t=1}^{n}a_{t}(s_t)\epsilon_{t} - \lambda a_{t}(s_t)^{2}\right]\le \frac{1}{2\lambda}\log\sum_{\ba\in \mathcal{A}} 1 = \frac{\log|\calA|}{2\lambda}.$$
\end{proof}

\pref{lem: offset-lemma} implies the following upper bound for non-offset Rademacher processes, which enables us to bound the Rademacher process with random coefficients that are only small on average.
\begin{lemma}\label{lem: non-offset-lemma}
	Suppose $\epsilon_{1:n}$ are $n$ i.i.d. Rademacher random variables, i.e. $\epsilon_{1:n}\stackrel{i.i.d.}{\sim} \unif(\{-1, 1\})$, and $\calG_{1:n}$ is a filtration which satisfies that $\EE[\epsilon_t\mid \calG_{t}] = 0$ for any $t\in [n]$. Given $n$ sets $\calS_1, \ldots, \calS_n$, we suppose $s_1, s_2, \ldots, s_n$ are $\calS_1, \calS_2, \ldots, \calS_n$-valued random variables such that $s_t$ is $\calG_t$-measurable, i.e. $\sigma$-algebra $\sigma(s_t)\subset \calG_t$. For class $\calA$ of tuples $\ba = (a_1, a_2, \ldots, a_n)$ with $a_t: \calS_t\to \RR$ for all $t\in [n]$, we have
	$$\left\{\EE_{s_t}\EE_{\epsilon_t}\right\}_{t=1}^n\left[\sup_{\ba\in \mathcal{A}}\sum_{t=1}^{n}a_{t}(s_t)\epsilon_{t}\right]\le \sqrt{2\log|\calA|}\cdot \sqrt{\EE\left[ \sup_{\ba\in\calA}\sum_{t=1}^n a_t(s_t)^2\right]}.$$
	In particular, for $\calS_t=\{\pm1\}^{t-1}$ and $s_t = (\epsilon_{1:t-1})\in \calS_t$,
     \begin{align}
    \label{eq:finite_class}
    \EE_{\epsilon_{1:n}}\left[\sup_{\ba\in \mathcal{A}}\sum_{t=1}^{n}a_{t}(\epsilon_{1:t-1})\epsilon_{t}\right]\le \sqrt{2\log|\calA|}\cdot \sqrt{\EE\left[ \sup_{\ba\in\calA}\sum_{t=1}^n a_t(\epsilon_{1:t-1})^2\right]}.
    \end{align}
\end{lemma}
\begin{proof}
	According to \pref{lem: offset-lemma}, we have for any $\lambda > 0$,
	$$\left\{\EE_{s_t}\EE_{\epsilon_t}\right\}_{t=1}^n\left[\sup_{\ba\in \mathcal{A}}\sum_{t=1}^{n}a_{t}(s_t)\epsilon_{t} - \lambda a_{t}(s_t)^{2}\right] \le \frac{\log|\calA|}{2\lambda}.$$
	We let
	$$\beta = \left\{\EE_{s_t}\EE_{\epsilon_t}\right\}_{t=1}^n\left[\sup_{\ba\in\calA} \sum_{t=1}^n a_t(s_t)^2\right] = \EE\left[\sup_{\ba\in\calA} \sum_{t=1}^n a_t(s_t)^2\right].$$
	By choosing $\lambda = \sqrt{\frac{\log|\calA|}{2 \beta}} > 0$, we obtain that
	\begin{align*}
		& \hspace{-0.5cm}\left\{\EE_{s_t}\EE_{\epsilon_t}\right\}_{t=1}^n\left[\sup_{\ba\in \mathcal{A}}\sum_{t=1}^{n}a_{t}(s_t)\epsilon_{t}\right]\\
		& \le \left\{\EE_{s_t}\EE_{\epsilon_t}\right\}_{t=1}^n\left[\sup_{\ba\in \mathcal{A}}\sum_{t=1}^{n}a_{t}(s_t)\epsilon_{t} - \lambda a_{t}(s_t)^{2}\right] + \lambda\cdot \left\{\EE_{s_t}\EE_{\epsilon_t}\right\}_{t=1}^n\left[\sup_{\ba\in\calA} \sum_{t=1}^n a_t(s_t)^2\right]\\
		& \le \frac{\log |\calA|}{2\lambda} + \lambda \beta = \sqrt{2 \beta\log|\calA|} = \sqrt{2\log|\calA|}\cdot \sqrt{\EE\left[ \sup_{\ba\in\calA}\sum_{t=1}^n a_t(s_t)^2\right]}.
	\end{align*}
\end{proof}
\pref{lem: non-offset-lemma} is an improvement on the finite class lemma in \cite[Lemma 1]{rakhlin2015sequentialcomplexity}; the latter result was proved with the supremum (rather than the expected value) over ${\epsilon_{1:n}}$ under the square root in \eqref{eq:finite_class}.

\section{Proof of \pref{thm: general-upper-bound}}\label{sec: proof-upper-bound-app}

\subsection{Proof Outline}\label{sec: proof-sketch}

The proof has the following structure. Our first step is to write the minimax regret in the dual form using the minimax theorem. This technique is widely used in the analysis of minimax regret of online learning \cite{abernethy2009stochastic, rakhlin2014online, rakhlin2015online,rakhlin2015sequential, foster2018logistic, bilodeau2020tight}. The next step is to truncate the functions and forecaster's strategies away from $0$. This analysis technique is also used in \cite{cesa1999minimax, rakhlin2015online}. 
Our main steps in the proof include constructing an offset Rademacher process using a symmetrization argument \cite{gine1984some} and using chaining techniques \cite{dudley1967sizes, geer2000empirical} to analyze the offset Rademacher process. The analysis of the chaining steps involves complex dependence of the Rademacher variables and the coefficients, and this is one of the technical hurdles. 

\subsubsection{Conversion to Dual Form Game}
We have the following standard result (see e.g. \cite[Lemma 6]{bilodeau2020tight} or \cite[Eq. 27] {rakhlin2015sequential})):
\begin{lemma}\label{lem: dual-form}
	For any $\calQ\in \Delta(\calY^n)$, the minimax regret $\calR_n(\calQ)$ has the following dual form representation
    $$\calR_n(\calQ) = \sup_{\bp} \EE_{\by\sim \bp}\calR_n(\calQ, \bp, \by),$$
	where the supremum is over all joint distributions $\bp\in \Delta(\calY^n)$ and
	\begin{equation}\label{eq: r-n-formula}
		\calR_n(\calQ, \bp, \by) \coloneqq \sup_{\bq\in \calQ} \log\left(\frac{\bq(\by)}{\bp(\by)}\right).
	\end{equation}
\end{lemma}
\begin{proof}[Proof of \pref{lem: dual-form}]
	We notice that 
	$$\calR_n(\calQ) = \inf_{\bhp}\sup_{\by} \calR_n(\calQ, \bhp, \by) = \inf_{\bhp}\sup_{\bp}\EE_{\by\sim \bp}[\calR_n(\calQ, \bhp, \by)].$$
	Since $\Delta(\calY^n)$ is compact, and $\EE_{\by\sim \bp}[\calR_n(\calQ, \bhp, \by)]$ is convex with respect to $\bhp$ and concave with respect to $\bp$, von Neumann minimax theorem \cite{v1928theorie} gives 
	$$\calR_n(\calQ) = \sup_{\bp}\inf_{\bhp}\EE_{\by\sim \bp}[\calR_n(\calQ, \bhp, \by)] = \sup_{\bp}\EE_{\by\sim \bp}[\calR_n(\calQ, \bp, \by)],$$
	where the last inequality uses the fact that the infimum of 
	$\inf_{\bhp}\EE_{\by\sim \bp}[\calR_n(\calQ, \bhp, \by)]$ is attained when $\bhp = \bp$.
\end{proof}

In the remainder, we upper bound the minimax regret for $\EE_{\by\sim \bp}\calR_n(\calQ, \bp, \by)$ for any fixed $\bp\in\Delta(\calY^n)$.

\subsubsection{Truncation of Functions and Probabilities}\label{sec: proof-upper-bound-1}

For $\bp\in \Delta(\calY^n)$ with $\bp(\by) = \prod_{t=1}^n p_t(y_t\mid \by)$ for every $\by\in \calY^n$, and $\delta < 1/(4|\calY|)$, we define distribution $\bp^\delta\in \Delta(\calY^n)$ with $\bp^\delta(\by) = \prod_{t=1}^n p_t^\delta(y_t\mid \by)$ for every $\by\in \calY^n$, where
\begin{align*}
	p_t^\delta(y_t\mid \by) = \begin{cases}
		\delta &\quad \text{if }p_t(y_t\mid \by) < \delta\\
		p_t(y_t\mid \by) &\quad \text{if }\delta\le p_t(y_t\mid \by)\le 2\delta,\\
		p_t(y_t\mid \by)\cdot \frac{1 - \sum_{y\in \calY}p_t(y\mid \by)\mathbb{I}[\delta\le p_t(y\mid \by) < 2\delta] - \delta\sum_{y\in \calY} \mathbb{I}[p_t(y\mid \by) < \delta]}{1 - \sum_{y\in \calY} p_t(y\mid \by)\mathbb{I}[p_t(y\mid \by) < 2\delta]} & \quad \text{if }p_t(y\mid \by) \ge 2\delta.
	\end{cases}. \numberthis \label{eq: def-truncation}
\end{align*}
It holds that $p_t^\delta(\cdot\mid \by)\in \Delta(\calY)$ for any $\by\in \calY^n$ and $t\in [n]$. Additionally, we notice that 
$$1 - \sum_{y\in \calY}p_t(y\mid \by)\mathbb{I}[\delta\le p_t(y_t\mid \by) < 2\delta] - \delta\sum_{y\in \calY} \mathbb{I}[p_t(y\mid \by) < \delta] \ge 1 - |\calY|\cdot 2\delta\ge \frac{1}{2},$$
which implies that $p_t^\delta(y_t\mid \by)\ge \frac{1}{2}p_t(y_t\mid \by)$ if $p_t(y_t\mid \by)\ge 2\delta$. Hence, for any $\by\in \calY^n$, we always have
$$p_t^\delta(y_t\mid \by) > \delta.$$

For class $\calQ\subseteq \Delta(\calY^n)$, we define $\calQ^\delta = \{\bq^\delta\mid \bq\in \calQ\}$.
Then we have the following lemmas:
\begin{lemma}\label{lem: q-delta}
	Suppose $\delta\le \frac{1}{4|\calY|}$. For any $\bp\in \Delta(\calY^n)$, $\by\in \calY^n$ and $\calQ\subseteq \Delta(\calY^n)$, we have
	$$\calR_n(\calQ, \bp, \by)\le \calR_n(\calQ^\delta, \bp, \by) + 4n\delta\cdot |\calY|.$$
\end{lemma}

\begin{lemma}\label{lem: p-delta}
	Suppose $\delta\le \frac{1}{4|\calY|}$. For any $\calQ\subseteq \Delta(\calY^n)$ and $\bp\in \Delta(\calY^n)$, we have
	$$\EE_{\by\sim \bp}\left[\calR_n(\calQ^\delta, \bp, \by)\right]\le \EE_{\by\sim \bp^\delta}\left[\calR_n(\calQ^\delta, \bp^\delta, \by)\right] + 2n^2|\calY|\delta\log\frac{1}{\delta}.$$
\end{lemma}
The proofs of \pref{lem: q-delta} and \pref{lem: p-delta} are deferred to \pref{sec: proof-upper-bound-1-app}. In the following, we use $\Delta_n(\calY^n)$ to denote the following joint distribution set: 
\begin{equation}\label{eq: def-delta-n}
	\Delta_n(\calY^n)\coloneqq \left\{\bq\in \Delta(\calY^n): q_t(y_t\mid \by)\ge 1/(n^2|\calY|), \forall \by\in \calY^n\right\}.
\end{equation}

\subsubsection{Symmetrization and Construction of Offset Rademacher Process} \label{sec: proof-upper-bound-2}
To facilitate the symmetrization argument, we define the following function $\zeta: \RR^+\to \RR$: for any $t\ge 0$, 
\begin{equation}\label{eq: def-zeta}
	\zeta(t) = \begin{cases}2\left(\sqrt{t}-1\right), & t\le 1,\\
	2\log\left(\frac{t + 1}{2}\right), & t>1.\end{cases}
\end{equation}
For the justification of this choice of $\zeta$ see \pref{sec: proof-sketch-main}. We next introduce the following three properties of the function $\zeta$, whose proofs are deferred to \pref{sec: proof-upper-bound-2-app}. 
\begin{proposition}\label{prop: log-inequality}
	For every $0 < x\le n^2|\calY|$, 
	\begin{equation}\label{eq: log-inequality}
		\log x\le \zeta(x) - \frac{1}{4\log (n|\calY|)}\cdot \zeta(x)^2.
	\end{equation}
\end{proposition}
\begin{proposition}\label{prop: symmetrization-inverse}
	For any distribution $f, p\in \Delta(\calY)$, we have 
	$$\EE_{y\sim p}\left[-\zeta\left(\frac{f(y)}{p(y)}\right) - \frac{1}{4\log (n|\calY|)}\cdot \zeta\left(\frac{f(y)}{p(y)}\right)^2\right]\ge 0.$$
\end{proposition}
The above proposition can also be obtained by noticing that function $ - \zeta(x) - \frac{1}{4\log(n|\calY|)}\cdot \zeta(x)^2$ is convex in $x$, and the result follows from the property of $f$-divergences \cite[Theorem 7.5]{polyanskiy2014lecture}.
\begin{proposition}\label{prop: lipschitz-property}
	For any $p, q\in [0, \infty)$, we have 
	$$|\zeta(p) - \zeta(q)|\le 2\left|\sqrt{p} - \sqrt{q}\right|.$$
\end{proposition}

Next, we state the symmetrization argument. The symmetrization argument will use the following circle-dot product distributions.

\begin{definition}[Circle-dot Product Distributions]\label{def: circle-dot}
	For label set $\calY$ and any distribution $\bp\in \Delta(\calY^n)$, we define the Circle-dot product distribution $\odot\bp\in \Delta(\{-1, 1\}^n)\times \Delta(\calY^n)\times \Delta(\calY^n)\times \Delta(\calY^n)$ such that  $(\bepsilon, \bw, \by, \bz)\sim \odot \bp$ are sampled according to the following process: first sample $\bepsilon = (\epsilon_{1:n})\simiid \unif\{-1, 1\}$, then repeat the following process for sampling $\bw = (w_{1:t}), \by = (y_{1:t})$ and $\bz = (z_{1:t})$ from $t = 1$ to $n$: sample $y_t, z_t\simiid p_t(\cdot\mid w_{1:t-1})$, and set $w_t = y_t$ if $\epsilon_t = 1$ or $w_t = z_t$ if $\epsilon_t = -1$.
\end{definition}

\begin{lemma}[Symmetrization]\label{lem: symmetrization}
	For any joint distribution $\bp\in \Delta_n(\calY^n)$ where $\Delta_n(\calY^n)$ is defined in \pref{eq: def-delta-n}, suppose $(\bepsilon, \bw, \by, \bz)\sim \odot \bp$.
    Then for any joint distribution class $\calQ\subseteq \Delta_n(\calY^n)$, we have the following upper bound:
	\begin{align*}
		&\hspace{-0.5cm} \EE_{\by\sim \bp}\calR_n(\calQ, \bp, \by)\\
		& \le \EE\left[\sup_{\bq\in\calQ}\sum_{t=1}^{n}\epsilon_{t}\zeta\left(\frac{q_t(y_{t}\mid \bw)}{p_t(y_t\mid \bw)}\right) - \frac{1}{4\log (n|\calY|)}\zeta\left(\frac{q_t(y_{t}\mid \bw)}{p_t(y_t\mid \bw)}\right)^{2}\right]\\
		&\quad + \EE\left[\sup_{\bq\in\calQ}\sum_{t=1}^{n} (- \epsilon_{t})\zeta\left(\frac{q_t(z_t\mid \bw)}{p_t(z_t\mid \bw)}\right) - \frac{1}{4\log (n|\calY|)}\zeta\left(\frac{q_t(z_t\mid \bw)}{p_t(z_t\mid \bw)}\right)^{2}\right], \numberthis\label{eq: symmetrization}
	\end{align*}
	where the expectation is with respect to $(\bepsilon, \bw, \by, \bz)\sim \odot \bp$.
\end{lemma}
The proof of \pref{lem: symmetrization} is deferred to \pref{sec: proof-upper-bound-2-app}. It is based on the aforementioned properties of function $\zeta$, and the symmetrization technique in \cite{rakhlin2011online}.

\subsubsection{Chaining}\label{sec: proof-upper-bound-3}
We next analyze the right hand side of \pref{eq: symmetrization} using a chaining argument. For simplicity we only upper bound the first term, and the bound on the second term is similar.

To adopt the chaining argument to the Rademacher process defined in the right hand side of \pref{eq: symmetrization} while keeping the offset term, we need to establish certain properties of the sequential cover of the function class. Specifically, for any joint distribution $\bq\in \calQ$, we are required to have some instance $\bv$ in the cover, such that the $\ell_2$-norm of the coefficients with $\bq$ is lower bounded by the $\ell_2$-norm of the coefficients with $\bv$, as is the following lemma, whose proof is deferred to \pref{sec: proof-upper-bound-3-app}:
\begin{lemma}\label{lem: another-cover}
	Fix joint distribution $\bp\in \Delta(\calY^n)$ and  class $\calQ\subseteq \Delta(\calY^n)$. Let $\calV(\alpha)$  be a sequential square-root cover of $\calQ$ at scale $\alpha > 0$. Then for any $\bq\in\calQ$, $\bv'\in \calV(\alpha)$ and $\bw, \by\in \calY^n$, there exists  $\bv\in\calV(\alpha)\cup\{\bp\}$ such that
	\begin{align*} 
		\sum_{t=1}^n\left(\zeta\left(\frac{q_t(y_t\mid \bw)}{p_t(y_t\mid \bw)}\right) - \zeta\left(\frac{v_t(y_t\mid \bw)}{p_t(y_t\mid \bw)}\right)\right)^2\le \sum_{t=1}^n\left(\zeta\left(\frac{q_t(y_t\mid \bw)}{p_t(y_t\mid \bw)}\right) - \zeta\left(\frac{v_t'(y_t\mid \bw)}{p_t(y_t\mid \bw)}\right)\right)^2, \numberthis \label{eq: condition-one}
	\end{align*}
	and 
	\begin{align*} 
		\sum_{t=1}^n \zeta\left(\frac{q_t(y_{t}\mid \bw)}{p_t(y_t\mid \bw)}\right)^{2}\ge \frac{1}{4}\sum_{t=1}^n \zeta\left(\frac{v_t(y_t\mid \bw)}{p_t(y_t\mid \bw)}\right)^{2}. \numberthis \label{eq: condition-two}
	\end{align*}
\end{lemma}
\begin{remark}
	The above lemma is similar to \cite[Eq. (40)]{rakhlin2015sequential}, \cite[Eq. (46)]{rakhlin2014online}. The additional atom $\bp$ serves as the `zero' element in \cite[Eq. (40)]{rakhlin2015sequential}.
\end{remark}

This lemma enables us to keep the offset terms during the chaining process. We now detail these steps. We fix $N$ scales $0 < \alpha_1 < \alpha_2 < \cdots < \alpha_N$, and let $\calV(\alpha_i)$ to be the smallest cover of $\calQ$ at scale $\alpha_i$ under \pref{def: sequential-covering}. Then we have the following lemma.
\begin{lemma}\label{lem: chaining}
    For any $i\in [N-1]$, we fix $\bv[\bq, \bp, \bw, \by, \alpha_i]\in \calV(\alpha_i)$. Suppose $\bv[\bq, \bp, \bw, \by, \alpha_N]\in \calV(\alpha_N)\cup\{\bp\}$ satisfies \pref{eq: condition-two} with $\bv = \bv[\bq, \bp, \bw, \by, \alpha_N]$. We then hwave 
    \begin{align}
    \label{eq:three_term_decomp}
	& \hspace{-0.5cm} \EE\left[\sup_{\bq\in\calQ}\sum_{t=1}^{n}\left\{\epsilon_{t}\zeta\left(\frac{q_t(y_{t}\mid \bw)}{p_t(y_t\mid \bw)}\right) - \frac{1}{4\log (n|\calY|)}\zeta\left(\frac{q_t(y_{t}\mid \bw)}{p_t(y_t\mid \bw)}\right)^{2}\right\}\right]\\
	& \le \EE\left[\sup_{\bq\in \calQ}\sum_{t=1}^n\epsilon_{t}\left\{\zeta\left(\frac{q_t(y_{t}\mid \bw)}{p_t(y_t\mid \bw)}\right) - \zeta\left(\frac{v_{t}[\bq, \bp, \bw, \by, \alpha_1](y_{t}\mid \bw)}{p_t(y_t\mid \bw)}\right)\right\}\right] \notag\\
	&\qquad + \sum_{i=1}^{N-1} \EE\left[\sup_{\bq\in \calQ}\sum_{t=1}^n\epsilon_{t}\left\{\zeta\left(\frac{v_{t}[\bq, \bp, \bw, \by, \alpha_i](y_{t}\mid \bw)}{p_t(y_t\mid \bw)}\right) - \zeta\left(\frac{v_{t}[\bq, \bp, \bw, \by, \alpha_{i+1}](y_{t}\mid \bw)}{p_t(y_t\mid \bw)}\right)\right\}\right]\notag\\
	& \qquad + \EE\left[\sup_{\bv\in\calV(\alpha_N)\cup\{\bp\}}\sum_{t=1}^n\left\{\epsilon_{t}\zeta\left(\frac{v_t(y_{t}\mid \bw)}{p_t(y_t\mid \bw)}\right) - \frac{1}{16\log (n|\calY|)}\cdot \zeta\left(\frac{v_t(y_{t}\mid \bw)}{p_t(y_t\mid \bw)}\right)^{2}\right\}\right], \notag
    \end{align}
    where the expectation is with respect to $(\bepsilon, \bw, \by, \bz)\sim \odot\bp$.
\end{lemma}
The proof of \pref{lem: chaining} is deferred to \pref{sec: proof-upper-bound-3-app}. Next, we further upper bound the three terms in \pref{eq:three_term_decomp}. Specifically, we will use Cauchy-Schwarz inequality to bound the first term. The second term is a form of sequential Rademacher process. The third term is an offset Rademacher process. However, the key difficulty in dealing with the second and third terms is that the coefficients of the Rademacher random variables are not uniformly bounded by a constant, and directly applying prior techniques does not provide the desired upper bounds. To overcome this issue, we employ a technique that uses offset complexities instead, as in the proof of \pref{lem: non-offset-lemma} (see \pref{rem:offset} in the proof of \pref{thm: general-upper-bound} for a discussion). The formal proof of these arguments, together with the full proof of \pref{thm: general-upper-bound}, is deferred to \pref{sec: proof-upper-bound-4-app}.

\subsection{Missing Proofs in \pref{sec: proof-upper-bound-1}}\label{sec: proof-upper-bound-1-app}

\begin{proof}[Proof of \pref{lem: q-delta}]
	Given the formula of $\calR_n(\calQ, \bp, \by)$ in \pref{eq: r-n-formula}, we only need to verify
	\begin{equation}\label{eq: truncation-Q-1}
		\sup_{\bq\in \calQ}\sum_{t=1}^n \frac{1}{q_t(y_t\mid \by)}\ge \sup_{\bq\in \calQ^\delta}\sum_{t=1}^n \frac{1}{q_t(y_t\mid \by)} - 4n|\calY|\delta.
	\end{equation}
	Notice that according to our construction of truncation in \pref{eq: def-truncation}, we have for any $y\in \calY$ and $p\in \Delta(\calY)$, 
	\begin{equation}\label{eq: p-p-delta} 
		\log p(y) - \log p^\delta(y) = \log\frac{p(y)}{p^\delta(y)}\le - \log\left(1 - 2\delta\cdot |\calY|\right)\le 4\delta\cdot |\calY|,
	\end{equation}
	where the last inequality uses the fact that $\delta\le \frac{1}{4|\calY|}$ and $-\log(1 - t)\le 2t$ for any $t\le 1/2$. Hence after noticing that $\calQ^\delta = \{\bq^\delta:\bq\in \calQ\}$, we obtain \pref{eq: truncation-Q-1}.
\end{proof}
\begin{proof}[Proof of \pref{lem: p-delta}]
	For any $s\in [n+1]$, we use notation $\bp^{s, \delta} = (p_1^{s, \delta}, \ldots, p_n^{s, \delta})$ to denote a joint distribution such that  
	$$p_t^{s, \delta}(y_t\mid \by) = \begin{cases}
		p_t(y_t\mid \by) &\quad \text{if }t < s,\\
		p_t^{\delta}(y_t\mid \by) &\quad \text{if }t\ge s.
	\end{cases}\qquad \forall \by\in \calY^n.$$
	Then we have $\bp^{1, \delta} = \bp^\delta$ and $\bp^{n+1, \delta} = \bp$, and we can decompose
	\begin{align*}
		&\EE_{\by\sim \bp}\calR_n(\calQ^\delta, \bp, \by) - \EE_{\by\sim \bp^\delta}\calR_n(\calQ^\delta, \bp^\delta, \by) \\
        &= \sum_{s=1}^n\left[\EE_{\by\sim \bp^{s+1, \delta}}\calR_n(\calQ^\delta, \bp^{s+1, \delta}, \by) - \EE_{\by\sim \bp^{s, \delta}}\calR_n(\calQ^\delta, \bp^{s, \delta}, \by)\right]. \numberthis \label{eq: p-decomposition-0}
	\end{align*}
	We expand the right hand side of \pref{eq: p-decomposition-0} for each $s\in [n]$:
	\begin{align*}
		& \hspace{-0.5cm} \EE_{\by\sim \bp^{s+1, \delta}}\calR_n(\calQ^\delta, \bp^{s+1, \delta}, \by) - \EE_{\by\sim \bp^{s, \delta}}\calR_n(\calQ^\delta, \bp^{s, \delta}, \by)\\
		& = \left\{\EE_{y_t\sim p_t(\cdot\mid \by)}\right\}_{t=1}^{s-1}\big[\EE_{y_s\sim p_s(\cdot\mid \by)}\{\EE_{y_t\sim p_t^\delta(\cdot\mid \by)}\}_{t=s+1}^{n}\calR_n(\calQ^\delta, \bp^{s+1, \delta}, \by)\\
		&\qquad - \EE_{y_s\sim p_s^\delta(\cdot\mid \by)}\{\EE_{y_t\sim p_t^\delta(\cdot\mid \by)}\}_{t=s+1}^{n}\calR_n(\calQ^\delta, \bp^{s, \delta}, \by)\big]. \numberthis\label{eq: p-decomposition}
	\end{align*}
	Next, we fix $y_{1:s-1}$ and upper bound the expression inside the expectation: 
	\begin{align*}
		& \hspace{-0.5cm}\EE_{y_s\sim p_s(\cdot\mid \by)}\{\EE_{y_t\sim p_t^\delta(\cdot\mid \by)}\}_{t=s+1}^{n}\calR_n(\calQ^\delta, \bp^{s+1, \delta}, \by) - \EE_{y_s\sim p_s^\delta(\cdot\mid \by)}\{\EE_{y_t\sim p^\delta_t(\cdot\mid \by)}\}_{t=s+1}^{n}\calR_n(\calQ^\delta, \bp^{s, \delta}, \by)\\
		& = \EE_{y_s\sim p_s(\cdot\mid \by)}\{\EE_{y_t\sim p_t^\delta(\cdot\mid \by)}\}_{t=s+1}^{n}\left[\calR_n(\calQ^\delta, \bp^{s+1, \delta}, \by) - \calR_n(\calQ^\delta, \bp^{s, \delta}, \by)\right]\\
		&\ + \big[\EE_{y_s\sim p_s(\cdot\mid \by)}\{\EE_{y_t\sim p_t^\delta(\cdot\mid \by)}\}_{t=s+1}^{n}\calR_n(\calQ^\delta, \bp^{s, \delta}, \by) - \EE_{y_s\sim p_s^\delta(\cdot\mid \by)}\{\EE_{y_t\sim p_t^\delta(\cdot\mid \by)}\}_{t=s+1}^{n}\calR_n(\calQ^\delta, \bp^{s, \delta}, \by)\big] \numberthis \label{eq: p-decomposition-1}
	\end{align*}
	For the first term in the right hand side of \pref{eq: p-decomposition-1}, when fixing $\by\in \calY^n$, we have
	\begin{align*}
		& \hspace{-0.5cm} \calR_n(\calQ^\delta, \bp^{s+1, \delta}, \by) - \calR_n(\calQ^\delta, \bp^{s, \delta}, \by) = \sum_{t=1}^n \log\left(\frac{p_t^{s, \delta}(y_t\mid \by)}{p_t^{s+1, \delta}(y_t\mid \by)}\right) = \log\left(\frac{p_s^\delta(y_s\mid \by)}{p_s(y_s\mid \by)}\right),
	\end{align*}
	which implies that
	\begin{align*}
		&\hspace{-0.5cm} \EE_{y_s\sim p_s(\cdot\mid \by)}\{\EE_{y_t\sim p_t(\cdot\mid \by)}\}_{t=s+1}^{n}\left[\calR_n(\calQ^\delta, \bp^{s+1, \delta}, \by) - \calR_n(\calQ^\delta, \bp^{s, \delta}, \by)\right]\\
		& = \EE_{y_s\sim p_s(\cdot\mid \by)}\left[\log\left(\frac{p_s^\delta(y_s\mid \by)}{p_s(y_s\mid \by)}\right)\right] = - D_{\mathrm{KL}}(p_s(y_s\mid \by)\|p_s^\delta(y_s\mid \by))\le 0.
	\end{align*}

	For the second term in \pref{eq: p-decomposition-1}, when fixing $y_{1:s-1}$, we have
	\begin{align*}
		&\hspace{-0.5cm}\EE_{y_s\sim p_s(\cdot\mid \by)}\{\EE_{y_t\sim p_t^\delta(\cdot\mid \by)}\}_{t=s+1}^{n}\calR_n(\calQ^\delta, \bp^{s, \delta}, \by) - \EE_{y_s\sim p_s^\delta(\cdot\mid \by)}\{\EE_{y_t\sim p_t^\delta(\cdot\mid \by)}\}_{t=s+1}^{n}\calR_n(\calQ^\delta, \bp^{s, \delta}, \by)\\
		& \stackrel{(i)}{=} \EE_{y_s\sim p_s(\cdot\mid \by)}\{\EE_{y_t\sim p_t^\delta(\cdot\mid \by)}\}_{t=s+1}^{n}\left[\sup_{\bq\in\calQ^\delta}\left\{\sum_{t=1}^{s-1} \log\frac{q_t(y_t\mid \by)}{p_t(y_t\mid \by)} + \sum_{t=s}^{n} \log\frac{q_t(y_t\mid \by)}{p_t^\delta(y_t\mid \by)}\right\}\right]\\
		&\quad - \EE_{y_s\sim p_s^\delta(\cdot\mid \by)}\{\EE_{y_t\sim p_t^\delta(\cdot\mid \by)}\}_{t=s+1}^{n}\left[\sup_{\bq\in \calQ^\delta}\left\{\sum_{t=1}^{s-1} \log\frac{q_t(y_t\mid \by)}{p_t(y_t\mid \by)} + \sum_{t=s}^{n} \log\frac{q_t(y_t\mid \by)}{p_t^\delta(y_t\mid \by)}\right\}\right]\\
            & \stackrel{(ii)}{=} \EE_{y_s\sim p_s(\cdot\mid \by)}\{\EE_{y_t\sim p_t^\delta(\cdot\mid \by)}\}_{t=s+1}^{n}\left[\sup_{\bq\in\calQ^\delta}\sum_{t=1}^{n} \log\frac{q_t(y_t\mid \by)}{p_t^\delta(y_t\mid \by)}\right]\\
		&\quad - \EE_{y_s\sim p_s^\delta(\cdot\mid \by)}\{\EE_{y_t\sim p_t^\delta(\cdot\mid \by)}\}_{t=s+1}^{n}\left[\sup_{\bq\in \calQ^\delta}\sum_{t=1}^{n} \log\frac{q_t(y_t\mid \by)}{p_t^\delta(y_t\mid \by)}\right], \numberthis \label{eq: p-decomposition-2}
	\end{align*}
	where $(i)$ uses the formula of $\calR_n(\calQ, \bp, \by)$ in \pref{eq: r-n-formula} and the form of $\bp^{s, \delta}$, and $(ii)$ uses the fact that for $t\le s-1$, $p_t(y_t\mid \by)$ cancels out in both terms, hence we can replace them by $p_t^\delta(y_t\mid \by)$ at no additional cost. Notice that $\delta\le p_t^\delta(y_t\mid \by)\le 1$ and $\delta\le q_t^\delta(y_t\mid \by)\le 1$ hold for any $\bq\in \calQ^\delta$, $\by\in \calY^n$ and $t\in [n]$. Hence,
	$$\left|\sup_{\bq\in\calQ^\delta}\sum_{t=1}^{n} \log\frac{q_t(y_t\mid \by)}{p_t^\delta(y_t\mid \by)}\right|\le n\log\frac{1}{\delta},\quad\forall \by\in \calY^n$$
	which implies that when fixed $y_{1:s-1}$,
	$$\text{RHS of \pref{eq: p-decomposition-2}}\le 2\mathrm{TV}\left(p_s(\cdot\mid \by), p_s^\delta(\cdot\mid \by)\right)\cdot n\log\frac{1}{\delta}.$$
	Based on \pref{eq: def-truncation}, we can calculate that
	$$\mathrm{TV}\left(p_s(\cdot \mid \by), p_s^\delta(\cdot \mid \by)\right) = \sum_{y\in \calY} \left(\delta - p_s(y\mid \by)\right)\vee 0\le |\calY|\delta,$$
	which implies that
	$$\EE_{y_s\sim p_s(\cdot\mid \by)}\{\EE_{y_t\sim p_t^\delta(\cdot\mid \by)}\}_{t=s+1}^{n}\calR_n(\calQ^\delta, \bp^{s, \delta}, \by) - \EE_{y_s\sim p_s^\delta(\cdot\mid \by)}\{\EE_{y_t\sim p_t^\delta(\cdot\mid \by)}\}_{t=s+1}^{n}\calR_n(\calQ^\delta, \bp^{s, \delta}, \by)\le 2n|\calY|\delta\log\frac{1}{\delta.}$$

	Bringing this upper bound back to \pref{eq: p-decomposition-1} and then further back to \pref{eq: p-decomposition}, we obtain that
	$$\EE_{\by\sim \bp^{s+1, \delta}}\calR_n(\calQ^\delta, \bp^{s+1, \delta}, \by) - \EE_{\by\sim \bp^{s, \delta}}\calR_n(\calQ^\delta, \bp^{s, \delta}, \by)\le 2n|\calY|\delta\log\frac{1}{\delta}.$$
	Hence, according to \pref{eq: p-decomposition-0}, we have 
	$$\EE_{\by\sim \bp}\calR_n(\calQ^\delta, \bp, \by)\le \EE_{\by\sim \bp^\delta}\calR_n(\calQ^\delta, \bp^\delta, \by) + 2n^2|\calY|\delta\log\frac{1}{\delta}.$$
\end{proof}

\subsection{Missing Proofs in \pref{sec: proof-upper-bound-2}}\label{sec: proof-upper-bound-2-app}
\begin{proof}[Proof of \pref{prop: log-inequality}]
	We first verify the upper bounds part in \pref{eq: log-inequality}. When $0 < x\le 1$, using the inequality $\log(1 + t)\le t - t^2/2$ which holds for any $-1 < t\le 0$, we have for $n\ge 7$,
	$$\log x = 2\log \left(\sqrt{x}\right)\le 2\left(\sqrt{x} - 1\right) - \left(\sqrt{x} - 1\right)^2 = \zeta(x) - \frac{1}{4}\zeta(x)^{2}\le \zeta(x) - \frac{1}{2\log (n|\calY|)}\zeta(x)^{2}.$$
	For $x > 1$, we first notice that function
	$$\xi(x) = \frac{2\log((x+1)/2) - \log(x)}{\log^2((x+1)/2)}.$$
	is a monotonically decreasing function on $[0, \infty)$, and for every $n\ge 7$ we have
	$$\xi(n^2|\calY|) = \frac{2\log((n^2|\calY|+1)/2) - \log(n^2|\calY|)}{\log^2((n^2|\calY|+1)/2)}\ge \frac{1}{2\log(n^2|\calY|)}\ge \frac{1}{4\log (n|\calY|)},$$
	which implies
	$$\xi(x)\ge \xi(n^2|\calY|)\ge \frac{1}{4\log (n|\calY|)}, \quad \forall x\le n^2|\calY|.$$
	Hence we obtain for any $0 < x\le n^2|\calY|$, 
	\begin{align*}
		\log x & \le \zeta(x) - \frac{1}{4\log (n|\calY|)}\zeta(x)^2.
	\end{align*}
\end{proof}

\begin{proof}[Proof of \pref{prop: symmetrization-inverse}]
	First notice that for any $x\ge 1$ we have $\zeta(x)\ge 0$, and 
	$$\log\left(\frac{x + 1}{2}\right)\le \sqrt{x} - 1.$$
	Hence, we only need to verify
	$$\EE_{y\sim p}\left[- 2\cdot \left(\sqrt{\frac{f(y)}{p(y)}} - 1\right) - \frac{1}{4\log (n|\calY|)}\cdot \left(2\cdot \left(\sqrt{\frac{f(y)}{p(y)}} - 1\right)\right)^2\right]\ge 0.$$
	This can be verified by
	\begin{align*} 
		& \hspace{-0.5cm}\EE_{y\sim p}\left[- \sqrt{\frac{f(y)}{p(y)}} + 1 - \frac{1}{2\log (n|\calY|)}\cdot \left(\sqrt{\frac{f(y)}{p(y)}} - 1\right)^2\right] \\
		& \ge \EE_{y\sim p}\left[- \sqrt{\frac{f(y)}{p(y)}} + 1 - \frac{1}{2}\cdot \left(\sqrt{\frac{f(y)}{p(y)}} - 1\right)^2\right]\\
		& = \sum_{y\in \calY}\left[-\sqrt{f(y)p(y)} + p(y) - \frac{1}{2}f(y) + \sqrt{f(y)p(y)} - \frac{1}{2}p(y)\right]\\
		& = 0.
	\end{align*}
\end{proof}

\begin{proof}[Proof of \pref{prop: lipschitz-property}]
	We only need to verify that the function 
	$$h(t) = \begin{cases}
		2(t-1) &\quad \text{if }0 < t\le 1\\
		2\log\left(\frac{1+t^2}{2}\right) &\quad \text{if }t > 1
	\end{cases}$$
	is a Lipschitz function with Lipschitz constant $2$. This can be seen from
	$$\frac{d h(t)}{dt} = \begin{cases} 2 &\quad \text{if } 0 < t\le 1,\\
		\frac{4t}{1+t^2} &\quad \text{if }t > 1,\end{cases}$$
	which satisfies $\left|\frac{dh(t)}{dt}\right|\le 2$ for any $t > 0$.
\end{proof}

\begin{proof}[Proof of \pref{lem: symmetrization}]
	Fix distribution $\bp\in \Delta_n(\calY^n)$, and suppose random variables $(\bepsilon, \bw, \by, \bz)\sim \odot \bp$. Noticing that the marginal distribution of $\bw$ is $\bp$, we have
	\begin{align*}
		\EE_{\by\sim \bp}\calR_n(\calQ, \bp, \by) & = \EE_{\bw\sim \bp}\left[\sup_{\bq\in\calQ}\sum_{t=1}^{n}\left(\log q_t(w_{t}\mid \bw) - \log p_{t}(w_t\mid \bw)\right)\right]\\
		&\le \EE_{\bw\sim \bp}\left[\sup_{\bq\in\calQ}\sum_{t=1}^{n}\left\{\zeta\left(\frac{q_t(w_{t}\mid \bw)}{p_{t}(w_t\mid \bw)}\right) - \frac{1}{4\log (n|\calY|)}\zeta\left(\frac{q_t(w_{t}\mid \bw)}{p_{t}(w_t\mid \bw)}\right)^2\right\}\right], \numberthis \label{eq: first-step}
	\end{align*}
	where the last steps follows from \pref{prop: log-inequality}, and the fact that $\bp\in \Delta_n(\calY^n)$ and $\calQ\subseteq \Delta(\calY^n)$. Next, we define random variables $\bv = (v_{1:n})$ coupled with random variables $(\bepsilon, \bw, \by, \bz)\sim \odot \bp$, in the way that 
	$$v_t = z_t \text{ if }\epsilon_t = 1\quad\text{and}\quad v_t = y_t \text{ if }\epsilon_t = -1.$$
	Then the marginal distribution of $v_t$ conditioned on $w_{1:t-1}$ is $p_t(\cdot\mid \bw)$. Hence \pref{prop: symmetrization-inverse} gives that for any $\bq\in \calQ$ and $t\in [n]$, 
	\begin{equation}\label{eq: z-t}\EE\left[-\zeta\left(\frac{q_t(v_{t}\mid \bw)}{p_{t}(v_t\mid \bw)}\right) - \frac{1}{4\log (n|\calY|)}\zeta\left(\frac{q_t(v_{t}\mid \bw)}{p_{t}(v_t\mid \bw)}\right)^2\ \Big{|}\ w_{1:t-1}\right]\ge 0.
	\end{equation}
	Hence we can further upper bound
	\begin{align*}
		&\hspace{-0.2cm} \EE\left[\sup_{\bq\in\calQ}\sum_{t=1}^{n}\left\{\zeta\left(\frac{q_t(w_{t}\mid \bw)}{p_{t}(w_t\mid \bw)}\right) - \frac{1}{4\log (n|\calY|)}\zeta\left(\frac{q_t(w_{t}\mid \bw)}{p_{t}(w_t\mid \bw)}\right)^2\right\}\right]\\
		&\stackrel{(i)}{\le} \EE\left[\sup_{\bq\in\calQ}\sum_{t=1}^{n}\left\{\zeta\left(\frac{q_t(w_{t}\mid \bw)}{p_{t}(w_t\mid \bw)}\right) - \frac{1}{4\log (n|\calY|)}\zeta\left(\frac{q_t(w_{t}\mid \bw)}{p_{t}(w_t\mid \bw)}\right)^{2}\right\}\right]\\
		&\quad + \sum_{t=1}^n\EE\left[-\zeta\left(\frac{q_t(v_{t}\mid \bw)}{p_{t}(v_t\mid \bw)}\right) - \frac{1}{4\log (n|\calY|)}\zeta\left(\frac{q_t(v_{t}\mid \bw)}{p_{t}(v_t\mid \bw)}\right)^2\ \Bigg{|}\ w_{1:t-1}\right]\Bigg]\\
		&\stackrel{(ii)}{\le} \EE\Bigg[\sup_{\bq\in\calQ}\sum_{t=1}^{n}\Bigg\{\zeta\left(\frac{q_t(w_{t}\mid \bw)}{p_{t}(w_t\mid \bw)}\right) - \zeta\left(\frac{q_t(v_{t}\mid \bw)}{p_{t}(v_t\mid \bw)}\right) \\
		&\qquad- \frac{1}{4\log (n|\calY|)}\zeta\left(\frac{q_t(w_{t}\mid \bw)}{p_{t}(w_t\mid \bw)}\right)^{2} - \frac{1}{4\log (n|\calY|)}\zeta\left(\frac{q_t(v_{t}\mid \bw)}{p_{t}(v_t\mid \bw)}\right)^2\Bigg\}\Bigg], \numberthis \label{eq: second-step}
	\end{align*}
	where in $(i)$ we use \pref{eq: z-t}, in $(ii)$ we use the Jensen's inequality. According to the construction of random variables $\bepsilon, \by, \bz, \bw, \bv$, we have
	\begin{align*} 
		\zeta\left(\frac{q_t(w_{t}\mid \bw)}{p_{t}(w_t\mid \bw)}\right) - \zeta\left(\frac{q_t(v_{t}\mid \bw)}{p_{t}(v_t\mid \bw)}\right) = \epsilon_t \zeta\left(\frac{q_t(y_{t}\mid \bw)}{p_{t}(y_t\mid \bw)}\right) - \epsilon_t\zeta\left(\frac{q_t(z_{t}\mid \bw)}{p_{t}(z_t\mid \bw)}\right)
	\end{align*}
	and 
	\begin{align*} 
		\zeta\left(\frac{q_t(w_{t}\mid \bw)}{p_{t}(w_t\mid \bw)}\right)^2 + \zeta\left(\frac{q_t(v_{t}\mid \bw)}{p_{t}(v_t\mid \bw)}\right)^2 = \zeta\left(\frac{q_t(y_{t}\mid \bw)}{p_{t}(y_t\mid \bw)}\right)^2 + \zeta\left(\frac{q_t(z_{t}\mid \bw)}{p_{t}(z_t\mid \bw)}\right)^2.
	\end{align*}
	Bringing this back to \pref{eq: second-step} and further back to \pref{eq: first-step}, we obtain that 
	\begin{align*}
		&\hspace{-0.5cm} \EE_{\by\sim \bp}\calR_n(\calQ, \bp, \by)\\
		&\le \EE\Bigg[\sup_{\bq\in\calQ}\sum_{t=1}^{n}\Bigg\{\epsilon_t\zeta\left(\frac{q_t(y_{t}\mid \bw)}{p_{t}(y_t\mid \bw)}\right) - \epsilon_t\zeta\left(\frac{q_t(z_{t}\mid \bw)}{p_{t}(z_t\mid \bw)}\right) \\
		&\qquad- \frac{1}{4\log (n|\calY|)}\zeta\left(\frac{q_t(y_{t}\mid \bw)}{p_{t}(y_t\mid \bw)}\right)^{2} - \frac{1}{4\log (n|\calY|)}\zeta\left(\frac{q_t(z_{t}\mid \bw)}{p_{t}(z_t\mid \bw)}\right)^2\Bigg\}\Bigg]\\
		&\le \EE\left[\sup_{\bq\in\calQ}\sum_{t=1}^{n}\left\{\epsilon_t\zeta\left(\frac{q_t(y_{t}\mid \bw)}{p_{t}(y_t\mid \bw)}\right) - \frac{1}{4\log (n|\calY|)}\zeta\left(\frac{q_t(y_{t}\mid \bw)}{p_{t}(y_t\mid \bw)}\right)^{2}\right\}\right]\\
		&\quad + \EE\left[\sup_{\bq\in\calQ}\sum_{t=1}^{n}\left\{(- \epsilon_{t})\zeta\left(\frac{q_t(z_{t}\mid \bw)}{p_{t}(z_t\mid \bw)}\right) - \frac{1}{4\log (n|\calY|)}\zeta\left(\frac{q_t(z_{t}\mid \bw)}{p_{t}(z_t\mid \bw)}\right)^{2}\right\}\right],
	\end{align*}
	where the last inequality uses Jensen's inequality.
\end{proof}

\subsection{Missing Proofs in \pref{sec: proof-upper-bound-3}}\label{sec: proof-upper-bound-3-app}
\begin{proof}[Proof of \pref{lem: another-cover}]
	We fix $\bp\in \Delta(\calY^n)$. For $\bq\in\calQ$, $\bv'\in \calV(\alpha)$ and $\bw, \by\in\calY^n$, if we have
    $$\sum_{t=1}^n\zeta\left(\frac{q_t(y_{t}\mid \bw)}{p_t(y_t\mid \bw)}\right)^{2}\ge \frac{1}{4}\sum_{t=1}^{n}\zeta\left(\frac{v_t'(y_t\mid \bw)}{p_t(y_t\mid \bw)}\right)^{2}$$
	then we let $\bv = \bv'$ and it is easy to see that \pref{eq: condition-one} and \pref{eq: condition-two} both hold. Next we assume
	\begin{equation}\label{eq: p-q-v}
		\sum_{t=1}^n \zeta\left(\frac{q_t(y_{t}\mid \bw)}{p_t(y_t\mid \bw)}\right)^{2} < \frac{1}{4}\sum_{t=1}^n \zeta\left(\frac{v_t'(y_t\mid \bw)}{p_t(y_t\mid \bw)}\right)^{2}.
	\end{equation}
	With $\bv = \bp\in \calV(\alpha)\cup\{\bp\}$ we will verify \pref{eq: condition-one} and \pref{eq: condition-two}. First, since $\zeta(1) = 0$, we have 
	$$\sum_{t=1}^n \zeta\left(\frac{q_t(y_{t}\mid \bw)}{p_t(y_t\mid \bw)}\right)^{2}\ge 0 = \frac{1}{4}\sum_{t=1}^n \zeta\left(\frac{v_t(y_t\mid \bw)}{p_t(y_t\mid \bw)}\right)^{2},$$
	hence \pref{eq: condition-two} holds. Next, according to \pref{eq: p-q-v} and Cauchy-Schwarz inequality, 
	$$\left(\sum_{t=1}^n \zeta\left(\frac{q_t(y_{t}\mid \bw)}{p_t(y_t\mid \bw)}\right)^{2}\right)\left(\sum_{t=1}^n \zeta\left(\frac{v_t'(y_{t}\mid \bw)}{p_t(y_t\mid \bw)}\right)^{2}\right)\ge \left(\sum_{t=1}^n \zeta\left(\frac{q_t(y_{t}\mid \bw)}{p_t(y_t\mid \bw)}\right)\zeta\left(\frac{v_t'(y_{t}\mid \bw)}{p_t(y_t\mid \bw)}\right)\right)^2,$$
	we have
	$$\sum_{t=1}^n \zeta\left(\frac{v_t'(y_{t}\mid \bw)}{p_t(y_t\mid \bw)}\right)^2 \ge 2\sum_{t=1}^n\zeta\left(\frac{q_t(y_{t}\mid \bw)}{p_t(y_t\mid \bw)}\right)\zeta\left(\frac{v_t'(y_{t}\mid \bw)}{p_t(y_t\mid \bw)}\right),$$
	which implies that
	\begin{align*}
		\sum_{t=1}^n\left(\zeta\left(\frac{q_t(y_t\mid \bw)}{p_t(y_t\mid \bw)}\right) - \zeta\left(\frac{v_t'(y_t\mid \bw)}{p_t(y_t\mid \bw)}\right)\right)^2 & \ge \sum_{t=1}^n\zeta\left(\frac{q_t(y_t\mid \bw)}{p_t(y_t\mid \bw)}\right)^2\\
		& = \sum_{t=1}^n\left(\zeta\left(\frac{q_t(y_t\mid \bw)}{p_t(y_t\mid \bw)}\right) - \zeta\left(\frac{v_t(y_t\mid \bw)}{p_t(y_t\mid \bw)}\right)\right)^2,
	\end{align*}
	hence \pref{eq: condition-one} holds. 
\end{proof}

\begin{proof}[Proof of \pref{lem: chaining}]
    According to our choice of $\bv[\bq, \bp, \bw, \by, \alpha_N]\in \calV(\alpha_N)\cup\{\bp\}$, \pref{eq: condition-two} holds with $\bv = \bv[\bq, \bp, \bw, \by, \alpha_N]$. Hence, we can upper bound the left hand side of \pref{eq:three_term_decomp} as follows:
    \begin{align*}
		& \hspace{-0.5cm} \sup_{\bq\in\calQ}\sum_{t=1}^{n}\left\{\epsilon_{t}\zeta\left(\frac{q_t(y_{t}\mid \bw)}{p_t(y_t\mid \bw)}\right) - \frac{1}{4\log (n|\calY|)}\zeta\left(\frac{q_t(y_{t}\mid \bw)}{p_t(y_t\mid \bw)}\right)^{2}\right\}\\
		& = \sup_{\bq\in \calQ}\Bigg\{\sum_{t=1}^n\epsilon_{t}\left\{\zeta\left(\frac{q_t(y_{t}\mid \bw)}{p_t(y_t\mid \bw)}\right) - \zeta\left(\frac{v_{t}[\bq, \bp, \bw, \by, \alpha_N](y_{t}\mid \bw)}{p_t(y_t\mid \bw)}\right)\right\}\\
		&\qquad + \sum_{t=1}^n\left\{\epsilon_t \zeta\left(\frac{v_{t}[\bq, \bp, \bw, \by, \alpha_N](y_{t}\mid \bw)}{p_t(y_t\mid \bw)}\right) - \frac{1}{16\log (n|\calY|)}\zeta\left(\frac{v_{t}[\bq, \bp, \bw, \by, \alpha_N](y_{t}\mid \bw)}{p_t(y_t\mid \bw)}\right)^2\right\}\\
		&\qquad + \frac{1}{16\log (n|\calY|)}\sum_{t=1}^n \zeta\left(\frac{v_{t}[\bq, \bp, \bw, \by, \alpha_N](y_{t}\mid \bw)}{p_t(y_t\mid \bw)}\right)^2 - \frac{1}{4\log (n|\calY|)}\sum_{t=1}^n \zeta\left(\frac{q_t(y_t\mid \bw)}{p_t(y_t\mid \bw)}\right)^2\Bigg\}\\
		& \stackrel{(i)}{\le} \sup_{\bq\in \calQ}\Bigg\{\sum_{t=1}^n\epsilon_{t}\left\{\zeta\left(\frac{q_t(y_{t}\mid \bw)}{p_t(y_t\mid \bw)}\right) - \zeta\left(\frac{v_{t}[\bq, \bp, \bw, \by, \alpha_N](y_{t}\mid \bw)}{p_t(y_t\mid \bw)}\right)\right\}\\
		&\qquad + \sum_{t=1}^n\left\{\epsilon_t \zeta\left(\frac{v_{t}[\bq, \bp, \bw, \by, \alpha_N](y_{t}\mid \bw)}{p_t(y_t\mid \bw)}\right) - \frac{1}{16\log (n|\calY|)}\zeta\left(\frac{v_{t}[\bq, \bp, \bw, \by, \alpha_N](y_{t}\mid \bw)}{p_t(y_t\mid \bw)}\right)^2\right\}\Bigg\}\\
		& \stackrel{(ii)}{\le} \sup_{\bq\in \calQ}\left\{\sum_{t=1}^n\epsilon_{t}\left\{\zeta\left(\frac{q_t(y_{t}\mid \bw)}{p_t(y_t\mid \bw)}\right) - \zeta\left(\frac{v_{t}[\bq, \bp, \bw, \by, \alpha_N](y_{t}\mid \bw)}{p_t(y_t\mid \bw)}\right)\right\}\right\}\\
		&\qquad + \sup_{\bv\in \calV(\alpha_N)\cup\{\bp\}}\left\{\sum_{t=1}^n\left\{\epsilon_t \zeta\left(\frac{v_{t}(y_{t}\mid \bw)}{p_t(y_t\mid \bw)}\right) - \frac{1}{16\log (n|\calY|)}\zeta\left(\frac{v_{t}(y_{t}\mid \bw)}{p_t(y_t\mid \bw)}\right)^2\right\}\right\}, \numberthis\label{eq: chaining-proof-eq}
	\end{align*}
	where $(i)$ uses the condition \pref{eq: condition-two}, and $(ii)$ uses the Jensen's inequality and the chioce $\bv[\bq, \bp, \bw, \by, \alpha_N]\in \calV(\alpha_N)\cup\{\bp\}$. Next, we introduce $\bv[\bq, \bp, \bw, \by, \alpha_i]$, and further upper bound the first term above via telescoping:
    \begin{align*}
        & \hspace{-0.5cm} \EE\left[\sup_{\bq\in \calQ}\sum_{t=1}^n\epsilon_{t}\left\{\zeta\left(\frac{q_t(y_{t}\mid \bw)}{p_t(y_t\mid \bw)}\right) - \zeta\left(\frac{v_{t}[\bq, \bp, \bw, \by, \alpha_N](y_{t}\mid \bw)}{p_t(y_t\mid \bw)}\right)\right\}\right]\\
        & = \EE\Bigg[\sup_{\bq\in \calQ}\Bigg\{\sum_{t=1}^n\epsilon_{t}\left\{\zeta\left(\frac{q_t(y_{t}\mid \bw)}{p_t(y_t\mid \bw)}\right) - \zeta\left(\frac{v_{t}[\bq, \bp, \bw, \by, \alpha_1](y_{t}\mid \bw)}{p_t(y_t\mid \bw)}\right)\right\} \notag\\
	&\qquad + \sum_{i=1}^{N-1} \sum_{t=1}^n\epsilon_{t}\left\{\zeta\left(\frac{v_{t}[\bq, \bp, \bw, \by, \alpha_i](y_{t}\mid \bw)}{p_t(y_t\mid \bw)}\right) - \zeta\left(\frac{v_{t}[\bq, \bp, \bw, \by, \alpha_{i+1}](y_{t}\mid \bw)}{p_t(y_t\mid \bw)}\right)\right\}\Bigg\}\Bigg]\\
        & \le \EE\left[\sup_{\bq\in \calQ}\sum_{t=1}^n\epsilon_{t}\left\{\zeta\left(\frac{q_t(y_{t}\mid \bw)}{p_t(y_t\mid \bw)}\right) - \zeta\left(\frac{v_{t}[\bq, \bp, \bw, \by, \alpha_1](y_{t}\mid \bw)}{p_t(y_t\mid \bw)}\right)\right\}\right] \notag\\
	&\qquad + \sum_{i=1}^{N-1} \EE\left[\sup_{\bq\in \calQ}\sum_{t=1}^n\epsilon_{t}\left\{\zeta\left(\frac{v_{t}[\bq, \bp, \bw, \by, \alpha_i](y_{t}\mid \bw)}{p_t(y_t\mid \bw)}\right) - \zeta\left(\frac{v_{t}[\bq, \bp, \bw, \by, \alpha_{i+1}](y_{t}\mid \bw)}{p_t(y_t\mid \bw)}\right)\right\}\right],
    \end{align*}
    where the last inequality is due to Jensen's inequality. Bringing this back to \pref{eq: chaining-proof-eq}, we obtain \pref{eq:three_term_decomp}.
\end{proof}

\subsection{Proof of \pref{thm: general-upper-bound}}\label{sec: proof-upper-bound-4-app}
\begin{proof}[Proof of \pref{thm: general-upper-bound}]
    First of all, according to \pref{lem: dual-form}, for any joint distribution class $\calQ\subseteq \Delta(\calY^n)$, we have 
    $$\calR_n(\calQ) = \sup_{\bp} \EE_{\by\sim \bp}\calR_n(\calQ, \bp, \by),$$
    where the supreme is taken over all $\bp\in \Delta(\calY^n)$. According to \pref{lem: q-delta} and \pref{lem: p-delta}, we have
    $$\EE_{\by\sim \bp}\left[\calR_n(\calQ, \bp, \by)\right]\le \EE_{\by\sim \bp^\delta}\left[\calR_n(\calQ^\delta, \bp^\delta, \by)\right] + 6n^2|\calY|\delta\log\frac{1}{\delta}.$$
    Choosing $\delta = 1/(n^2|\calY|)$ for $n\ge 2$, we conclude that
    $$\EE\left[\calR_n(\calQ, \bp, \by)\right]\le \EE\left[\calR_n(\calQ^{\delta}, \bp^{\delta}, \by)\right] + 12\log (n|\calY|).$$
    Hence in order to prove \pref{thm: general-upper-bound}, we only need to prove that
    \begin{equation}\label{eq: objective-after-truncation}
    	\EE_{\by\sim \bp}\left[\calR_n(\calQ, \bp, \by)\right] = \tilde{\mathcal{O}}\left(\inf_{\gamma > \delta > 0}\left\{n\delta\sqrt{|\calY|} + \sqrt{n|\calY|}\int_{\delta}^\gamma\sqrt{\calHsq(\calQ, \alpha, n)}d\alpha + \calHsq(\calQ, \gamma, n)\right\}\right)
    \end{equation}
    holds for any $\bp\in \Delta_n(\calY^n)$ and $\calQ\subseteq \Delta_n(\calY^n)$, where $\Delta_n(\calY^n)$ is defined in \pref{eq: def-delta-n}. To prove this, we first notice that according to \pref{lem: symmetrization}, for $\bp\in \Delta_n(\calY^n)$ and $\calQ\subseteq \Delta_n(\calY^n)$, we have
    \begin{align*}
		&\hspace{-0.5cm} \EE_{\by\sim \bp}\calR_n(\calQ, \bp, \by)\\
		& \le \EE\left[\sup_{\bq\in\calQ}\sum_{t=1}^{n}\epsilon_{t}\zeta\left(\frac{q_t(y_{t}\mid \bw)}{p_t(y_t\mid \bw)}\right) - \frac{1}{4\log (n|\calY|)}\zeta\left(\frac{q_t(y_{t}\mid \bw)}{p_t(y_t\mid \bw)}\right)^{2}\right]\\
		&\quad + \EE\left[\sup_{\bq\in\calQ}\sum_{t=1}^{n} (- \epsilon_{t})\zeta\left(\frac{q_t(z_t\mid \bw)}{p_t(z_t\mid \bw)}\right) - \frac{1}{4\log (n|\calY|)}\zeta\left(\frac{q_t(z_t\mid \bw)}{p_t(z_t\mid \bw)}\right)^{2}\right],
    \end{align*}
    In the following, we will upper bound the right hand side in the above formula. For convenience, we only provide upper bounds to the first term in the right hand side. The upper bound to the second term in the right hand side can be obtained similarly.
    
    Next, we choose $N$ positive real numbers $\alpha_1 < \cdots < \alpha_N$ (values to be specified later). We let $\calV(\alpha_i)$ be a smallest sequential square-root cover (as per \pref{def: sequential-covering}) at scale $\alpha_i$. For any $\bq\in\calQ$, $\bw, \by\in \calY^n$ and $t\in [n]$, we let 
	\begin{equation}\label{eq: def-v-prime}
		\bv'[\bq, \bp, \bw, \by, \alpha_N] = \argmin_{\bv\in \calV(\alpha_N)}\left\{\sum_{t=1}^n\left(\zeta\left(\frac{q_t(y_t\mid \bw)}{p_t(y_t\mid \bw)}\right) - \zeta\left(\frac{v_t(y_t\mid \bw)}{p_t(y_t\mid \bw)}\right)\right)^2\right\}.
	\end{equation}
	According to \pref{lem: another-cover}, there exists some $\bv[\bq, \bp, \bw, \by, \alpha_N]\in \calV(\alpha_N)\cup\{\bp\}$ such that 
    \begin{align*}
        &\hspace{-0.5cm} \sum_{t=1}^n\left(\zeta\left(\frac{q_t(y_t\mid \bw)}{p_t(y_t\mid \bw)}\right) - \zeta\left(\frac{v_t[\bq, \bp, \bw, \by, \alpha_N](y_t\mid \bw)}{p_t(y_t\mid \bw)}\right)\right)^2\\
        & \le \sum_{t=1}^n\left(\zeta\left(\frac{q_t(y_t\mid \bw)}{p_t(y_t\mid \bw)}\right) - \zeta\left(\frac{v_t'[\bq, \bp, \bw, \by, \alpha_N](y_t\mid \bw)}{p_t(y_t\mid \bw)}\right)\right)^2, \numberthis \label{eq: condition-one-v}
    \end{align*}
    and 
    \begin{equation}
        \sum_{t=1}^n \zeta\left(\frac{q_t(y_{t}\mid \bw)}{p_t(y_t\mid \bw)}\right)^{2}\ge \frac{1}{4}\sum_{t=1}^n \zeta\left(\frac{v_t[\bq, \bp, \bw, \by, \alpha_N](y_t\mid \bw)}{p_t(y_t\mid \bw)}\right)^{2}. \numberthis \label{eq: condition-two-v}
    \end{equation}
    both hold. For $1\le i\le N-1$, we use $\calV(\alpha_i)$ to denote the sequential cover of $\calQ$ at scale $\alpha_i$. For every $\bq\in \calQ$ and $\bw, \by\in \calY^n$, we let
	\begin{equation}\label{eq: v-f-w}
		\bv[\bq, \bp, \bw, \by, \alpha_i] = \argmin_{\bv\in \calV(\alpha_i)}\left\{\sum_{t=1}^n\left(\zeta\left(\frac{q_t(y_t \mid \bw)}{p_t(y_t\mid \bw)}\right) - \zeta\left(\frac{v_t(y_t \mid \bw)}{p_t(y_t\mid \bw)}\right)\right)^2\right\}.
	\end{equation}
    Then according to \pref{lem: chaining}, we have
    \begin{align*}
		& \hspace{-0.5cm} \EE\left[\sup_{\bq\in\calQ}\sum_{t=1}^{n}\epsilon_{t}\left\{\zeta\left(\frac{q_t(y_{t}\mid \bw)}{p_t(y_t\mid \bw)}\right) - \frac{1}{4\log (n|\calY|)}\zeta\left(\frac{q_t(y_{t}\mid \bw)}{p_t(y_t\mid \bw)}\right)^{2}\right\}\right]\\
		& \le \EE\left[\sup_{\bq\in \calQ}\left\{\sum_{t=1}^n\epsilon_{t}\left\{\zeta\left(\frac{q_t(y_{t}\mid \bw)}{p_t(y_t\mid \bw)}\right) - \zeta\left(\frac{v_{t}[\bq, \bp, \bw, \by, \alpha_1](y_{t}\mid \bw)}{p_t(y_t\mid \bw)}\right)\right\}\right\}\right]\\
		& \qquad + \EE\left[\sum_{i=1}^{N-1}\sup_{\bq\in \calQ}\left\{\sum_{t=1}^n\epsilon_{t}\left\{\zeta\left(\frac{v_{t}[\bq, \bp, \bw, \by, \alpha_i](y_{t}\mid \bw)}{p_t(y_t\mid \bw)}\right) - \zeta\left(\frac{v_{t}[\bq, \bp, \bw, \by, \alpha_{i+1}](y_{t}\mid \bw)}{p_t(y_t\mid \bw)}\right)\right\}\right\}\right]\\
		&\qquad + \EE\left[\sup_{\bv\in \calV(\alpha_N)\cup\{\bp\}}\left\{\sum_{t=1}^n\left\{\epsilon_t \zeta\left(\frac{v_{t}(y_{t}\mid \bw)}{p_t(y_t\mid \bw)}\right) - \frac{1}{16\log (n|\calY|)}\zeta\left(\frac{v_{t}(y_{t}\mid \bw)}{p_t(y_t\mid \bw)}\right)^2\right\}\right\}\right]. \numberthis\label{eq: decomposition}
    \end{align*}
    In the following, we upper bound the three terms in \pref{eq: decomposition} respectively. 
    
    We let 
	\begin{equation}\label{eq: def-v-bar}
		\bar{\bv}[\bq, \bp, \bw, \alpha_N] = \argmin_{\bv\in \calV(\alpha_N)}\left\{\max_{t\in [n]}\max_{y\in \calY} \left|\sqrt{q_t(y \mid \bw)} - \sqrt{v_t(y \mid \bw)}\right|\right\}.
	\end{equation}
	Then for $(\bepsilon, \bw, \by, \bz)\sim \odot \bp$, we have 
	\begin{align*}
		& \hspace{-0.5cm} \EE\left[\sup_{\bq\in \calQ}\sum_{t=1}^n\left(\zeta\left(\frac{q_t(y_t\mid \bw)}{p_t(y_t\mid \bw)}\right) - \zeta\left(\frac{v_t[\bq, \bp, \bw, \by, \alpha_N](y_t\mid \bw)}{p_t(y_t\mid \bw)}\right)\right)^2\right]\\
		& \stackrel{(i)}{\le} \EE\left[\sup_{\bq\in \calQ}\sum_{t=1}^n\left(\zeta\left(\frac{q_t(y_t\mid \bw)}{p_t(y_t\mid \bw)}\right) - \zeta\left(\frac{v_t'[\bq, \bp, \bw, \by, \alpha_N](y_t\mid \bw)}{p_t(y_t\mid \bw)}\right)\right)^2\right]\\
		& \stackrel{(ii)}{\le} \EE\left[\sup_{\bq\in \calQ}\sum_{t=1}^n\left(\zeta\left(\frac{q_t(y_t\mid \bw)}{p_t(y_t\mid \bw)}\right) - \zeta\left(\frac{\bar{v}_t[\bq, \bp, \bw, \alpha_N](y_t\mid \bw)}{p_t(y_t\mid \bw)}\right)\right)^2\right]\\
		& \stackrel{(iii)}{\le} 4\EE\left[\sup_{\bq\in \calQ}\sum_{t=1}^n\left(\sqrt{\frac{q_t(y_t\mid \bw)}{p_t(y_t\mid \bw)}} - \sqrt{\frac{\bar{v}_t[\bq, \bp, \bw, \alpha_N](y_t\mid \bw)}{p_t(y_t\mid \bw)}}\right)^2\right]\\
		& \stackrel{(iv)}{\le} 4\EE\left[\sup_{\bq\in \calQ}\sum_{t=1}^n \frac{\alpha_N^2}{p_t(y_t\mid \bw)}\right] = 4\alpha_N^2\cdot \EE\left[\sum_{t=1}^n \frac{1}{p_t(y_t\mid \bw)}\right]\\
		& \stackrel{(v)}{\le} 4n\alpha_N^2 |\calY|, \numberthis \label{eq: property-v-n}
	\end{align*}
	where $(i)$ uses \pref{eq: condition-one-v}, $(ii)$ uses the definition of $\bv'[\bq, \bp, \bw, \by, \alpha_N]$ in \pref{eq: def-v-prime}, $(iii)$ uses the Lipschitz property of $\zeta$ function in \pref{prop: lipschitz-property}, $(iv)$ uses the definition of $\bar{\bv}$ in \pref{eq: def-v-bar} and \pref{def: sequential-covering}: for fixed $\bp, \bq$ and $\bw$, for any $t\in [n]$ and $y_t\in \calY$,
	\begin{align*}
		&\hspace{-0.5cm} \left|\sqrt{q_t(y_t\mid \bw)} - \sqrt{\bar{v}_t[\bq, \bp, \bw, \alpha_N](y_t\mid \bw)}\right|\\
		& = \min_{\bv\in \calV(\alpha_N)}\max_{t\in [n]}\max_{y\in \calY} \left|\sqrt{q_t(y_t\mid \bw)} - \sqrt{v_t(y_t\mid \bw)}\right|\\
		& \le \sup_{\bq\in \calQ} \max_{\bw\in \calY^n} \min_{\bv\in \calV(\alpha_N)}\max_{t\in [n]}\max_{y\in \calY} \left|\sqrt{q_t(y_t\mid \bw)} - \sqrt{v_t(y_t\mid \bw)}\right|\le \alpha_N,
	\end{align*}
	and finally $(v)$ uses the fact that for any $t\in [n]$, conditionally on $w_{1:t-1}$, the  distribution of $y_t$ is $p_t(\cdot\mid \bw)$, hence
	$$\EE\left[\frac{1}{p_t(y_t\mid \bw)}\right] = \EE\left[\sum_{y\in \calY} p_t(y\mid \bw)\cdot \frac{1}{p_t(y\mid \bw)}\right] = |\calY|.$$
    Then similar to \pref{eq: property-v-n} (the definition of $\bv[\bq, \bp, \bw, \by, \alpha_i]$ for $1\le i\le N-1$ is similar to the definition of $\bv'[\bq, \bp, \bw, \by, \alpha_N]$, hence the following inequality can be obtained by starting from the second line of \pref{eq: property-v-n}), we can show that with $(\bepsilon, \bw, \by, \bz)\sim \odot\bp$, for any $i\in [N-1]$,
	\begin{align*}
		\EE\left[\sup_{\bq\in \calQ}\sum_{t=1}^n\left(\zeta\left(\frac{q_t(y_t\mid \bw)}{p_t(y_t\mid \bw)}\right) - \zeta\left(\frac{v_t[\bq, \bp, \bw, \by, \alpha_i](y_t\mid \bw)}{p_t(y_t\mid \bw)}\right)\right)^2\right]\le 4n\alpha_i^2|\calY|. \numberthis\label{eq: property-v-i}
	\end{align*}

We are now ready to provide upper bounds for the three terms in \pref{eq: decomposition}. For the first term in \pref{eq: decomposition}, we have
\begin{align*}
	&\hspace{-0.5cm} \EE\left[\sup_{\bq\in \calQ}\left\{\sum_{t=1}^n\epsilon_{t}\left\{\zeta\left(\frac{q_t(y_{t}\mid \bw)}{p_t(y_t\mid \bw)}\right) - \zeta\left(\frac{v_{t}[\bq, \bp, \bw, \by, \alpha_1](y_{t}\mid \bw)}{p_t(y_t\mid \bw)}\right)\right\}\right\}\right]\\
	& \stackrel{(i)}{\le} \sqrt{n}\cdot\EE\left[\sup_{\bq\in \calQ}\left\{\sqrt{\sum_{t=1}^n \left(\zeta\left(\frac{q_t(y_{t}\mid \bw)}{p_t(y_t\mid \bw)}\right) - \zeta\left(\frac{v_{t}[\bq, \bp, \bw, \by, \alpha_1](y_{t}\mid \bw)}{p_t(y_t\mid \bw)}\right)\right)^2}\right\}\right]\\
	& \stackrel{(ii)}{\le} \sqrt{n}\cdot\sqrt{\EE\left[\sup_{\bq\in \calQ}\sum_{t=1}^n \left(\zeta\left(\frac{q_t(y_{t}\mid \bw)}{p_t(y_t\mid \bw)}\right) - \zeta\left(\frac{v_{t}[\bq, \bp, \bw, \by, \alpha_1](y_{t}\mid \bw)}{p_t(y_t\mid \bw)}\right)\right)^2\right]}\\
	& \stackrel{(iii)}{\le} 2n\alpha_1\sqrt{|\calY|}, \numberthis \label{eq: decomposition-1}
\end{align*}
where $(i)$ uses Cauchy-Schwarz inequality, $(ii)$ uses Jensen's inequality and $(iii)$ uses \pref{eq: property-v-i}.

We next analyze the second term in \pref{eq: decomposition}. Notice that for any $i\in [N-1]$ and $\lambda > 0$, we can decompose the second term in \pref{eq: decomposition} as 
\begin{align*}
	& \hspace{-0.5cm} \EE\left[\sup_{\bq\in \calQ}\left\{\sum_{t=1}^n\epsilon_{t}\left\{\zeta\left(\frac{v_{t}[\bq, \bp, \bw, \by, \alpha_i](y_{t}\mid \bw)}{p_t(y_t\mid \bw)}\right) - \zeta\left(\frac{v_{t}[\bq, \bp, \bw, \by, \alpha_{i+1}](y_{t}\mid \bw)}{p_t(y_t\mid \bw)}\right)\right\}\right\}\right]\\
	& \le \EE\Bigg[\sup_{\bq\in \calQ}\Bigg\{\sum_{t=1}^n\epsilon_{t}\left\{\zeta\left(\frac{v_{t}[\bq, \bp, \bw, \by, \alpha_i](y_{t}\mid \bw)}{p_t(y_t\mid \bw)}\right) - \zeta\left(\frac{v_{t}[\bq, \bp, \bw, \by, \alpha_{i+1}](y_{t}\mid \bw)}{p_t(y_t\mid \bw)}\right)\right\}\\
	&\qquad - \lambda\cdot \left\{\zeta\left(\frac{v_{t}[\bq, \bp, \bw, \by, \alpha_i](y_{t}\mid \bw)}{p_t(y_t\mid \bw)}\right) - \zeta\left(\frac{v_{t}[\bq, \bp, \bw, \by, \alpha_{i+1}](y_{t}\mid \bw)}{p_t(y_t\mid \bw)}\right)\right\}^2\Bigg\}\Bigg]\\
	&\quad + \lambda\cdot \EE\left[\sup_{\bq\in \calQ}\sum_{t=1}^n\left\{\zeta\left(\frac{v_{t}[\bq, \bp, \bw, \by, \alpha_i](y_{t}\mid \bw)}{p_t(y_t\mid \bw)}\right) - \zeta\left(\frac{v_{t}[\bq, \bp, \bw, \by, \alpha_{i+1}](y_{t}\mid \bw)}{p_t(y_t\mid \bw)}\right)\right\}^2\right] \numberthis \label{eq: decomposition-lambda}
\end{align*}
For fixed $\bp$ and $i\in [N-1]$, we define set $\calU_i$ as 
\begin{align*}
	\calU_i \coloneqq \{\bu: \bu = \bu[\bv^i, \bv^{i+1}]\text{ for some }\bv^i\in \calV(\alpha_i)\text{ and }\bv^{i+1}\in \calV(\alpha_{i+1})\cup\{\bp\}\},
\end{align*}
where for $\bv^i\in \calV(\alpha_i)$ and $\bv^{i+1}\in \calV(\alpha_{i+1})\cup\{\bp\}$, the element $\bu[\bv^i, \bv^{i+1}]$ is defined as $(u_{1:n})$ with $u_t = u_t(\cdot\mid \cdot): \calY\times\calY^{t-1}\to \RR$ defined as
$$u_t(y_t\mid \bw) = \zeta\left(\frac{v_{t}^i(y_{t}\mid \bw)}{p_t(y_t\mid \bw)}\right) - \zeta\left(\frac{v_{t}^{i+1}(y_{t}\mid \bw)}{p_t(y_t\mid \bw)}\right),\qquad \forall\ \bw, \by\in \calY^n \text{ and }t\in [n].$$
Then we have $|\calU| = |\calV(\alpha_i)|\cdot (|\calV(\alpha_{i+1})| + 1)$, and we can further upper bound \pref{eq: decomposition-lambda} by 
\begin{align*}
	& \hspace{-0.5cm} \EE\left[\sup_{\bq\in \calQ}\left\{\sum_{t=1}^n\epsilon_{t}\left\{\zeta\left(\frac{v_{t}[\bq, \bp, \bw, \by, \alpha_i](y_{t}\mid \bw)}{p_t(y_t\mid \bw)}\right) - \zeta\left(\frac{v_{t}[\bq, \bp, \bw, \by, \alpha_{i+1}](y_{t}\mid \bw)}{p_t(y_t\mid \bw)}\right)\right\}\right\}\right]\\
	& \le \EE\left[\sup_{\bu\in \calU}\sum_{t=1}^n\big\{\epsilon_{t}u_t(y_t\mid \bw) - \lambda\cdot u_t(y_t\mid \bw)^2\big\}\right]\\
	&\quad + \lambda\cdot \EE\left[\sup_{\bq\in \calQ}\sum_{t=1}^n\left\{\zeta\left(\frac{v_{t}[\bq, \bp, \bw, \by, \alpha_i](y_{t}\mid \bw)}{p_t(y_t\mid \bw)}\right) - \zeta\left(\frac{v_{t}[\bq, \bp, \bw, \by, \alpha_{i+1}](y_{t}\mid \bw)}{p_t(y_t\mid \bw)}\right)\right\}^2\right] \numberthis \label{eq: decomposition-lambda-1}
\end{align*}
For the first term in \pref{eq: decomposition-lambda-1}, we adopt \pref{lem: offset-lemma} with 
\begin{align*} 
	s_t = (w_{1:t-1}, y_t), \quad \calG_t = \sigma(y_{1:t}, z_{1:t}, w_{1:t-1})\quad\text{and}\quad a_t(s_t) = u_t(y_t\mid \bw),
\end{align*}
it is easy to see that $\EE[\epsilon_t\mid \calG_t] = 0$, and $\sigma(s_t)\subseteq \calG_t$. Hence we have 
\begin{align*}
	&\hspace{-0.5cm} \EE\left[\sup_{\bu\in \calU}\sum_{t=1}^n\big\{\epsilon_{t}u_t(y_t\mid \bw) - \lambda\cdot u_t(y_t\mid \bw)^2\big\}\right]\\
	& \le \frac{\log |\calU|}{2\lambda}\le \frac{\log (|\calV(\alpha_i)|(|\calV(\alpha_{i+1})|+1))}{2\lambda}\le \frac{2\log |\calV(\alpha_i)|}{\lambda}, \numberthis \label{eq: decomposition-lambda-first}
\end{align*}
where the last inequality uses the fact that $\alpha_i\le \alpha_{i+1}$ hence $|\calV(\alpha_{i+1})|\le |\calV(\alpha_i)|$. For the second term in \pref{eq: decomposition-lambda-1}, we have
\begin{align*}
	&\hspace{-0.5cm} \lambda\cdot \EE\left[\sup_{\bq\in \calQ}\sum_{t=1}^n\left\{\zeta\left(\frac{v_{t}[\bq, \bp, \bw, \by, \alpha_i](y_{t}\mid \bw)}{p_t(y_t\mid \bw)}\right) - \zeta\left(\frac{v_{t}[\bq, \bp, \bw, \by, \alpha_{i+1}](y_{t}\mid \bw)}{p_t(y_t\mid \bw)}\right)\right\}^2\right]\\
	& \stackrel{(i)}{\le} 2\lambda\cdot \EE\left[\sup_{\bq\in \calQ}\sum_{t=1}^n\left(\zeta\left(\frac{q_t(y_t\mid \bw)}{p_t(y_t\mid \bw)}\right) - \zeta\left(\frac{v_t[\bq, \bp, \bw, \by, \alpha_i](y_t\mid \bw)}{p_t(y_t\mid \bw)}\right)\right)^2\right]\\
	&\qquad + 2\lambda\cdot \EE\left[\sup_{\bq\in \calQ}\sum_{t=1}^n\left(\zeta\left(\frac{q_t(y_t\mid \bw)}{p_t(y_t\mid \bw)}\right) - \zeta\left(\frac{v_t[\bq, \bp, \bw, \by, \alpha_{i+1}](y_t\mid \bw)}{p_t(y_t\mid \bw)}\right)\right)^2\right]\\
	& \stackrel{(ii)}{\le} 8\lambda n\alpha_i^2|\calY| + 8\lambda n\alpha_{i+1}^2|\calY| \stackrel{(iii)}{\le} 16\lambda n\alpha_{i+1}^2|\calY|, \numberthis\label{eq: decomposition-lambda-second}
\end{align*}
where $(i)$ we uses Cauchy-Schwarz inequality, $(ii)$ we uses \pref{eq: property-v-n} and \pref{eq: property-v-i}, and $(iii)$ uses $\alpha_i\le \alpha_{i+1}$. Finally we choose $\lambda = \sqrt{\frac{\log|\calV(\alpha_i)|}{8n\alpha_{i+1}^2|\calY|}}$. Then bringing \pref{eq: decomposition-lambda-second} and \pref{eq: decomposition-lambda-first} into \pref{eq: decomposition-lambda-1}, we obtain
\begin{align*} 
	&\hspace{-0.5cm} \EE\left[\sup_{\bq\in \calQ}\left\{\sum_{t=1}^n\epsilon_{t}\left\{\zeta\left(\frac{v_{t}[\bq, \bp, \bw, \by, \alpha_i](y_{t}\mid \bw)}{p_t(y_t\mid \bw)}\right) - \zeta\left(\frac{v_{t}[\bq, \bp, \bw, \by, \alpha_{i+1}](y_{t}\mid \bw)}{p_t(y_t\mid \bw)}\right)\right\}\right\}\right]\\
	& \le \frac{2\log|\calV(\alpha_i)|}{\lambda} + 16\lambda n \alpha_{i+1}^2|\calY| \le 8\alpha_{i+1}\sqrt{2n|\calY|\log |\calV(\alpha_i)|}. \numberthis \label{eq: decomposition-2}
\end{align*}
\begin{remark}
\label{rem:offset}
    Let us briefly discuss the novel and key aspect of the approach used to upper bound the left-hand side of \eqref{eq: decomposition-2}. In the classical case when the coefficients are non-random, one simply observes that the supremum is a maximum over a finite collection. Unfortunately, here the squared increments are themselves random and only small in expectation, due to the presence of the $p_t(y_t\mid\bw)$ term in the denominator. The key technical observation here is that one can alternatively work with the offset process \eqref{eq: decomposition-lambda-first}, which can be controlled for any predictable coefficients \emph{irrespective of their magnitude}, as well as the expected squared distance between the coefficients in \eqref{eq: decomposition-lambda-second}. We believe that this technique, which is summarized in \pref{lem: non-offset-lemma}, will be useful beyond this paper.
\end{remark}

Finally, we analyze the last term (offset term) in \pref{eq: decomposition}. We let the filtration $\calG_t = \sigma(y_{1:t}, z_{1:t}, \epsilon_{1:t-1})$ for $t\in [n]$. Then we have $\EE[\epsilon_t\mid \calG_t] = 0$, and according to the process of getting $w_{1:n}$, $\sigma(w_{1:t-1})\subseteq \calG_t$. We let $s_t$ in \pref{lem: offset-lemma} to be $(w_{1:t-1}, y_t)$, and
$$\calA = \left\{\ba = (a_1, \cdots, a_n)\ \Big{|}\ a_t(w_{1:t-1}, y_t) = \zeta\left(\frac{v_t(y_t\mid w_{1:t-1})}{p_t(y_t\mid w_{1:t-1})}\right)\text{ for any }t\in [n]\text{ for some }\bv\in\calV(\alpha_N)\cup\{\bp\}\right\}$$
Applying \pref{lem: offset-lemma} with $\lambda = \frac{1}{16\log (n|\calY|)}$, we obtain
\begin{align*}
	&\hspace{-0.5cm} \EE\left[\sup_{\bv\in \calV(\alpha_N)\cup\{\bp\}}\left\{\sum_{t=1}^n\left\{\epsilon_t \zeta\left(\frac{v_{t}(y_{t}\mid \bw)}{p_t(y_t\mid \bw)}\right) - \frac{1}{16\log (n|\calY|)}\zeta\left(\frac{v_{t}(y_{t}\mid \bw)}{p_t(y_t\mid \bw)}\right)^2\right\}\right\}\right]\\\
	&\le 8\log (n|\calY|)\cdot \log(|\calV(\alpha_N)| + 1).  \numberthis \label{eq: decomposition-3}
\end{align*}

Finally, we specify the scales $\alpha_i = 2^{i - l}$ for some positive integer $l\ge N$. Bringing \pref{eq: decomposition-1}, \pref{eq: decomposition-2} and \pref{eq: decomposition-3} into \pref{eq: decomposition}, we obtain that
\begin{align*}
	&\hspace{-0.5cm} \EE\left[\sup_{\bq\in\calQ}\sum_{t=1}^{n}\epsilon_{t}\left\{\zeta\left(\frac{q_t(y_{t}\mid \bw)}{p_t(y_t\mid \bw)}\right) - \frac{1}{4\log (n|\calY|)}\zeta\left(\frac{q_t(y_{t}\mid \bw)}{p_t(y_t\mid \bw)}\right)^{2}\right\}\right]\\
	&\le 2n\alpha_1\sqrt{|\calY|} + \sum_{i=1}^{N-1} 8\alpha_{i+1}\sqrt{2n|\calY|}\cdot \sqrt{\log|\calV(\alpha_i)|} + 8\log (n|\calY|)\cdot (\log|\calV(\alpha_N)| + 1)\\
	& \stackrel{(i)}{\lesssim} n\alpha_1\sqrt{|\calY|} + \sum_{i=1}^{N-1} \alpha_{i+1}\sqrt{n|\calY|}\cdot \sqrt{\calHsq(\calQ, \alpha_i, n)} + \log (n|\calY|)\cdot \calHsq(\calQ, \alpha_N, n)\\
	& \stackrel{(ii)}{\lesssim} n\alpha_1\sqrt{|\calY|} + \sum_{i=1}^{N-1} (\alpha_{i} - \alpha_{i-1})\sqrt{n|\calY|}\cdot \sqrt{\calHsq(\calQ, \alpha_i, n)} + \log (n|\calY|)\cdot \calHsq(\calQ, \alpha_N, n)\\
	& \stackrel{(iii)}{\le} \frac{n\sqrt{|\calY|}}{2^{-l}} + \sqrt{n|\calY|}\int_{2^{-l}}^{2^{N-l}}\sqrt{\calHsq(\calQ, \alpha, n)}d\alpha + \log (n|\calY|)\cdot \calHsq(\calQ, 2^{N-l}, n), 
\end{align*}
where $(i)$ uses the definition of the covering $\calV(\alpha_i)$ we have $\log (|\calV(\alpha_i)| + 1)\lesssim \calHsq(\calQ, \alpha, n)$ for any $\bp\in \Delta(\calY^n)$, $(ii)$ uses $\alpha_{i+1} = 2\alpha_i = 4\alpha_{i-1}$, and $(iii)$ uses the fact that for any $\alpha_{i-1}\le \alpha\le \alpha_i$,
$$\calHsq(\calQ, \alpha_i, n)\le \calHsq(\calQ, \alpha, n).$$
For $0 < \delta < \gamma\le 1$, letting $l = \log_2(1/\delta)$ and $N = l + \log_2(\gamma)$, and according to \pref{lem: symmetrization}, we obtain that for any $\bp\in \Delta_n(\calY^n)$,
\begin{align*}
	\EE_{\by\sim \bp}\left[\calR_n(\calQ, \bp, \by)\right] \lesssim n\delta\sqrt{|\calY|} + \sqrt{n|\calY|}\int_{\delta}^\gamma\sqrt{2\calHsq(\calQ, \alpha, n)}d\alpha + \log (n|\calY|)\cdot \calHsq(\calQ, \gamma, n),
\end{align*}
hence \pref{eq: objective-after-truncation} is verified.
\end{proof}

\section{Missing Proofs in \pref{sec: sequential-probability}}

\subsection{Missing Proofs in \pref{sec: upper-bound}}\label{sec: upper-bound-proof}
\begin{proof}[Proof of \pref{lem: tree-transductive}]
    We write the proof for 
    $$\phi(y_{1:n},x_{1:n}) = \inf_{f\in \calF} \sum_{t=1}^n \ell(f(x_t), y_t),$$
    but it will be clear from the following that this particular structure is not used. When $\ell$ is convex with respect to its first argument, we can write
    \begin{align*}
        \calR_n(\calF) & = \left\{\sup_{x_t}\inf_{\hp_t}\sup_{y_t}\right\}_{t=1}^n\left[\sum_{t=1}^n \ell(\hp_t, y_t) - \inf_{f\in \calF}\sum_{t=1}^n \ell(f(x_t), y_t)\right]\\
        & = \left\{\sup_{x_t}\inf_{\hp_t}\sup_{p_t}\EE_{y_t\sim p_t}\right\}_{t=1}^n\left[\sum_{t=1}^n \ell(\hp_t, y_t) - \inf_{f\in \calF}\sum_{t=1}^n \ell(f(x_t), y_t)\right]\\
        & \stackrel{(i)}{=} \left\{\sup_{x_t}\sup_{p_t}\inf_{\hp_t}\EE_{y_t\sim p_t}\right\}_{t=1}^n\left[\sum_{t=1}^n \ell(\hp_t, y_t) - \inf_{f\in \calF}\sum_{t=1}^n \ell(f(x_t), y_t)\right]\\
        & = \left\{\sup_{x_t}\sup_{p_t}\EE_{y_t\sim p_t}\right\}_{t=1}^n\left[\sum_{t=1}^n \inf_{\hp_t}\EE_{y_t\sim p_t} \ell(\hp_t, y_t) - \inf_{f\in \calF}\sum_{t=1}^n \ell(f(x_t), y_t)\right]\\
        & \stackrel{(ii)}{=} \sup_{\bx} \left\{\sup_{p_t}\EE_{y_t\sim p_t}\right\}_{t=1}^n\left[\sum_{t=1}^n \inf_{\hp_t}\EE_{y_t\sim p_t} \ell(\hp_t, y_t) - \inf_{f\in \calF}\sum_{t=1}^n \ell(f(x_t(\by)), y_t)\right]\\
        & = \sup_{\bx} \left\{\sup_{p_t}\inf_{\hp_t}\EE_{y_t\sim p_t}\right\}_{t=1}^n\left[\sum_{t=1}^n  \ell(\hp_t, y_t) - \inf_{f\in \calF}\sum_{t=1}^n \ell(f(x_t(\by)), y_t)\right]\\
        & \stackrel{(iii)}{=} \sup_{\bx} \left\{\inf_{\hp_t}\sup_{y_t}\right\}_{t=1}^n\left[\sum_{t=1}^n  \ell(\hp_t, y_t) - \inf_{f\in \calF}\sum_{t=1}^n \ell(f(x_t(\by)), y_t)\right]\\
        & = \sup_{\bx} \calR_n(\calF\circ \bx),
    \end{align*}
    where $(i)$ and $(iii)$ use the minimax theorem for convex-concave functions \cite[Theorem 1.(ii)]{fan1953minimax}, and also the fact that $\ell$ is convex with respect to its first argument, and $(ii)$ uses the fact that interleaving the supremum of $x_t$ and expectation over $y_t$ is equivalent to taking the supremum over trees $\bx$ first (or, skolemization).
\end{proof}

\begin{proof}[Proof of \pref{thm: upper-bound}]
	For a given function $\calF\subseteq [0, 1]^\calX$ and depth-$n$ $\calX$-valued tree $\bx$, we define a class $\calF(\bx) = \{f(\bx): f\in \calF\}\subseteq \Delta(\{0, 1\}^n)$ of joint distributions over $\{0, 1\}^n$ as follows: for $\by = (y_{1:n})\in \{0, 1\}^n$, the probability of joint distribution $f(\bx)$ takes value $\by$ equals to
	$$f(\bx)(\by) = \prod_{t=1}^n f(\bx)_t(y_t\mid \by),$$
	and 
	\begin{equation}\label{eq: def-f-x-t}
		f(\bx)_t(y_t\mid \by) = \begin{cases}
		f(x_t(\by)) &\quad \text{if } y_t = 1,\\
		1 - f(x_t(\by)) &\quad \text{if } y_t = 0.
	\end{cases}\end{equation}
	
	According to \pref{lem: tree-transductive} (see also \cite[Lemma 6]{bilodeau2020tight}, \cite{liu2024sequential} and \cite[Eq. 27] {rakhlin2015sequential}), we can write
	$$\calR_n(\calF) = \sup_{\bx, \bp} \EE_{\by\sim \bp}\calR_n(\calF(\bx), \bp, \by) = \sup_{\bx}\left[\calR_n(\calF(\bx))\right],$$
	where $\calR_n(\cdot, \bp, \by)$ is defined in \pref{eq: r-n-formula}, and $\calR_n(\calF(\bx))$ is the Shtarkov sum \pref{eq: minimax-regret-Q} for joint distribution class $\calF(\bx)$. 
    
    In order to prove \pref{thm: upper-bound}, we only need to verify that $\calHsq(\calF, \alpha, n, \bx)\le \calHsq(\calF(\bx), n, \alpha)$ for any tree $\bx$. In the following, we verify this by showing that the sequential square-root covering of $\calF\circ \bx$ defined in \pref{def: sequential-covering-f} can form a sequential square-root covering of $\calF(\bx)$ defined in \pref{def: sequential-covering} of the same size. 

	Suppose $\calV(\alpha)$ is the sequential square-root covering of $\calF\circ \bx$ at scale $\alpha$ defined in \pref{def: sequential-covering-f}. We define 
	$\calU(\alpha)\subseteq \Delta(\{0, 1\}^n)$ from $\calV$: 
	$$\calU(\alpha) = \left\{\bu[\bv]\in \Delta(\{0, 1\}^n), \bv\in\calV(\alpha): \bu(\by) = \prod_{t=1}^n u_t[\bv](y_t\mid \by), \forall \by\in \{0, 1\}^n\right\},$$
	where 
	$$u_t[\bv](y_t\mid \by)\coloneqq \begin{cases} v_t(\by)\quad \text{if }y_t = 1,\\1 - v_t(\by)\quad \text{if }y_t = 0.\end{cases}$$
	Then we have $|\calU(\alpha)| = |\calV(\alpha)|$. And according to \pref{def: sequential-covering-f} we have for any $f\in \calF$, and $\bw\in \{0, 1\}^n$, there exists $\bu\in \calU$ such that for any $t\in [n]$,
	$$\max_{y\in \{0, 1\}}\left|\sqrt{u_t(y\mid \bw)} - \sqrt{f(\bx)_t(y \mid \bw)}\right|\le \alpha.$$
	Therefore, $\calU(\alpha)$ is a sequential square-root covering of $\calF(\bx)$ according to \pref{def: sequential-covering}. Noticing that $|\calU(\alpha)| = |\calV(\alpha)|$, \pref{thm: upper-bound} directly follows from \pref{thm: general-upper-bound}.
\end{proof}

\begin{proof}[Proof of \pref{corr: entropy}]
	\pref{corr: entropy} follows from \pref{thm: upper-bound} after replacing $\calHsq(\calF, \alpha, n, \bx)$ with $\tO(\alpha^{-p})$ and the following choices of $\gamma$ and $\delta$:
	\begin{equation}
		(\gamma, \delta) = \begin{cases}
			\left(n^{-\frac{1}{p+2}}, n^{-1}\right) &\quad \text{if }0\le p\le 2,\\
			\left(1, n^{-\frac{1}{p}}\right) &\quad \text{if }p > 2.
		\end{cases}
	\end{equation}
\end{proof}

\subsection{Missing Proofs in \pref{sec: entropy-connection}}\label{sec: entropy-connection-proof}
\begin{proof}[Proof of \pref{prop: entropy-connection-1}]
	This proposition directly follows from the following inequality: for any $p, q\in [0, 1]$ with $p\in [\delta, 1 - \delta]$, 
	$$\max\left\{\left|\sqrt{p} - \sqrt{q}\right|, \left|\sqrt{1 - p} - \sqrt{1 - q}\right|\right\}\le \frac{|p - q|}{\sqrt{\delta}}.$$
	In fact, we have 
	\begin{align*} 
		&\hspace{-0.5cm} \max\left\{\left|\sqrt{p} - \sqrt{q}\right|, \left|\sqrt{1 - p} - \sqrt{1 - q}\right|\right\}\\
		&  = |p - q|\cdot \max\left\{\frac{1}{\sqrt{p} + \sqrt{q}}, \frac{1}{\sqrt{1-p} + \sqrt{1-q}}\right\}\le \frac{|p - q|}{\sqrt{\delta}}.
	\end{align*}
\end{proof}

\begin{proof}[Proof of \pref{prop: entropy-connection}]
	This proposition is a direct corollary of the standard inequality, e.g.~\cite[(7.22)]{polyanskiy2024information}, which shows that the squared Hellinger distance and TV distance satisfy the bound: 
	$$H(p, q)^2\le 2 \TV(p, q).$$
\end{proof}

\section{Missing Proofs in \pref{sec: lower-bound}}\label{sec: lower-bound-proof}
In this section, we will present the formal proof to \pref{thm: lower-bound-entropy} and \pref{thm: lower-bound-large-p}.

\subsection{Proof of \pref{thm: lower-bound-entropy}}
To prove \pref{thm: lower-bound-entropy}, we define a dimension of function classes which characterizes the difficulty of sequential learning with the class. We will relate both the sequential square-root entropy and minimax regret to this dimension of the function class. 

First, we define distance $h$ between two real numbers in $[0, 1]$:
\begin{align}
\label{def:h}
    h(a, b) = \max\left\{\left|\sqrt{a} - \sqrt{b}\right|, \left|\sqrt{1 - a} - \sqrt{1 - b}\right|\right\}, \qquad \forall a, b\in [0, 1].
\end{align}

\begin{definition}\label{def: def-fat-shattering}
	Suppose $\calF\in [0, 1]^\calX$ is a function class and $0 < \beta < \alpha$. An $\calX$-valued depth-$d$ tree $\bx$ is said to be $(\alpha, \beta)$-shattered by $\calF$ distance if there exists a $[\beta, 1 - \beta]\times [\beta, 1 - \beta]$-valued depth-$d$ tree $\bs$ such that: for any path $\by = (y_{1:d})\in \{0, 1\}^d$, $s_t(\by) = (s_t(\by)[0], s_t(\by)[1])$ with $s_t(\by)[0] < s_t(\by)[1]$, and also there exists $f^\by\in\calF$ such that 
	$$\left|f^\by(x_t(\by)) - s_t(\by)[y_t]\right| < \beta \quad \text{and}\quad h\left(s_t(\by)[0], s_t(\by)[1]\right) > \alpha,\qquad \forall t\in [d].$$
	The dimension $\mathfrak{D}(\calF, \alpha, \beta)$ is defined to be the largest $d$ such that there exists a depth-$d$ tree $\bx$ which is $(\alpha, \beta)$-shattered by $\calF$.
\end{definition}

The following proposition relates the dimension $\frakd(\calF, \alpha, \beta)$ to the sequential square-root entropy.

\begin{proposition}\label{prop: covering-fat}
	For any class $\calX$, function class $\calF\in [0, 1]^\calX$, positive integer $n$ and $\alpha > 0$, we have 
	$$\sup_{\bx}\calHsq(\calF, \alpha + \sqrt{2\beta}, n, \bx)\le \frakd(\calF, \alpha, \beta)\log\left(\frac{en}{\beta}\right),$$
	where the supremum is over all depth-$n$ $\calX$-valued tree $\bx$.
\end{proposition}

The following proposition relates the dimension $\frakd(\calF, \alpha, \beta)$ to the minimax regret $\calR_n(\calF)$.
\begin{proposition}\label{prop: lower-bound}
	Suppose the function class $\calF\subseteq [0, 1]^{\calX}$ satisfies $\frakd(\calF, \alpha, \alpha^4/16) = \tilde{\Omega}\left(\alpha^{-p}\right)$. Then for any positive integer $n$, 
	$$\calR_n(\calF) = \tilde{\Omega}\left(n^{\frac{p}{p+2}}\right).$$
\end{proposition}

With the above two propositions of the dimension $\frakd(\calF, \alpha, \beta)$, we are ready to prove \pref{thm: lower-bound-entropy}.
\begin{proof}[Proof of \pref{thm: lower-bound-entropy}]
	Suppose class $\calF$ satisfies $\sup_{\bx}\calHsq(\calF, \alpha, n, \bx) = \tilde{\Omega}\left(\alpha^{-p}\right)$. Then according to \pref{prop: covering-fat}, we have 
	$$\frakd(F, \alpha, \alpha^4/16)\ge \sup_{\bx} \calHsq(\calF, \alpha + \alpha^2/(2\sqrt{2}), n, \bx)\cdot \log^{-1}\left(\frac{16en}{\alpha^4}\right) = \tilde{\Omega}\left(\alpha^{-p}\right).$$
	Hence according to \pref{prop: lower-bound}, we have 
	$$\calR_n(\calF) = \tilde{\Omega}\left(n^{\frac{p}{p+2}}\right).$$
\end{proof}
The following two subsections will be devoted to the proof of \pref{prop: covering-fat} and \pref{prop: lower-bound}. A more general treatment of this approach, including the non-sequential analogue, will appear in the companion paper \cite{JiaPolRak25modified}.

\subsubsection{Proof of \pref{prop: covering-fat}}
Before proving \pref{prop: covering-fat}, we first prove a similar results for sets of discrete-valued function classes. Suppose $\beta\in (0, 1)$ satisfies $M = 1/(2\beta)$ is an positive integer. We define set:
$$U_\beta = \left\{\beta, 3\beta, 5\beta, \cdots, \left(2M - 1\right)\cdot \beta\right\}\subseteq [0, 1].$$
And we further define the dimension $\frakd(\calF, \alpha)$ for the discrete-valued function class $\calF$ which contains functions mapping $\calX$ into the set $U_\beta$.
\begin{definition}\label{def: shattering-finite}
	Fix real number $\beta\in (0, 1)$ which satisfies $1/(2\beta)$ is a positive integer. For function $\calF\subseteq (U_\beta)^\calX$ and real number $\alpha > 0$, a depth-$d$ $\calX$-valued $\bx$ is said to be shattered by $\calF$ at scale $\alpha$, if there exists a depth-$d$ $(\calU_\beta\times \calU_\beta)$-valued tree $\bs$ such that: for any $\by\in \{0, 1\}^d$, $s_t(\by) = (s_t(\by)[0], s_t(\by)[1])$ satisfies $s_t(\by)[0] < s_t(\by)[1]$, and for any $t\in [d]$, $h(s_t(\by)[0], s_t(\by)[1]) > \alpha$ (where $h$ is defined in \eqref{def:h}), and for any $\by\in \{0, 1\}^d$, there exists $f^\by\in \calF$ such that $f^\by(x_t(\by)) = s_t(\by)[y_t]$ holds for all $t\in [d]$.

	The dimension $\frakd(\calF, \alpha)$ of $\calF$ is defined to be the largest $d$ such that there exists a depth-$d$ $\calX$-valued tree $\bx$ shattered by $\calF$ at scale $\alpha$.
\end{definition}
The following lemma indicates that for discrete-valued function set $\calF$, the sequential square-root covering number of class $\calF$ can be bounded by the dimension $\frakd(\calF, \alpha)$ of class $\calF$.
\begin{lemma}\label{lem: shattering-finite}
	For function class $\calF: \calX\to U_\beta$, then for any depth-$n$ $\calX$-valued tree $\bx$, we have
	$$\calNsq(\calF\circ \bx, \alpha, n)\le \left(\frac{en}{\beta}\right)^{\frakd(\calF, \alpha)}$$
\end{lemma}
\begin{proof}[Proof of \pref{lem: shattering-finite}]
	For any $\beta > 0$ such that $1/(2\beta)$ is an integer, we define 
	$$g_\beta(n, d) = \sum_{i=0}^d \binom{n}{i}\cdot \left(M - 1\right)^i,$$
	which satisfies \cite{RakSri2015Steele}
	\begin{equation}\label{eq: equation-g-beta}
		g_\beta(n, d) = g_\beta(n-1, d) + \left(M - 1\right)\cdot g_\beta(n-1, d-1).
	\end{equation}

	We will prove this result by induction with the following induction argument:

	\noindent$\mathfrak{G}(n, d)$: For any function class $\calF\subseteq (U_\beta)^\calX$ with $\frakd(\calF, \alpha)\le d$, and any depth-$n$ $\calX$-valued tree $\bx$, $\calNsq(\calF\circ \bx, \alpha, n)\le g_\beta(d, n)$.
	
	\paragraph{Base: } There are two base case: $n \le d$ and $d = 0$.

	When $n\le d$, we let 
	$$\calV = \left\{\bv[l_1, l_2, \cdots, l_n]: l_1, \cdots, l_n\in U_\beta\right\},$$
	where $\bv[l_1, \cdots, l_n]$ denotes the tree which takes value $l_t$ at depth $t$ along any path. Then it is easy to see that for any $f\in \calF$, depth-$n$ $\calX$-valued tree $\bx$, and any path $\by\in \{0, 1\}^n$, there exists some $\bv\in \calV$ such that $f(x_t(\by)) = v_t(\by)$ for all $t\in [n]$. Hence $\calV$ is a $0$-sequential covering of $\calF\circ \bx$, hence $\calV$ is also an $\alpha$-sequential covering of $\calF\circ\bx$ as well. Hence we have
	$$\calNsq(\calF\circ \bx, \alpha, n)\le |\calV| = |U_\beta|^n = M^n = \sum_{i=0}^d \binom{n}{i}\cdot \left(M - 1\right)^i = g_\beta(n, d).$$

	When $d = 0$, there is no depth-$1$ $\calX$-valued tree which shatters $\calF$ at scale $\alpha$. This implies for any two $x, x'\in \calX$, we always have $h(f(x), h(x'))\le \alpha$ (otherwise we can construct depth-$1$ tree $\bx$ with $x_1(0) = x$ and $x_1(1) = x'$, then this tree is shattered by $\calF$). For any $x_0\in \calX$, we construct depth-$n$ $[0, 1]$-valued tree $\bv$ which always takes value $f(x_0)$ no matter the path and depth. Then for any $f\in \calF$, depth-$n$ $\calX$-valued tree $\bx$ and any path $\by\in \{0, 1\}$, we always have $h(f(x_t(\by)), v_t(\by)) =  h(f(x_t(\by)), f(x_0))\le \alpha$. Hence $\calV$ is an $\alpha$-sequential covering of $\calF\circ \bx$, and it satisfies $|\calV| = 1 = g_\beta(n, 0)$.

	\paragraph{Induction: } Suppose the induction hypothesis $\mathfrak{G}(n-1, d-1)$ and $\mathfrak{G}(n-1, d)$ both holds. We will prove induction statement $\mathfrak{G}(n, d)$. For fixed function class $\calF$ with $\frakd(\calF, \alpha) = d$ and depth-$n$ $\calX$-valued tree $\bx$, we will construct a $\alpha$-sequential covering to $\calF\circ \bx$ whose size is no more than $g_\beta(n, d)$. Suppose the root of tree $\bx$ is $x_1$, the left subtree of $x_1$ is $\bx^l$, and the right subtree of $x_1$ is $\bx^r$. We partition the function class $\calF$ as:
	$$\calF = \calF_1\cup\calF_2\cup\cdots\cup \calF_{1/2\beta}\quad \text{where}\quad \calF_k = \{f\in \calF: f(x_1) = (2k-1)\beta\}, \forall 1\le k\le M.$$
	Then we have $\frakd(\calF_k, \alpha)\le \frakd(\calF, \alpha) = d$ for all $k\in [M]$.  We let $\calK = \{k\in [M]: \frakd(\calF_k, \alpha) = d\}$. Then for any $a, b\in \calK$ and $a < b$, there exist two depth-$d$ $\calX$-valued trees $\bx^a$ and $\bx^b$, and also two depth-$d$ $(U_\beta\times U_\beta)$-valued trees $\bs^a$ and $\bs^b$ such that for any $\by\in \{0, 1\}^d$ and $t\in [d]$, 
    $$h(s_t^a(\by)[0], s_t^a(\by)[1]) > \alpha\quad \text{and}\quad h(s_t^b(\by)[0], s_t^b(\by)[1]) > \alpha,$$
    and further for any $\by\in \{0, 1\}^d$, there exists $f_a^\by\in \calF_a$ and $f_b^\by\in \calF_b$ such that for any $t\in [d]$,
	$$f_a^\by(x_t^a(\by)) = s_t^a(\by)[y_t]\quad \text{and}\quad f_b^\by(x_t^b(\by)) = s_t^b(\by)[y_t],$$
	If we further have $h((2a-1)\beta, (2b-1)\beta) > \alpha$, we construct a depth-$(d+1)$ $\calX$-valued tree $\bx$ with root $x_0$, left subtree of the root to be $\bx^a$, and right subtree of the root to be $\bx^b$, and also a depth-$(d+1)$ $U_\beta\times U_\beta$-valued tree $\bs$ with root $((2a-1)\beta, (2b-1)\beta)$, left subtree of the root to be $\bs^a$, and right subtree of the root to be $\bs^b$. Then we can verify that for any $\by\in \{0, 1\}^{d+1}$, and any $t\in [d+1]$, we have $s_t^b(\by)[0] < s_t^b(\by)[1]$, and $h(s_t(\by)[0], s_t(\by)[1]) > \alpha$. Further we let $\by' = (y_2, y_3, \cdots, y_{d+1})\in \{0, 1\}^d$, and if $y_1 = 0$, then by letting $f^\by = f_a^{\by'}$ we can verify that $f^\by(x_t(\by)) = s_t(\by)[y_t]$ for any $t\in [d+1]$, and if $y_1 = 1$, then by letting $f^\by = f_b^{\by'}$ we can verify that $f^\by(x_t(\by)) = s_t(\by)[y_t]$ for any $t\in [d+1]$. Hence, $\calF$ is shattered by tree $\bx$ of depth-$(d+1)$, leading to contradiction. Therefore, we have  
	\begin{equation}\label{eq: root-close}
		h((2a-1)\beta, (2b-1)\beta) \le \alpha,\qquad \forall a, b\in \calK
	\end{equation}
	
	Next, for any $k\in [M]$ with $\frakd(\calF_k, \alpha)\le d-1$, according to induction hypothesis $\mathfrak{G}(n-1, d-1)$, there exists a sequential cover $\calV_k^l$ of size $g_\beta(n-1, d-1)$ for the depth-$(n-1)$ $\calX$-valued tree $\bx^l$, and also a sequential cover $\calV_k^r$ of size $g_\beta(n-1, d-1)$ for the depth-$(n-1)$ $\calX$-valued tree $\bx^r$. We then combine the elements in $\calV_k^l$ and $\calV_k^r$ into a set $\calV_k$ of depth-$n$ $U_\beta$-valued trees. We let $v_1 = (2k-1)\beta\in U_\beta$. Then according to the construction of $\calF_k$ we have for any $f\in \calF$, $f(x_1) = v_1$ hence $h(f(x_1), v_1)\le \alpha$. For $\bv^l\in \calV_k^l$ and $\bv^r\in \calV_k^r$, we define depth-$n$ $U_\beta$-valued tree $\bv[\bv^l, \bv^r]$ as: for any path $\by\in \{0, 1\}^n$, we let $\by' = (y_{2:n})\in \{0, 1\}^{n-1}$, and let $v_1[\bv^l, \bv^r](\by) = v_1$. If $y_1 = 0$, then let $v_t[\bv^l, \bv^r](\by) = v_{t-1}^l(\by')$, and if $y_1 = 1$, then let $v_t[\bv^l, \bv^r](\by) = v_{t-1}^r(\by')$. We construct $\calV_k = \{\bv[\bv^l, \bv^r]\}$ with $|\calV_k|\le \max\{|\calV_k^l|, |\calV_k^r|\}$ to make sure that every element in $\calV_k^l$ and $\calV_k^r$ at least appear once in the construction of $\calV_k$. Next, we will argue that $\calV_k$ is a $\alpha$-sequential cover of $\calF_k\circ\bx$. For any $f\in \calF_k$ and $\by\in \{0, 1\}^n$, if $y_1 = 0$, then since $\calV_k^l$ is a $\alpha$-sequential cover of $\calF_k$, there exists $\bv^l\in \calV_k^l$ such that for any $2\le t\le n$, $h(f(x_t(\by)), v_t^l(\by))\le \alpha$. Suppose $\bv = \bv[\bv^l, \bv^r]\in \calV_k$ for some $\bv^r\in \calV_k^r$, and we also have $h(f(x_1(\by)), v_1(\by))\le \alpha$ according to the construction of $\calF_k$. Hence for any $t\in [n]$, we always $h(f(x_t(\by)), v_t(\by))\le \alpha$. Therefore, $\calV'$ is a cover of $\calF_k$. Further by induction hypothesis we have $\max\{|\calV_k^l|, |\calV_k^r|\}\le g_\beta(n-1, d-1)$. Hence $|\calV_k|\le g_\beta(n-1, d-1)$.

	If $\calK = \emptyset$, then by letting $\calV = \cup_{k\in [M]} \calV_k$, $\calV$ will be a $\alpha$-sequential cover of $\calF\circ \bx$, and also 
	$$|\calV| \le M\cdot g_\beta(n-1, d-1)\le g_\beta(n-1, d) + (M-1)g_\beta(n-1, d-1) = g_\beta(n, d),$$
	where the inequality follows from the fact that $g_\beta(n-1, d-1)\le g_\beta(n-1, d)$ for any $n, d$, and the last equation follows from \pref{eq: equation-g-beta}.
	
	Next, we consider cases where $|\calK|\ge 1$ We construct $\calF' = \cup_{k\in \calK} \calF_k$, then we have $\frakd(\calF', \alpha)\le \frakd(\calF, \alpha) = d$. According to the induction hypothesis $\mathfrak{G}(n-1, d)$, there exists a sequential cover $\calV^l$ of size $g_\beta(n-1, d)$ for the depth-$(n-1)$ $\calX$-valued tree $\bx^l$, and also a sequential cover $\calV^l$ of size $g_\beta(n-1, d)$ for the depth-$(n-1)$ $\calX$-valued tree $\bx^l$. We then combine the elements in $\calV^l$ and $\calV^r$ into a $\calV'$ of depth-$n$ $U_\beta$-valued trees. We let $v_1 = f(x_1)\in U_\beta$ for some $f\in \calF'$. Then according to the construction of $\calF'$ we have for any $f\in \calF'$, $h(f(x_1), v_1)\le \alpha$. For $\bv^l\in \calV^l$ and $\bv^r\in \calV^r$, we define depth-$n$ $U_\beta$-valued tree $\bv[\bv^l, \bv^r]$ as: for any path $\by\in \{0, 1\}^n$, we let $\by' = (y_{2:n})\in \{0, 1\}^{n-1}$, and let $v_1[\bv^l, \bv^r](\by) = v_1$. If $y_1 = 0$, then let $v_t[\bv^l, \bv^r](\by) = v_{t-1}^l(\by')$, and if $y_1 = 1$, then let $v_t[\bv^l, \bv^r](\by) = v_{t-1}^r(\by')$. We construct $\calV' = \{\bv[\bv^l, \bv^r]\}$ with $|\calV'|\le \max\{|\calV^l|, |\calV^r|\}$ to make sure that every element in $\calV^l$ and $\calV^r$ at least appear once in the construction of $\calV'$. Next, we will argue that $\calV'$ is a $\alpha$-sequential cover of $\calF'\circ\bx$. For any $f\in \calF'$ and $\by\in \{0, 1\}^n$, if $y_1 = 0$, then since $\calV^l$ is a $\alpha$-sequential cover of $\calF'$, there exists $\bv^l\in \calV^l$ such that for any $2\le t\le n$, $h(f(x_t(\by)), v_t^l(\by))\le \alpha$. Suppose $\bv = \bv[\bv^l, \bv^r]\in \calV'$ for some $\bv^r\in \calV^r$, and we also have $h(f(x_1(\by)), v_1(\by))\le \alpha$ according to \pref{eq: root-close} and the construction of $\calF'$. Hence for any $t\in [n]$, we always $h(f(x_t(\by)), v_t(\by))\le \alpha$. Therefore, $\calV'$ is a cover of $\calF'$. Further by induction hypothesis we have $\max\{|\calV^l|, |\calV^r|\}\le g_\beta(n-1, d)$. Hence $|\calV'|\le g_\beta(n-1, d)$.

	We further let $\calV = \calV'\cup\left(\cup_{k\not\in \calK} \calV_k\right)$, and we have 
	$$|\calV| \le (M-1)\cdot g_\beta(n-1, d-1) + g_\beta(n-1, d) = g_\beta(n, d),$$
	where the last equation follows from \pref{eq: equation-g-beta}. Above all, we finish the proof of induction statement $\mathfrak{G}(n, d)$.

	Hence by induction, we have 
	$$\calNsq(\calF\circ \bx, \alpha, n)\le g_\beta(n, \frakd(\calF, \alpha)) = \sum_{i=0}^{\frakd(\calF, \alpha)} \binom{n}{i}\cdot \left(M - 1\right)^i\le \left(\frac{en}{\beta \frakd(\calF, \alpha)}\right)^{\frakd(\calF, \alpha)}\le \left(\frac{en}{\beta}\right)^{\frakd(\calF, \alpha)}$$
\end{proof}

Equipped with \pref{lem: shattering-finite}, we are ready to prove \pref{prop: covering-fat} that works for real-valued function classes.
\begin{proof}[Proof of \pref{prop: covering-fat}]
	For $\beta > 0$ where $1/(2\beta)$ is an integer, we let $M = 1/(2\beta)$, and we define 
	$$U_\beta = \left\{\beta, 3\beta, 5\beta, \cdots, (2M-1)\beta\right\}.$$
	And for every $u\in [0, 1]$, we define $\lfloor u\rfloor_\beta = \argmin_{r\in U_\beta}|u - r|$. For any function $f\in \calF$, we define $\lfloor f\rfloor_\beta: \calX\to U_\beta$ as
	$$\lfloor f\rfloor_\beta(x)\coloneqq \lfloor f(x)\rfloor_\beta.$$ 
	We further let $\lfloor \calF\rfloor_\beta = \{\lfloor f\rfloor_\beta: f\in \calF\}$. According to \pref{def: def-fat-shattering} and \pref{def: shattering-finite} we know that if $\lfloor\calF\rfloor$ is shattered by some $\calX$-valued tree $\bx$ at scale $\alpha$, then $\bx$ also also $(\alpha, \beta)$-shattered by $\calF$. Hence we have 
	$$\frakd(\calF, \alpha, \beta)\ge \frakd(\lfloor \calF\rfloor_\beta, \alpha).$$
	Hence \pref{lem: shattering-finite} gives that for any depth-$n$ $\calX$-valued tree $\bx$, 
	$$\calN(\lfloor\calF\rfloor_\beta\circ \bx, \alpha, n)\le \left(\frac{en}{\beta}\right)^{\frakd(\lfloor \calF\rfloor_\beta, \alpha)}\le \left(\frac{en}{\beta}\right)^{\frakd(\calF, \alpha, \beta)}.$$
	We let $\calV$ to be the covering of $\lfloor \calF\rfloor_\beta$ at scale $\alpha$ with size no more than $(en/\beta)^{\frakd(\calF, \alpha, \beta)}$. Hence for any $f\in \calF$, $\bx\in \calX$ and $\by\in \{0, 1\}^n$, there exists $\bv\in \calV$ such that for any $t\in [n]$, 
	$$h\left(\lfloor f\rfloor_\beta(x_t(\by)), v_t(\by)\right)\le \alpha.$$
	According to the construction of $\lfloor f\rfloor_\beta$, we have $|\lfloor f\rfloor_\beta(x_t(\by)) - f(x_t(\by))|\le \beta$, which implies that
	\begin{align*}
		&\hspace{-0.5cm} h(\lfloor f\rfloor_\beta(x_t(\by)), f(x_t(\by)))^2\\
		& \le \left(\sqrt{\lfloor f\rfloor_\beta(x_t(\by))} - \sqrt{f(x_t(\by))}\right)^2 + \left(\sqrt{1 - \lfloor f\rfloor_\beta(x_t(\by))} - \sqrt{1 - f(x_t(\by))}\right)^2\\
		& \le \left|\lfloor f\rfloor_\beta(x_t(\by)) - f(x_t(\by))\right| + \left|1 - \lfloor f\rfloor_\beta(x_t(\by)) - (1 - f(x_t(\by)))\right|\le 2\beta,
	\end{align*}
	where the first inequality uses the fact that $(a - b)^2\le (a - b)(a + b)$ for $a, b\ge 0$. Therefore, for any $t\in [t]$, we have
	$$h(f(x_t(\by)), v_t(\by))\le h(\lfloor f\rfloor_\beta(x_t(\by)), f(x_t(\by))) + h\left(\lfloor f\rfloor_\beta(x_t(\by)), v_t(\by)\right)\le \alpha + \sqrt{2\beta},$$
	which implies that $\calV$ is an $(\alpha + \sqrt{2\beta})$ sequential covering of $\calF$, i.e. 
	$$\calN(\calF\circ\bx, \alpha + \sqrt{2\beta}, n)\le |\calV|\le \left(\frac{en}{\beta}\right)^{\frakd(\calF, \alpha, \beta)}.$$
	Taking supremum over $\bx$, we obtain that 
	$$\sup_\bx \calHsq(\calF, \alpha + \sqrt{2\beta}, n, \bx) = \sup_{\bx}\log \calN(\calF\circ \bx, \alpha + \sqrt{2\beta}, n)\le \frakd(\calF, \alpha, \beta)\log\left(\frac{en}{\beta}\right).$$
\end{proof}

\subsubsection{Proof of \pref{prop: lower-bound}}
Before proving \pref{prop: lower-bound}, we  present the following helper lemma.
\begin{lemma}\label{lem: empirical-log}
	Suppose real number $p, \alpha, \beta\in [0, 1]$ and integer $n$ satisfies $\beta < \alpha^2$ and 
	\begin{equation}\label{eq: condition-alpha}
		\alpha + \beta < p < 1 - \alpha - \beta,\quad \text{and}\quad n\le \frac{p(1-p)}{324\alpha^2}\vee 1.
	\end{equation}
	We collect $n$ samples from $\ber(p)$ to form an empirical estimation $\hp\in [0, 1]$, and we define 
	\begin{equation}\label{eq: def-epsilon}
		\epsilon = \begin{cases} 1 &\quad \text{if }\hp\ge p\\-1 &\quad \text{if }\hp < p.\end{cases}
	\end{equation}
	Then we have 
	\begin{equation} \label{eq: empirical-alpha}
		\EE\left[\hp\log\frac{p + \epsilon \alpha - \beta}{p} + (1-\hp)\log\frac{1-p - \epsilon \alpha - \beta}{1-p}\right]\ge \frac{\alpha^2}{8p(1-p)},
	\end{equation}
	where the expectation is over $\hp$ and $\epsilon$. Additionally, when choosing $n = \lfloor \frac{p(1-p)}{324\alpha^2}\rfloor\vee 1$, we have
	\begin{equation}\label{eq: empirical-n}
		n\cdot \EE\left[\hp\log\frac{p + \epsilon \alpha - \beta}{p} + (1-\hp)\log\frac{1-p - \epsilon \alpha - \beta}{1-p}\right]\ge \frac{1}{5184}.
	\end{equation}
\end{lemma}
\begin{proof}[Proof of \pref{lem: empirical-log}]
	Without loss of generality, we assume $p \le 1/2$ (otherwise we replace $p$ with $1 - p$ and the argument follows similarly). Then we have $\alpha < 1/2$ and $\beta < 1/4$. Consider the following three cases:
	\begin{enumerate}[label=$(\roman*)$]
		\item $1/36 < \alpha \le 1/2$,
		\item $\alpha + \beta \ge p/2$ and $\alpha < 1/36$,
		\item $\alpha + \beta < p/2$ and $\alpha < 1/36$.
	\end{enumerate}

	When $1/36 < \alpha < 1/2$, since $p(1-p)\le 1/4$ and $\alpha^2\ge 1/1296$, we always have $n = 1$, which implies 
	\begin{align*}
		&\hspace{-0.5cm} \EE\left[\hp\log\frac{p + \epsilon \alpha - \beta}{p} + (1-\hp)\log\frac{1-p - \epsilon \alpha - \beta}{1-p}\right]\\
		& = p\cdot \log\left(1 + \frac{\alpha - \beta}{p}\right) + (1-p)\cdot \log\left(1 + \frac{\alpha - \beta}{1 - p}\right)\stackrel{(i)}{\ge} \frac{\alpha - \beta}{2}\stackrel{(ii)}{\ge} \frac{\alpha}{4}\stackrel{(iii)}{\ge} \frac{\alpha^2}{8p(1-p)},
	\end{align*}
	where $(i)$ uses $\alpha - \beta > 0$ and either $(\alpha - \beta) / p > 1/2$ or $(\alpha - \beta) / (1 - p) > 1/2$ and $\log(1 + x)\ge x/2$ for $0\le x\le 2$, $(ii)$ uses the fact that $\alpha\le 1/2$ and $\beta\le \alpha^2$, and $(iii)$ uses the fact $\alpha\le p$ and $1 - p\ge 1/2$.

	When $\alpha + \beta > p/2$ and $\alpha < 1/36$, since $\beta < \alpha^2$ we have $\alpha > p/3$, which implies that 
	$$n\le \frac{p(1-p)}{324\alpha^2}\vee 1 < 1/p.$$
	Hence $\hp < p$, i.e. $\epsilon = -1$ if and only if $\hp = 0$. Therefore,
	\begin{align*}
		&\hspace{-0.5cm}\EE\left[\hp\log\frac{p + \epsilon \alpha - \beta}{p} + (1-\hp)\log\frac{1-p - \epsilon \alpha - \beta}{1-p}\right]\\
		& = \Pr(\hp = 0)\cdot \log\left(1 + \frac{\alpha - \beta}{1-p}\right) + \EE\left[\left(\hp\log\frac{p + \alpha - \beta}{p} + (1-\hp)\log\frac{1-p - \alpha - \beta}{1-p}\right)\cdot \mathbb{I}[\hp > 0]\right]. \numberthis \label{eq: hp-second-case}
	\end{align*}
	Since $p < 1/2$, we have $\alpha + \beta < 1/36 + 1/36^2 < (1-p)/2$, which implies $(\alpha + \beta)/(1-p)\le 1/2$. After noticing that $\log(1 + x)\ge x - x^2$ holds for all $x\ge -1/2$ and also $\alpha - \beta > 0$, we can further upper bound \pref{eq: hp-second-case} by
	\begin{align*}
		&\hspace{-0.5cm} \Pr(\hp = 0)\cdot \left(\left(\frac{\alpha - \beta}{1 - p}\right) - \left(\frac{\alpha - \beta}{1 - p}\right)^2\right)\\
		&\qquad + \EE\left[\left(\hp\cdot \left(\left(\frac{\alpha - \beta}{p}\right) - \left(\frac{\alpha - \beta}{p}\right)^2\right) + (1-\hp)\cdot \left(\left(\frac{ - \alpha - \beta}{1 - p}\right) - \left(\frac{ - \alpha - \beta}{1 - p}\right)^2\right)\right)\cdot \mathbb{I}[\hp > 0]\right]\\
		& = \EE\left[\hp\cdot \left(\left(\frac{\epsilon\alpha - \beta}{p}\right) - \left(\frac{\epsilon\alpha - \beta}{p}\right)^2\right) + (1 - \hp)\cdot \left(\left(\frac{ - \epsilon\alpha - \beta}{1- p}\right) - \left(\frac{ - \epsilon\alpha - \beta}{1 - p}\right)^2\right)\right]\\
		& \stackrel{(i)}{\ge} \EE\left[\frac{\epsilon \hp\alpha}{p} + \frac{-\epsilon(1-\hp)\alpha}{1-p}\right] - \beta\cdot \EE\left[\frac{\hp}{p} + \frac{1-\hp}{1-p}\right] - (\alpha + \beta)^2\cdot \EE\left[\frac{\hp}{p^2} + \frac{1-\hp}{(1-p)^2}\right]\\
		& \stackrel{(ii)}{=} \EE\left[\frac{\epsilon (\hp - p)\alpha}{p} + \frac{-\epsilon(p - \hp)\alpha}{1-p}\right] - 2\beta - (\alpha + \beta)^2\left(\frac{1}{p} + \frac{1}{1-p}\right)\\
		& = \alpha\cdot \EE\left[\frac{|\hp - p|}{p(1 - p)}\right] - 2\beta - \frac{(\alpha + \beta)^2}{p(1 - p)}\\
		& \stackrel{(iii)}{\le} \alpha\cdot \EE\left[\frac{|\hp - p|}{p(1 - p)}\right] - \frac{2\alpha^2}{p(1-p)}
	\end{align*}
	where $(i)$ uses $|\epsilon \alpha - \beta|\le \alpha + \beta$, $|-\epsilon \alpha - \beta|\le \alpha + \beta$ and $\EE[\hp] = p$, $(ii)$ uses the definition of $\epsilon$ in \pref{eq: def-epsilon} and $(iii)$ uses the fact that $0 < \beta\le \alpha^2\le 1/1296$. According to Khintchine inequality \cite{haagerup1981best}, we have 
	$$\EE[|\hp - p|] \ge \sqrt{\frac{p(1-p)}{2n}},$$
	which implies 
	\begin{align*}
		\text{LHS of }\pref{eq: empirical-alpha} & \ge \frac{\alpha}{\sqrt{2n p(1-p)}} - \frac{2\alpha^2}{p(1 - p)}\ge \frac{\alpha^2}{p(1-p)},
	\end{align*}
	where the last inequality uses the fact that $$n\le \frac{p(1-p)}{324\alpha^2}\vee 1\quad \text{and}\quad \frac{\alpha}{\sqrt{p(1-p)}}\le \sqrt{\alpha}\cdot \sqrt{\frac{p}{p(1-p)}}\le \frac{1}{3\sqrt{2}}.$$

	When $\alpha + \beta < p/2$ and $\alpha < 1/36$, we have $(p-\alpha -\beta)/p\ge 1/2$ and also $(1-p-\alpha - \beta)/(1-p)\ge 1/2$, then for any $\epsilon \in \{-1, 1\}$,
	$$\frac{p + \epsilon\alpha - \beta}{p}\ge \frac{1}{2},\quad \text{and}\quad \frac{1 - p - \epsilon\alpha - \beta}{1 - p}\ge \frac{1}{2}.$$
	Notice that $\log(1 + x)\ge x - x^2$ holds for all $x\ge -1/2$, which implies 
	\begin{align*}
		&\hspace{-0.5cm} \text{LHS of }\pref{eq: empirical-alpha}\\
		& \ge \EE\left[\hp\cdot \left(\left(\frac{\epsilon\alpha - \beta}{p}\right) - \left(\frac{\epsilon\alpha - \beta}{p}\right)^2\right) + (1 - \hp)\cdot \left(\left(\frac{ - \epsilon\alpha - \beta}{1- p}\right) - \left(\frac{ - \epsilon\alpha - \beta}{1 - p}\right)^2\right)\right]\\
		& = \EE\left[\hp\left(\frac{\epsilon\alpha - \beta}{p}\right) + (1 - \hp)\left(\frac{ - \epsilon\alpha - \beta}{1- p}\right) - \hp\left(\frac{\epsilon\alpha - \beta}{p}\right)^2 - (1-\hp)\left(\frac{ - \epsilon\alpha - \beta}{1 - p}\right)^2\right]\\
		& \stackrel{(i)}{\ge} \EE\left[\frac{\epsilon \hp\alpha}{p} + \frac{-\epsilon(1-\hp)\alpha}{1-p}\right] - \beta\cdot \EE\left[\frac{\hp}{p} + \frac{1-\hp}{1-p}\right] - (\alpha + \beta)^2\cdot \EE\left[\frac{\hp}{p^2} + \frac{1-\hp}{(1-p)^2}\right]\\
		& \stackrel{(ii)}{=} \EE\left[\frac{\epsilon (\hp - p)\alpha}{p} + \frac{-\epsilon(p - \hp)\alpha}{1-p}\right] - 2\beta - (\alpha + \beta)^2\left(\frac{1}{p} + \frac{1}{1-p}\right)\\
		& = \alpha\cdot \EE\left[\frac{|\hp - p|}{p(1 - p)}\right] - 2\beta - \frac{(\alpha + \beta)^2}{p(1 - p)}\\
		& \stackrel{(iii)}{\le} \alpha\cdot \EE\left[\frac{|\hp - p|}{p(1 - p)}\right] - \frac{2\alpha^2}{p(1-p)}
	\end{align*}
	where $(i)$ uses $|\epsilon \alpha - \beta|\le \alpha + \beta$, $|-\epsilon \alpha - \beta|\le \alpha + \beta$ and $\EE[\hp] = p$, $(ii)$ uses the definition of $\epsilon$ in \pref{eq: def-epsilon} and $(iii)$ uses the fact that $0 < \beta\le \alpha^2\le 1$. According to \cite[Theorem 1]{berend2013sharp}, we have 
	$$\EE[|\hp - p|] \ge \sqrt{\frac{p(1-p)}{2n}},$$
	which implies 
	\begin{align*}
		\text{LHS of }\pref{eq: empirical-alpha} & \ge \frac{\alpha}{\sqrt{2n p(1-p)}} - \frac{2\alpha^2}{p(1 - p)}\ge \frac{\alpha^2}{p(1-p)},
	\end{align*}
	where the last inequality uses the fact that $$n\le \frac{p(1-p)}{324\alpha^2}\vee 1\quad \text{and}\quad \frac{\alpha}{\sqrt{p(1-p)}}\le \sqrt{\alpha}\cdot \sqrt{\frac{p}{p(1-p)}}\le \frac{1}{3\sqrt{2}}.$$
    Above all, we have verified \pref{eq: empirical-alpha}, and \pref{eq: empirical-n} follows from \pref{eq: empirical-alpha} directly.
\end{proof}

Now we are ready to prove \pref{prop: lower-bound}.
\begin{proof}[Proof of \pref{prop: lower-bound}]
	Let $x_0\in \calX$ be an arbitrary context. For fixed positive integer $n$, we let 
	$$\alpha_n = \argmax_{\alpha > 0}\left\{\frakd(\calF, \alpha, \alpha^4/16)\cdot \left(\left\lceil\frac{1}{162\alpha^2}\right\rceil\vee 1\right)\le n\right\}.$$
	Since $\frakd(\calF, \alpha, \alpha^4/16) = \tilde{\Omega}\left(\alpha^{-p}\right)$ for every $\alpha\ge 0$, we have 
	$$\frakd(\calF, \alpha_n, \alpha_n^4/16) = \tilde{\Omega}\left(n^{\frac{p}{p+2}}\right).$$
	In the following, we will prove that for any positive integer $n$, we have 
	$$\calR_n(\calF)\ge \frac{\frakd(\calF, \alpha_n, \alpha_n^4/16)}{5184}.$$

	We fix $n$, and let $\alpha = \alpha_n$, $d = \frakd(\calF, \alpha_n, \alpha_n^4/16)$. We let $\btx$ to be the depth-$d$ $\calX$-valued tree shattered by $\calF$ at scale $(\alpha_n, \alpha_n^4/16)$. Then according to \pref{def: def-fat-shattering}, there exists a depth-$d$ $[0, 1]\times [0, 1]$-valued tree $\bs$ such that for any path $\bty = (\ty_{1:d})\in \{0, 1\}^d$, there exists $f^\by\in \calF$ such that
	\begin{equation}\label{eq: eq-natarajan}
		\left|f^{\bty}(\tx_t(\bty)) - s_t(\bty)[\ty_t]\right| < \frac{\alpha^4}{16}\quad \text{and}\quad h(s_t(\bty)[0], s_t(\bty)[1])\ge \alpha\qquad \forall t\in [d],
	\end{equation}
	After noticing that $h(u, v)^2/2\le |u-v|$ for any $u, v\in [0, 1]$, we have
	$$\left|f^{\bty}(\tx_t(\bty)) - s_t(\bty)[\ty_t]\right|\le \frac{\left(s_t(\bty)[0] - s_t(\bty)[1]\right)^2}{4}$$
	We define depth-$d$ $[0, 1]$-valued $\bv$ as 
	\begin{equation}\label{eq: def-v-t}
		v_t(\bty) = \frac{s_t(\bty)[0] + s_t(\bty)[1]}{2},\quad \forall \by\in \{0, 1\}^n.
	\end{equation}
	We can further verify that
	\begin{align*} 
		&\hspace{-0.5cm} h(s_t(\bty)[0], s_t(\bty)[1])^2\\
		& = \left(s_t(\bty)[0] - s_t(\bty)[1])\right)^2 \\
		&\quad \cdot\max\left\{\frac{1}{(\sqrt{s_t(\bty)[0]} + \sqrt{s_t(\bty)[1]})^2}, \frac{1}{(\sqrt{1 - s_t(\bty)[0]} + \sqrt{1 - s_t(\bty)[1]})^2}\right\}\\
		& \le \left(s_t(\bty)[0] - s_t(\bty)[1])\right)^2\cdot \left(\frac{1}{s_t(\bty)[0] + s_t(\bty)[1]} + \frac{1}{1 - s_t(\bty)[0] + 1 - s_t(\bty)[1]}\right)\\
		& = \frac{2\left(s_t(\bty)[0] - s_t(\bty)[1])\right)^2}{v_t(\ty)(1-v_t(\ty))} \numberthis\label{eq: h-k-connection}
	\end{align*}

	Next, we will show $\calR_n(\calF)\ge d$. According to \pref{lem: dual-form}, for any $\bp\in \Delta(\{0, 1\}^n)$ and depth-$n$ $\calX$-valued tree $\bx$, we have
	\begin{equation}\label{eq: r-n-r-n-p}
		\calR_n(\calF)\ge \EE_{\by\sim \bp} \calR_n(\calF(\bx), \bp, \by).
	\end{equation}
	In the following, we will construct specific $\bp$ and $\bx$ so that the right-hand side in the above inequality is lower bounded by $d$. For a fixed a path $\by\in \{0, 1\}^n$, we first define an auxillary $\{0, 1\}$-path $\bty = (\ty_{1:d})\in \{0, 1\}^d$ of length-$d$ and also $d$ integers: $k_1, k_2, \cdots, k_d$ in the following way: calculate $\ty_{1:d}$ and $k_{1:d}$ by turn:
	\begin{equation}\label{eq: def-k-t}
		k_t = \left\lfloor\frac{v_t(\ty)(1-v_t(\ty))}{324\left(s_t(\bty)[1] - s_t(\bty)[0]\right)^2}\right\rfloor\vee 1,\qquad \forall t\in [d].
	\end{equation}
	and 
	\begin{equation}\label{eq: def-ty-t}
		\ty_t = \mathbb{I}\bigg\{\sum_{j=1}^{k_t} y_{k_1 + \cdots + k_{t-1} + j} \ge k_t\cdot v_t(\bty)\bigg\},\qquad \forall t\in [d],
	\end{equation}
	where $v_t(\bty)$ is defined in \pref{eq: def-v-t}.
	Notice that according to the above definition, $k_t$ only depends on $y_{1}, \cdots, y_{k_1 + \cdots + k_{t-1}}$, and $\ty_t$ depends on $y_1, \cdots, y_{k_1 + \cdots + k_t}$. Additionally, according to \pref{eq: h-k-connection} and \pref{eq: eq-natarajan}, we have 
	$$k_t\le \frac{1}{162\alpha^2}\vee 1,\qquad \forall t\in [d]$$
	which implies $k_1 + \cdots + k_d\le n$ always holds according to our choice of $\alpha = \alpha_n$. Hence $k_{1:d}$ and $\ty_{1:d}$ are all well-defined.

	The value of $(x_1(\by), x_t(\by), \cdots, x_n(\by))$ 
	are in the following form:
	$$(\underbrace{\tx_1(\bty), \tx_1(\bty), \cdots, \tx_1(\bty)}_{k_1 \text{ pieces}}, \underbrace{\tx_2(\bty), \tx_2(\bty), \cdots, \tx_2(\bty)}_{k_2\text{ pieces}}, \cdots, \underbrace{\tx_d(\bty), \tx_d(\bty), \cdots, \tx_d(\bty)}_{k_d\text{ pieces}}, \underbrace{x_0, x_0, \cdots, x_0}_{(n - k_1 - k_2 - \cdots - k_d)\text{ pieces}}),$$
	Similarly, the value of $(p_1(\by), p_t(\by), \cdots, p_n(\by))$ 
	are in the following form:
	$$(\underbrace{v_1(\bty), v_1(\bty), \cdots, v_1(\bty)}_{k_1 \text{ pieces}}, \underbrace{v_2(\bty), v_2(\bty), \cdots, v_2(\bty)}_{k_2 \text{ pieces}}, \cdots, \underbrace{v_d(\bty), v_d(\bty), \cdots, v_d(\bty)}_{k_d \text{ pieces}}, \underbrace{f^{\bty}(x_0), f^{\bty}(x_0)\cdots f^{\bty}(x_0)}_{(n - k_1 - k_2 - \cdots - k_d)\text{ pieces}}),$$
    where $x_0$ is the state we fixed at the beginning of the proof.

	Then we write $\calR_n(\calF(\bx), \bp, \by)$ in terms of segments $1$ to $d$, and after noticing that $f^\by\in \calF$ for any depth-$d$ path $\by\in \{-1, 1\}^d$, we obtain
	\begin{align*} 
		&\hspace{-0.5cm} \calR_n(\calF(\bx), \bp, \by)\\
		& = \sup_{f\in \calF} \Bigg\{\sum_{t=1}^d\sum_{j=1}^{k_d} \log\frac{f(y_{k_1 + \dots + k_{t-1}+j}\mid x_{k_1 + \dots + k_{t-1}+j}(\by))}{p_{k_1 + \dots + k_{t-1}+j}(y_{k_1 + \dots + k_{t-1}+j}\mid \by)}\\
		&\qquad + \sum_{j=1}^{n - k_1 - \cdots - k_d}\log\frac{f(y_{k_1 + \dots + k_d+j}\mid x_{k_1 + \dots + k_d+j}(\by))}{p_{k_1 + \dots + k_d+j}(y_{k_1 + \dots + k_d+j}\mid \by)}\Bigg\}\\
		& = \sup_{f\in \calF}\left\{\sum_{t=1}^d\sum_{j=1}^{k_d} \log\frac{f(y_{k_1 + \dots + k_{t-1}+j}\mid \tx_{t}(\bty))}{v_{t}(y_{k_1 + \dots + k_{t-1}+j}\mid \bty)} + \sum_{j=1}^{n - k_1 - \cdots - k_d}\log\frac{f(y_{k_1 + \dots + k_d+j}\mid x)}{f^{\bty}(y_{k_1 + \dots + k_d+j}\mid x_0)}\right\}\\
		& \ge \sum_{t=1}^d\sum_{j=1}^{k_d} \log\frac{f^{\bty}(y_{k_1 + \dots + k_{t-1}+j}\mid \tx_{t}(\bty))}{v_{t}(y_{k_1 + \dots + k_{t-1}+j}\mid \bty)},\numberthis\label{eq: R-n-lower-bound}
	\end{align*}
 	where the last step takes $f = f^\by\in \calF$. Next, for fixed $\by$, we define
	$$\hat{v}_t = \frac{1}{k_t} \sum_{j=1}^{k_t} y_{k_1 + \cdots + k_{t-1} + j},$$
	and let
	\begin{equation}\label{eq: def-gamma-t-bty}
            \gamma_t(\bty) = \frac{s_t(\bty)[1] - s_t(\bty)[0]}{2}.
        \end{equation}
	Then using inequality $h(u, v)^2/2\le |u - v|$ for any $u, v\in [0, 1]$, we have 
	\begin{equation}\label{eq: gamma-alpha}
		\gamma_t(\bty) \ge \frac{h(s_t(\bty)[1], s_t(\bty)[0])^2}{4}\ge \frac{\alpha^2}{4}.
	\end{equation}
	Notice that $\tx_t(\bty)$ and $v_t(\bty)$ is independent to $y_{k_1 + \cdots + k_{t-1} + 1}, \cdots, y_{k_1 + \cdots + k_{t-1} + k_t}$, we can rewrite \pref{eq: R-n-lower-bound} as
	\begin{align*}
		\calR_n(\calF(\bx), \bp, \by) & \ge \sum_{t=1}^d k_t\cdot \left(\hat{v}_t\log\frac{f^{\bty}(\tx_{t}(\bty))}{v_{t}(\bty)} + \left(1-\hat{v}_t\right)\log\frac{1 - f^{\bty}(\tx_{t}(\bty))}{1 - v_{t}(\bty)}\right).
	\end{align*}
	Next, we will calculate $\EE_{\by\sim \bp}\left[\calR_n(\calF(\bx), \bp, \by)\right]$. According to the above equation, we can separate the expectation into the sum of $d$ expectations as follows:
	\begin{align*} 
		&\hspace{-0.5cm} \EE_{\by\sim \bp}\left[\calR_n(\calF(\bx), \bp, \by)\right]\\
		& = \sum_{t=1}^d \EE\left[k_t\cdot \left(\hat{v}_t\log\frac{f^{\bty}(\tx_{t}(\bty))}{v_{t}(\bty)} + \left(1-\hat{v}_t\right)\log\frac{1 - f^{\bty}(\tx_{t}(\bty))}{1 - v_{t}(\bty)}\right)\right]\\
		& = \sum_{t=1}^d \EE\left[\EE\left[k_t\cdot \left(\hat{v}_t\log\frac{f^{\bty}(\tx_{t}(\bty))}{v_{t}(\bty)} + \left(1-\hat{v}_t\right)\log\frac{1 - f^{\bty}(\tx_{t}(\bty))}{1 - v_{t}(\bty)}\right)\ \bigg{|}\ y_{1:(k_1 + \cdots k_{t-1})}\right]\right]\\
		& \stackrel{(i)}{\ge} \sum_{t=1}^d \EE\left[\EE\left[k_t\cdot \left(\hat{v}_t\log\frac{s_{t}(\bty)[\ty_t] - \alpha^4/16}{v_{t}(\bty)} + \left(1-\hat{v}_t\right)\log\frac{1 - s_{t}(\bty)[\ty_t] - \alpha^4/16}{1 - v_{t}(\bty)}\right)\ \bigg{|}\ y_{1:(k_1 + \cdots k_{t-1})}\right]\right]\\
		& \stackrel{(ii)}{\ge} \sum_{t=1}^d \EE\bigg[\EE\bigg[k_t\cdot \bigg(\hat{v}_t\log\frac{v_{t}(\bty) + \epsilon(\hat{v}_t)\gamma_t(\bty) - \gamma_t(\bty)^2}{v_{t}(\bty)}\\
        &\qquad+ \left(1-\hat{v}_t\right)\log\frac{1 - v_t(\bty) - \epsilon(\hat{v}_t) v_{t}(\bty) - \gamma_t(\bty)^2}{1 - v_{t}(\bty)}\bigg)\ \bigg{|}\ y_{1:(k_1 + \cdots k_{t-1})}\bigg]\bigg]  \numberthis \label{eq: equation-lower-decompose}
	\end{align*}
	where $(i)$ uses the choice of $f^{\bty}$ in \pref{eq: eq-natarajan}, and in $(ii)$ we define 
	$$\epsilon(\hat{v}_t) = \begin{cases}
		1 &\quad \text{if } \hat{v}_t\ge v_t(\bty),\\
		-1 &\quad \text{if } \hat{v}_t < v_t(\bty).
	\end{cases}$$
	and it follows from our construciton of $\ty_t$ in \pref{eq: def-ty-t} and also $\alpha^2/4\le \gamma_t(\bty)$ from \pref{eq: gamma-alpha}. In \pref{eq: equation-lower-decompose}, conditioned on $y_{1:(k_1 + \cdots + k_{t-1})}$, there is no randomness on $k_t$, $v_t(\bty)$ and also $\tx_t(\bty)$, hence all the randomness of the inner expectation comes from $\hat{v}_t$ and also $s_t(\bty)[\ty_t]$. With our choice of $\gamma_t(\bty)$ in \pref{eq: def-gamma-t-bty}, we can further verify that $\gamma_t(\bty) + \gamma_t(\bty)^2\le s_t(\bty)[0] + \gamma_t(\bty)\le v_t(\bty)$. Hence, after noticing \pref{eq: def-k-t}, we can verify that the conditions in \pref{lem: empirical-log} hold with $\alpha = \gamma_t(\bty), \beta = \gamma_t(\bty)^2, p = v_t(\bty)$ and $n = k_t$. Additionally, according to our construction of $p_t(\by)$, when conditioned on $y_{1:(k_1 + \cdots + k_t)}$, we have 
	$$y_{k_1 + \cdots k_{t-1} + 1}, \cdots, y_{k_1 + \cdots k_{t-1} + k_t}\simiid v_t(\cdot\mid  \bty).$$
	Hence \pref{eq: empirical-n} in \pref{lem: empirical-log} implies that 
	$$\EE\left[k_t\cdot \left(\hat{v}_t\log\frac{v_{t}(\bty) + \epsilon(\hat{v}_t)\gamma_t(\bty) - \gamma_t(\bty)^2}{v_{t}(\bty)} + \left(1-\hat{v}_t\right)\log\frac{1 - v_t(\bty) - \epsilon(\hat{v}_t) v_{t}(\bty) - \gamma_t(\bty)^2}{1 - v_{t}(\bty)}\right)\ \bigg{|}\ y_{1:(k_1 + \cdots k_{t-1})}\right]\ge \frac{1}{5184}.$$
	Bringing this back to \pref{eq: equation-lower-decompose} and then further back to \pref{eq: r-n-r-n-p}, we obtain that 
	$$\calR_n(\calF)\ge \calR_n(\calF, \alpha, \alpha^4)/16\ge d\cdot \frac{1}{5184} = \frac{\frakd(\calF, \alpha, \alpha^4/16)}{5184} = \Omega(n^{\frac{p}{p+2}}).$$
\end{proof}

\subsection{Proof of \pref{thm: lower-bound-large-p}}
We present the proof of \pref{thm: lower-bound-large-p} in this section. We first present a lemma showing that when $f, p\in [c, 1-c]$, we have $h(f, p)\asymp |f - p|$.
\begin{lemma}\label{lem: tv-hellinger-c}
	Suppose $c$ to be a positive constant in $(0, 1/2)$, then for any $f, p\in [c, 1-c]$, we have 
	$$\frac{|f - p|}{\sqrt{2}}\le h(f, p)\le \frac{|f - p|}{2\sqrt{c}}.$$
\end{lemma}
\begin{proof}
	As for the lower bound part, we notice that 
	\begin{align*} 
		h(f, p)^2 & \ge \frac{1}{2}\cdot \left(\left(\sqrt{f} - \sqrt{p}\right)^2 + \left(\sqrt{1 - f} - \sqrt{1 - p}\right)^2\right)\\
		& = \frac{1}{2}\cdot (f - p)^2\cdot \left(\frac{1}{(\sqrt{f} + \sqrt{p})^2} + \frac{1}{(\sqrt{1 - f} + \sqrt{1 - p})^2}\right)\\
		& \ge \frac{1}{2}\cdot (f-p)^2\cdot \left(\frac{1}{2(f+p)} + \frac{1}{2(2-f-p)}\right)\\
		& \ge \frac{1}{2}(f - p)^2,
	\end{align*}
	where the second and third inequalities both use Cauchy-Schwarz inequality.

	Next we prove the upper bound part. Since $f, p\in [c, 1-c]$, we have 
	$$(\sqrt{f} + \sqrt{p})^2\ge 4c \quad \text{and}\quad (\sqrt{1-f} + \sqrt{1-p})^2\ge 4c,$$
	which implies that 
	$$h(f, p)\le (f-p)^2 \cdot \frac{1}{4c} = \frac{(f-p)^2}{4c}.$$
\end{proof}

We are now ready to prove \pref{thm: lower-bound-large-p}.
\begin{proof}[Proof of \pref{thm: lower-bound-large-p}]
	Since $\sup_{\bx} \calHsq (\calF, \alpha, n, \bx) = \tilde{\Omega}(\alpha^{-p})$, by choosing $\beta = \alpha^2/2$ in \pref{def: def-fat-shattering}, we obtain 
	$$\frakd(\calF, \alpha, \alpha^2/2) = \tilde{\Omega}\left(\alpha^{-p}\right),\quad \forall \alpha > 0.$$ 
	Hence if we choose $\beta$ such that
	$$\beta = \sup\left\{\beta\in (0, 1/16): \frakd(\calF, \beta, \beta^2/2)\ge n\right\},$$
	we have $\beta = \tilde{\Omega}\left(n^{-1/p}\right)$. And according to \pref{def: def-fat-shattering}, there exists a depth-$n$ $\calX$-valued tree $\bx$ shattered by $\calF$ at scale $(\beta, \beta^2/2)$, i.e. there exists a depth-$n$ $[0, 1]\times [0, 1]$-valued tree $\bs$ such that for any path $\by = (y_{1:d})\in \{0, 1\}^d$, $s_t(\by) = (s_t(\by)[0], s_t(\by)[1])$ with $s_t(\by)[0] < s_t(\by)[1]$, and also there exists $f^\by\in\calF$ such that 
	\begin{equation}\label{eq: condition-tv-hellinger}
		\left|f^\by(x_t(\by)) - s_t(\by)[y_t]\right| < \frac{\beta^2}{2} \quad \text{and}\quad h\left(s_t(\by)[0], s_t(\by)[1]\right) > \beta,\qquad \forall t\in [n].
	\end{equation}
	We further define $[0, 1]$-valued tree $\bu$ where for any $\by\in \{0, 1\}^d$ and $t\in [n]$, 
	$$u_t(\by) = \frac{s_t(\by)[0] + s_t(\by)[1]}{2}.$$
	Since $\calF\subseteq [7/16, 9/16]^\calX$, according to \pref{lem: tv-hellinger-c} we have for any $\by\in \{0, 1\}^n$, $3/8 \le s_t(\by)[0] < s_t(\by)[1] < 5/8$, which implies that 
	$$s_t(\by)[1] - s_t(\by)[0]\ge 2\sqrt{3/8}\cdot h\left(s_t(\by)[0], s_t(\by)[1]\right) > \beta.$$
	Hence we have $u_t(\by)\in [7/16, 9/16]$, and for any $\by\in \{0, 1\}^n$,
	\begin{align*} 
		f^\by(y_t\mid x_t(\by)) - u_t(y_t\mid \by) & \ge \beta - \frac{\beta^2}{2}\ge \frac{\beta}{2}.
	\end{align*}

	We next notice that for any $f\in \calF$, $\by\in \{0, 1\}^n$ and $t\in [n]$,
    $$\frac{f(y_t\mid x_t(\by))}{u_t(y_t\mid \by)}\ge f(y_t\mid x_t(\by))\ge 7/16,$$
    hence according to inequality $\log(1 + x)\ge x - 3x^2/2$ for any $x\ge -9/16$, we have
	\begin{equation}\label{eq: sum-log-lower-bound}
		\sum_{t=1}^n \log\frac{f(y_t\mid x_t(\by))}{u_t(y_t\mid \by)}\ge \sum_{t=1}^n\left\{\left(\frac{f(y_t\mid x_t(\by))}{u_t(y_t\mid \by)} - 1\right) - \frac{3}{2}\cdot \left(\frac{f(y_t\mid x_t(\by))}{u_t(y_t\mid \by)} - 1\right)^2\right\}.
	\end{equation}
	We choose $f = f^\by$ in the above inequality. After noticing that 
	$$f^\by(y_t\mid x_t(\by))\in \left[\frac{7}{16}, \frac{9}{16}\right],\qquad u_t(y_t\mid \by)\in \left[\frac{7}{16}, \frac{9}{16}\right]\quad \text{and}\quad f^\by(y_t\mid x_t(\by))\ge \frac{\beta}{2} + u_t(y_t\mid \by),$$
	we have 
	$$0 < \frac{f^\by(y_t\mid x_t(\by))}{u_t(y_t\mid \by)} - 1\le \frac{1-7/16}{7/16} - 1 = \frac{2}{7}.$$
	Therefore, we have
	\begin{align*}
		&\hspace{-0.5cm} \left(\frac{f^\by(y_t\mid x_t(\by))}{u_t(y_t\mid \by)} - 1\right) - \frac{3}{2}\cdot \left(\frac{f^\by(y_t\mid x_t(\by))}{u_t(y_t\mid \by)} - 1\right)^2\\
		& \ge \frac{4}{7}\left(\frac{f^\by(y_t\mid x_t(\by))}{u_t(y_t\mid \by)} - 1\right)\\
		& \stackrel{(i)}{\ge} f^\by(y_t\mid x_t(\by)) - u_t(y_t\mid \by)\ge \frac{\beta}{2},
	\end{align*}
	where inequality $(i)$ uses the fact that $u_t(y_t\mid \by)\le 9/16 < 4/7$. Bringing back to \pref{eq: sum-log-lower-bound}, we obtain that 
	$$\calR_n(\calF) = \EE_{\by\sim \bp}\left[\sup_{f\in \calF}\sum_{t=1}^n \log\frac{f^\by(y_t\mid x_t(\by))}{u_t(y_t\mid \by)}\right]\ge \frac{n\beta}{2} = \tilde{\Omega}\left(n^{\frac{p-1}{p}}\right).$$
\end{proof}

\section{Missing Proofs in \pref{sec: hilbert ball}}\label{sec: example-hilbert-app}
We first prove \pref{lem: truncation-hilbert}.
\begin{proof}[Proof of \pref{lem: truncation-hilbert}]
	Recall from \pref{lem: dual-form} we have 
	$$\calR_n(\calF) = \sup_{\bx}\sup_{\by\sim \bp}\left[\sum_{t=1}^n \log\frac{1}{p_t(y_t\mid \by)} - \inf_{f\in\calF} \sum_{t=1}^n\log\frac{1}{f(y_t\mid x_t(\by))}\right].$$
	Hence we only need to prove that for any path $\by = (y_{1:n}) \in \{0, 1\}^n$,
	\begin{equation}\label{eq: truncation-inequality-hilbert}
		\sup_{f\in\calF} \sum_{t=1}^n\log f(y_t\mid x_t(\by))\le \sup_{f\in\calF_{1/n}} \sum_{t=1}^n\log f(y_t\mid x_t(\by)) + 2.
	\end{equation}
	According to the definition of Hilbert ball class \pref{eq: def-f-hilbert}, for any $f\in\calF$, there exists $w\in B_2(1)$ such that
	$$f(y_t\mid x_t(\by)) = \frac{1 + (-1)^{y_t}\langle x_t(\by), w\rangle}{2}.$$
	Next, we notice that for any real number $a\in (-1, 1)$, we have
	\begin{align*} 
		\log \frac{a + 1}{2} & \le \log \frac{a + n/(n-1)}{2}\le \log \frac{(1 - 1/n)a + 1}{2} + \log\frac{n}{n-1}\\
		&\le \log \frac{(1 - 1/n)a + 1}{2} + \frac{1}{n-1}.
	\end{align*}
	Therefore, we obtain
	$$\frac{1 + (-1)^{y_t}\langle x_t(\by), w\rangle}{2}\le \frac{1 + (-1)^{y_t}\langle x_t(\by), (1-1/n)w\rangle}{2} + \frac{1}{n-1},$$
	which implies that 
	$$\sum_{t=1}^n\frac{1 + (-1)^{y_t}\langle x_t(\by), w\rangle}{2}\le \sum_{t=1}^n\frac{1 + (-1)^{y_t}\langle x_t(\by), (1-1/n)w\rangle}{2} + \frac{n}{n-1}.$$
	Since for any $w\in B_2(1)$, we always have $(1-1/n) w\in B_2(1-1/n)$, according to the definition of function class $\calF_{1/n}$ we have 
	\begin{align*} 
		&\hspace{-0.5cm} \sup_{f\in\calF} \sum_{t=1}^n\log f(y_t\mid x_t(\by))\\
		&\le \sup_{f\in\calF_{1/n}} \sum_{t=1}^n\log f(y_t\mid x_t(\by)) + \frac{n}{n-1}\\
		&\le \sup_{f\in\calF_{1/n}} \sum_{t=1}^n\log f(y_t\mid x_t(\by)) + 2,
	\end{align*}
	which proves \pref{eq: truncation-inequality-hilbert}.
\end{proof}

We next prove \pref{prop: tree-fat}. Our proof requires the following definition of skipping binary tree.
\begin{definition}
	For a given binary tree $\mathbf{x}$, we say a binary tree $\mathbf{y}$ is a skipping tree of $\mathbf{x}$ if
	\begin{enumerate}
		\item The set of vertices of $\mathbf{y}$ is a subset of the set of vertices of $\mathbf{x}$.
		\item For two vertices $a, b$ of $\mathbf{y}$, if $a$ is $b$'s left child in $\mathbf{y}$, then $a$ is a descendant of $b$'s left child in $\mathbf{x}$; and if $a$ is $b$'s right child in $\mathbf{y}$, then $a$ is a descendant of $b$'s right child in $\mathbf{x}$.
	\end{enumerate}
\end{definition}
We have the following properties of coloring over binary trees.
\begin{lemma}\label{lem: pigeonhole}
	We consider $k$-coloring over the vertices of a depth-$n$ binary tree $\mathbf{x}$, where each vertices has been colored in one of $k$ colors. Then when $n\ge k(d-1) + 1$, $\mathbf{x}$ has a skipping tree of depth $d$ whose nodes are of the same color.
\end{lemma}
\begin{proof}
\par We prove a stronger result: For integers $d_1, d_2, \ldots, d_k\ge 0$, if binary tree $\mathbf{x}$ has depth at least $d_1 + d_2 + \cdots + d_k + 1$, and is colored in $1, \ldots, k$. Then there exists $1\le i\le k$, such that $\mathbf{x}$ has a skipping tree of depth $d_i + 1$ whose vertices are all in color $i$. 
\par We will prove this result by induction on $d_1 + \cdots + d_k$. When $d_1 + \cdots + d_k = 0$, we have $d_1 = \cdots = d_k = 0$. In this case, assuming the root is colored in $i$, then the root itself is a skipping tree with depth $d_i + 1$. 
\par Next we assume that this result holds when $d_1 + \cdots + d_k = m$. When $d_1 + \cdots + d_k = m+1$, we assume the root $a$ is colored in $j$ ($1\le j\le k$). If $d_j = 0$, then we already have a skipping tree with only one vertex $a$ in color $j$ at depth $d_j + 1$. Next, we assume that $d_j\ge 1$. We consider the left binary tree $\mathbf{x}_1$ rooted at the left child of $a$, and also the right binary tree $\mathbf{x}_2$ rooted at the right child of $a$. Then both $\mathbf{x}_1$ and $\mathbf{x}_2$ are binary trees with depth
$$m = d_1 + \cdots + d_{j-1} + (d_i - 1) + d_{j+1} + \cdots + d_k = \hat{d}_1 + \cdots + \hat{d}_k.$$
where we let $\hat{d}_i = d_i$ if $i\neq j$ and $\hat{d}_i = d_i - 1$ if $i = j$. According to induction, $\exists 1\le i_1, i_2\le k$ such that there exists $\mathbf{x}_1$'s skipping tree $\bu_1$ of depth $\hat{d}_{i_1} + 1$ whose vertices are all in color $i_1$ and also $\mathbf{x}_2$'s skipping tree $\bu_2$ of depth $\hat{d}_{i_2} + 1$ whose vertices are all in color $i_2$. If $i_1\neq j$, then we have $\hat{d}_{i_1} = d_{i_1}$. Hence the skipping tree $\bu_1$ is also a skipping tree of $\mathbf{x}$ with depth $d_{i_1} + 1$ whose vertices are all in color $i_1$. This skipping tree is desirable. If $i_2\neq j$, similarly we can also find a desirable skipping tree of $\mathbf{x}$. Finally if $i_1 = i_2 = j$, we consider the tree $\mathbf{y}$ with root $j$, and two subtrees of $j$'s left child and right child to be $\bu_1$ and $\bu_2$. Then since the root of $\bu_1$ is a descendant of $j$'s left child the root of $\bu_2$ is a descendant of $j$'s right child, $\mathbf{y}$ is a skipping tree of $\mathbf{x}$. And we further know that vertices of $\mathbf{y}$ are all in color $j$ and the depth of $\mathbf{y}$ is $d_j + 1$. Therefore, $\mathbf{y}$ is a desirable skipping tree. And we have finished proving the result for $d_1 + \cdots + d_k = m + 1$. 
\par According to induction, this result holds for any $d_1, \ldots, d_n\ge 0$.
\end{proof}

Now we are ready to prove \pref{prop: tree-fat}.
\begin{proof}[Proof of \pref{prop: tree-fat}]
	First by choosing $\beta = \alpha^2/2$ in \pref{prop: covering-fat} and replacing $\alpha$ by $4\alpha$, we obtain
	$$\sup_{\bx} \calHsq(\calF_{1/n}, 5\alpha, n, \bx)\le \frakd(\calF, 4\alpha, \alpha^2/2)\cdot \log\left(\frac{2en}{\alpha^2}\right).$$
	Therefore, in order to prove \pref{prop: tree-fat}, we only need to verify 
	\begin{equation}\label{eq: hilbert-ball-frakd}
		\frakd(\calF_{1/n}, 4\alpha, \alpha^2/2) = \mathcal{O}\left(\frac{\log n}{\alpha^2}\right).
	\end{equation}
	We let tree $\bx'$ of depth $d'$ to be the largest tree shattered by $\calF_{1/n}$ at scale $(4\alpha, \alpha^2/2)$. Without loss of generality we assume $d'$ is an odd number and $d'= 2d + 1$. Then there exists a depth-$d'$ $([0, 1]\times [0, 1])$-valud tree $\bs'$ such that for any path $\by'\in \{0, 1\}^{d'}$, $s_t'(\by') = (s_t'(\by')[0], s_t'(\by')[1])$ with $s_t'(\by')[0] < s_t'(\by')[1]$, and for any $\by'\in \{0, 1\}^{d'}$, there exists $w^{\by'}\in B_2(1 - 1/n)$ such that 
	\begin{equation}\label{eq: condition-w-y-s}
		\left|\frac{1 + \langle w^{\by'}, x_t'(\by')\rangle}{2} - s_t'(\by')[y_t]\right| < \frac{\alpha^2}{2}\quad \text{and}\quad h\left(s_t'(\by')[0], s_t'(\by')[1]\right) > 4\alpha,\qquad \forall t\in [d'].
	\end{equation}
	We further construct a depth-$d'$ $[0, 1]$-valued tree $\bu'$ where for any $\by'\in \{0, 1\}^{d'}$ and $t\in [d']$, 
	$$s_t'(\by')[0] < u_t'(\by') < s_t'(\by')[1],\qquad h(u_t'(\by'), s_t'(\by')[0])\ge 2\alpha\quad \text{and}\quad  h(u_t'(\by'), s_t'(\by')[1])\ge 2\alpha.$$
	According to \pref{eq: condition-w-y-s}, we have for any $\by'\in \{0, 1\}^{d'}$ and $t\in [d']$,
	$$(2y_t-1)\cdot \left(\frac{1 + \langle w^{\by'}, x_t'(\by')\rangle}{2} - u_t'(\by')\right)\ge 0 \quad \text{ and }\quad h\left(\frac{1 + \langle w^{\by'}, x_t'(\by')\rangle}{2},  u_t'(\by')\right)\ge \alpha.$$
	We color the tree $\bx'$ with two colors according to $u_t(\by')$: for each node $x_t'(\by')$, if $u_t'(\by')\le 1/2$ we color it with color $1$, otherwise we color it with color $0$. According to \pref{lem: pigeonhole}, there exists a skipping tree of depth $d = (d' - 1)/2$ such that every node in this tree are of the same color. Without loss of generality, we assume that the skipping tree is in color $1$. In the following, we only consider nodes of $\bx'$ in this skipping tree, and also the corresponding nodes (along the same path) of $\bu'$. And we obtain a tree $\bx$ of depth $d$ and an $[0, 1/2]$-valued tree $\bar{\bu}$ of depth $d$ such that for any $\by\in \{0, 1\}^d$, there exists $w\in B_2(1 - 1/n)$ such that for any $t\in [d]$,
	$$(2y_t-1)\cdot \left(\frac{1 + \langle w, x_t(\by)\rangle}{2} - \baru_t(\by)\right)\ge 0 \quad \text{ and }\quad h\left(\frac{1 + \langle w, x_t(\by)\rangle}{2},  \baru_t(\by)\right)\ge \alpha.$$
	We next notice that since function $2\sqrt{x}$ and $\log x - 2\sqrt{x}$ are both monotonically increasing function over $[0, 1]$, for any $p, f\in [0, 1]$ we have
	$$|\log p - \log f|\ge 2|\sqrt{p} - \sqrt{f}|.$$
	Additionally, since for any $p\in [0, 1/2]$ and $f\in [0, 1]$, we always have 
	$$|\sqrt{p} - \sqrt{f}| = \frac{|p - f|}{\sqrt{p} + \sqrt{f}}\ge \frac{(\sqrt{2} + 1)|p-f|}{\sqrt{1-p} + \sqrt{1-f}} = (\sqrt{2} + 1)\cdot \left|\sqrt{1 - p} - \sqrt{1 - f}\right|,$$
	which implies that $|\sqrt{p} - \sqrt{f}|\ge h(f, p)/4$. Hence we obtain
	$$\log f - \log p\ge \frac{h(f, p)}{2}.$$
	By letting depth-$d$ $\RR$-valued tree $\bs$ to be $s_t(\by) = \log \baru_t(\by)$ for any path $\by\in\{0, 1\}^d$, we have that for $\by\in \{0, 1\}^d$, there exists $w\in B_2(1 - 1/n)$ such that for any $t\in [d]$,
	$$(2y_t-1)\cdot \left(\log \frac{1 + \langle w, x_t(\by)\rangle}{2} - s_t(\by)\right)\ge \frac{\alpha}{2}$$

	Since $x_t(\by)$ and $s_t(\by)$ only depend on $y_{1:t-1}$, by choosing $y_t = 1$, we obtain for some $w\in B_2(1 - 1/n)$,
	$$\log\frac{1 + \langle w, x_t(\by)\rangle}{2} - s_t(\by)\ge \frac{\alpha}{2} > 0,$$
	which implies that
	$$s_t(\by) < \log\frac{1 + \langle w, x_t(\by)\rangle}{2}\le \log\frac{1 + \|w\|_2\|x_t(\by)\|_2}{2}\le \log\frac{1 + 1}{2} = 0.$$
	Similarly by choosing $y_t = 0$, we obtain for some $w\in B_2(1 - 1/n)$,
	$$\log\frac{1 + \langle w, x_t(\by)\rangle}{2} - s_t(\by)\le -\frac{\alpha}{2} <  0,$$
	which implies that
	$$s_t(\by) > \log\frac{1 + \langle w, x_t(\by)\rangle}{2}\ge \log\frac{1 - \|w\|_2\|x_t(\by)\|_2}{2}\ge \log\frac{1 - (1-1/n)}{2} = -\log(2n).$$
	Hence we obtain that for any $t$, we always have $s_t(\by)\in (-\log(2n), 0)$. 
	\par Next, we will color the binary tree $\mathbf{x}$ with $\lceil \log(2n)\rceil$ number of colors $0, 1, \ldots, \lfloor \log(2n)\rfloor$: for $\by\in \{0, 1\}^d$, if $s_t(\by)\in [-k-1, -k)$, we will color vertex $x(\by)$ in color $k$. According to Lemma \ref{lem: pigeonhole}, there exists some $i$ such that there exists a skipping tree $\bv$ of depth $\bard\ge \frac{d-1}{\lceil \log(2n)\rceil}$ whose nodes are all colored in $k$. 
	\par We consider a sequence of nodes $v_1(\by)\to v_t(\by)\to\cdots\to v_{\bard}(\by)$ in the skipping tree $\bv$. Here we assume $v_{i+1}(\by)$ is the left child or descendant of the left child of $v_{i}$ if $y_i = 0$, or the right child or the descendant of the right child of $v_i$ if $y_i = 1$. We let 
	\begin{equation}\label{eq: y-x-i}
		v_i(\by) =  x_{i, 1}\to\cdots \to x_{i, l_i} = v_{i+1}(\by)
	\end{equation}
	to be the sequence of nodes in tree $\mathbf{x}$ from $v_i(\by)$ to $v_{i+1}(\by)$, where $x_{1, 2}$ is the right child of $x_{1, 1}$, and $x_{i, j}$ is a child of $x_{i, j-1}$ (since $v_{i+1}(\by)$ is a descendant of $v_i(\by)$, there must exist such a path). We consider the following sequence of nodes in tree $\mathbf{x}$: 
	\begin{equation}\label{eq: epsilon-tilde}
		x_{1, 1}\to x_{1, 2}\cdots \to x_{1, l_1}\to  x_{2, 2}\to\cdots\to  x_{2, l_2}\to x_{3, 2}\to\cdots\to  x_{3, l_3}\to\cdots  x_{\bard, 2}\to\cdots  x_{\bard, l_{\bard}}.
	\end{equation}
	We define length-$d$ $\{0, 1\}$-valued path 
	$$\tilde{\by} = (\ty_{1, 1}, \ty_{1, 2}, \cdots, \ty_{1, l_1}, \ty_{2, 1}, \cdots \ty_{2, l_2-1}, \ty_{3, 1}, \cdots, \ty_{3, l_3-1}, \cdots, \ty_{\bard, 1}, \cdots, \ty_{\bard, l_{\bard}-1}, y_{\ty+1, 1}, \cdots, y_{\bard+1, l_{\bard+1} - 1}),$$
	where $\tilde{y}_{i, j}$ is chosen to be 1 if $ x_{i, j}$ is the right child of $ x_{i, j-1}$ and be $0$ if $ x_{i, j}$ is the left child of $ x_{i, j-1}$, and $y_{\bard + 1, 1}, \cdots, y_{\bard + 1, l_{\bard + 1} - 1}$ can be arbitrarily chosen with $l_{\bard + 1} - 1 = d - l_1 - \cdots - l_{\bard} + \bard$. Then according to the construction of this path we have 
	$$\ty_{i, 1} = y_i, \quad \forall 1\le i\le \bard.$$
	Suppose the vertices we meet in tree $\bs$ along path $\ty$ at the same depth as $x_{i, j}$ to be $s_{i, j}(\tilde{\by})$. Then according to our assumption, there exists some $w\in B_2(1 - 1/n)$ such that
	\begin{equation}\label{eq: s-epsilon}
		(2\ty_{i, j}-1)\cdot \left(\log\frac{1 + \langle w, x_{i, j}\rangle}{2} - s_{i, j}(\tilde{\by})\right)\ge \frac{\alpha}{2}, \qquad \forall 1\le i\le \bard, 1\le j\le l_i.
	\end{equation}
	We further define depth-$\bard$ $\RR$-tree $\bu = (u_1, \ldots, u_{\bard})$ as
	$$u_i(\by) = s_{i, 1}(\tilde{\by}).$$
	Choosing $j = 1$ in \pref{eq: s-epsilon} and notice that $v_i(\by) = x_{i, 1}$ from \pref{eq: y-x-i} and $y_i = \tilde{y}_{i, 1}$ from \pref{eq: epsilon-tilde}, we obtain
	$$(2y_i - 1)\left(\log\frac{1 + \langle w, v_{i}(\by)\rangle}{2} - u_{i}(\by)\right)\ge\frac{\alpha}{2}, \quad \forall 1\le i\le \bard.$$
	According to our coloring and the definition of $s_{i, j}(\tilde{\by})$, we know that $u_i(\by) = s_{i, 1}(\tilde{\by})\in [- k - 1, - k)$ holds for any $1\le i\le \bard$. Therefore, for any $\by\in \{0, 1\}^{\bard}$, there exists $w\in B_2(1 - 1/n)$ such that for any $i\in [\bard]$,
	\begin{equation}\label{eq: u-i-epsilon}
		(2y_i - 1)\cdot \left(\log\frac{1 + \langle w, v_i(\by)\rangle}{2} - u_i(\by)\right)\ge \frac{\alpha}{2}\quad \text{ and }\quad u_i(\by)\in [-k-1, -k).
	\end{equation}
	The above inequality is equivalent to for any $\by\in \{0, 1\}^{\bard}$, there exists $w\in B_2(1 - 1/n)$ such that for any $i\in [\bard]$,
	\begin{equation}\label{eq: w-y-t}
		\langle w, v_t(\by)\rangle\ge 2e^{u_t(\by)}e^{\alpha/2} - 1 \quad \text{if }y_t = 1,\quad \text{and}\quad \langle w, v_t(\by)\rangle\le 2e^{u_t(\by)}e^{-\alpha/2} - 1 \quad \text{if }y_t = 0.
	\end{equation}
	For any given path $\by\in \{0, 1\}^{\bard}$, we call the $w\in B_2(1-1/n)$ which satisfies the above inequalities as $w[\by]$. We use $v_0 = v_1(\by)$ to denote the root of the tree $\mathbf{y}$. Then for any path $\by = (y_1, \ldots, y_{\bard})$ with $y_1 = 0$, i.e. turn to left subtree in the first step, according to \pref{eq: w-y-t} we have
	$$\langle w[\by], v_0\rangle = \langle w[\by], v_1(\by)\rangle\le 2e^{u_1(\by)}e^{-\alpha/2} - 1\le 2e^{-k} - 1,$$
	where in the last inequality we uses the second inequality in \pref{eq: u-i-epsilon}. 
	
	In the following, for every vector $v$, we decompose it into the parallel component and perpendicular component with respect to vector $v_0$: $v = v^\perp + v^\parallel$, where $v^\parallel\parallel v_0$ and $v^\perp\perp v_0$. Then we have
	$$\|w[\by]^\parallel\|_2\|v_0\|_2 = |\langle w[\by], v_0\rangle| = 1 - 2e^{-k}.$$
	Noticing that $\|v_0\|_2, \|w[\by]^\parallel\|_2\le 1$, we will have $\|v_0\|_2, \|w[\by]^\parallel\|_2\ge 1 - 2e^{-k}$, hence
	\begin{equation}\label{eq: norm-w}
		\|w[\by]^\parallel + v_0\|_2 \le 1 - (1 - 2e^{-k}) = 2e^{-k} \quad \text{and}\quad \|w[\by]^\perp\|_2\le \sqrt{1 - (1 - 2e^{-k})^2} \le 2e^{-k/2}.
	\end{equation}
	Next, we consider any node $v_t(\by)$ on the left subtree of $\mathbf{y}$, where we require the path $\by$ to the node satisfies $y_1 = 0$. By letting $y_t = 0$ (since $v_t(\by)$ does not depends on $y_t$ so we can assign arbitrary value of $y_t$ to obtain some properties of $v_t(\by)$), according to \eqref{eq: w-y-t} and the second inequality of \pref{eq: u-i-epsilon}, we obtain
	$$\langle w[\by], v_t(\by)\rangle \le 2e^{u_t(\by)}e^{-\alpha/2} - 1\le 2e^{-k} - 1,$$
	which implies
	$$\|w[\by] + v_t(\by)\|_2\le \sqrt{\|w[\by]\|_2^2 + \|v_t(\by)\|_2^2 + 2\langle w[\by], v_t(\by)\rangle}\le \sqrt{2 + 2(2e^{-k} - 1)} = 2e^{-k/2},$$
	Choosing $t = 1$ in the above inequality we obtain $\|w[\by] + v_0\|_2\le 2e^{-k/2}$. These two inequalities together indicates that
	$$\|v_t(\by) - v_0\|_2\le 4e^{-k/2}.$$
	Hence we have,
	\begin{equation}\label{eq: norm-y}
		\|v_t(\by)^\perp\|_2\le \|v_t(\by) - v_0\|_2\le 4e^{-k/2}.
	\end{equation}
	Next, we decompose the inner product into the sum of inner product of parallel components and perpendicular components:
	$$\langle w[\by], v_t(\by)\rangle = \langle w[\by]^\parallel, v_t(\by)^\parallel\rangle + \langle w[\by]^\perp, v_t(\by)^\perp\rangle.$$
	Noticing $\|w[\by]^\perp\|_2\le 2e^{-k/2}$ and $\|v_t(\by)^\perp\|_2\le 4e^{-k/2}$ from \pref{eq: norm-w} and \pref{eq: norm-y}, we obtain that
	$$\langle w[\by]^\parallel, v_t(\by)^\parallel\rangle = \langle w[\by], v_t(\by)\rangle - \langle w[\by]^\perp, v_t(\by)^\perp\rangle\le 2e^{-k} - 1 + 2e^{-k/2}\cdot 4e^{-k/2} = 10e^{-k} - 1.$$
	Since $\|w[\by]^\parallel\|_2\le 1$ and $\|v_t(\by)^\parallel\|_2\le 1$, we have $\|w[\by]^\parallel\|_2, \|v_t(\by)^\parallel\|_2\ge 1 - 10e^{-k}$. Hence,
	$$\|w[\by]^\parallel + v_t(\by)^\parallel\|_2 \le 1 - (1 - 10e^{-k}) = 10e^{-k},$$
	This inequality together with the first inequality of \pref{eq: norm-w} indicates that
	$$\|v_2(\by)^\parallel - v_0\|_2\le 12e^{-k}.$$

	Finally, we construct a tree $\bz$ of depth $\bard - 1$ shattered by the following function class $\calG$ at scale $1/(20e)$ (definition of the shattering in the sense of \cite{rakhlin2015online}), hence according to \cite[Page 67-68]{rakhlin24} we have an upper bound on $\bard$.
	\begin{equation}\label{eq: def-calG}
		\calG = \{f|f(x) = \langle w, x\rangle, w, x\in B_2(1)\}
	\end{equation}
	For any $\by\in \{0, 1\}^{\bard - 1}$, we let
	\begin{equation}\label{eq: def-z}
		z_t(\by) = \frac{1}{5}e^{k/2}v_{t+1}((0, \by))^\perp + \frac{1}{5}v_{t+1}((0, \by))^\parallel, \quad \forall 1\le t\le \bard - 1,
	\end{equation}
	where we use $(0, \by)$ to denote the path of length $\bard$ whose $t$-th element equals to $y_{t-1}$ for $t\ge 2$ and the first element equals to $0$. Then according to \pref{eq: norm-y},  for any path $\by$ we have 
	$$\|z_t(\by)\|_2\le \frac{1}{5}e^{k/2}\|v_{t+1}((0, \by))^\perp\|_2 + \frac{1}{5}\|v_{t+1}((0, \by))^\parallel\|_2\le \frac{4}{5} + \frac{1}{5} = 1,$$
	which implies that $z_t(\by)\in B_2(1)$. We further notice from \pref{eq: norm-w} that
	$$\|w((0, \by))^\parallel + v_0\|_2 \le 2e^{-k}\quad\text{and} \quad \|w((0, \by))^\perp\|_2\le 2e^{-k/2}.$$
	Hence by choosing
	\begin{equation}\label{eq: def-w-bar}
		\bar{w}(\by) = \frac{1}{4}e^k \left(w((0, \by))^\parallel + v_0\right) + \frac{1}{4}e^{k/2}w((0, \by))^\perp,
	\end{equation}
	we have
	$$\|\bar{w}(\by)\|_2\le \frac{1}{4}e^k \left\|(w((0, \by))^\parallel + v_0\right\|_2 + \frac{1}{4}e^{k/2}\left\|w((0, \by))^\perp\right\|_2\le \frac{2}{4} + \frac{2}{4} = 1,$$
	which implies that $\bar{w}(\by)\in B_2(1)$. According to our choice of $\bar{w}(\by)$ in \pref{eq: def-w-bar} and $z_t(\by)$ in \pref{eq: def-z}, we have
	\begin{align*}
		\langle \bar{w}(\by), z_t(\by)\rangle & = \langle \bar{w}(\by)^\parallel, z_t(\by)^\parallel\rangle + \langle \bar{w}(\by)^\perp, z_t(\by)^\perp\rangle\\
		& = \frac{1}{20}e^{k}\left\langle w((0, \by))^\parallel + v_0, v_{t+1}((0, \by))^\parallel\right\rangle + \frac{1}{20}e^{k}\left\langle w((0, \by))^\perp, v_{t+1}((0, \by))^\perp\right\rangle\\
		& = \frac{1}{20}e^{k}\cdot \left(\langle w((0, \by)) , v_{t+1}((0, \by))\rangle + \langle v_0, v_{t+1}((0, \by))\rangle\right).
	\end{align*}
	We construct another $(\bard-1)$-depth $\RR$-valued tree $\bar{\bs}$ as: for any path $\by\in \{0, 1\}^{\bard - 1}$,
	$$\bar{s}_t(\by) = \frac{1}{20}e^{k}e^{u_{t+1}((0, \by))}\left(e^{\alpha/2} + e^{-\alpha/2}\right) + \langle v_0, v_{t+1}((0, \by))\rangle - \frac{1}{20}e^k.$$
	The above defined $\bar{\bs}$ is a tree since $\bu$ and $\bv$ are both trees. When $y_t = 1$, according to the first inequality of \pref{eq: w-y-t} and also $u_{t+1}((0, \by))\ge -k-1$ according to the second inequality of \pref{eq: u-i-epsilon}, we have
	\begin{align*}
		\langle \bar{w}(\by), z_t(\by)\rangle - \bar{s}_t(\by) & = \frac{1}{20}e^{k}\cdot \left(\langle w((0, \by)) , v_{t+1}((0, \by))\rangle - e^{u_{t+1}((0, \by))}\left(e^{\alpha/2} + e^{-\alpha/2}\right) + 1\right)\\
		& \ge \frac{1}{20}e^k\cdot\left(2e^{u_{t+1}((0, \by))}e^{\alpha/2} - 1 - e^{u_{t+1}((0, \by))}\left(e^{\alpha/2} + e^{-\alpha/2}\right) + 1\right)\\
		& = \frac{1}{20}e^{k}e^{u_{t+1}((0, \by))}(e^{\alpha/2} - e^{-\alpha/2})\\
		&\ge \frac{1}{20}e^{k}e^{u_{t+1}((0, \by))}\alpha\ge \frac{1}{20}e^{k}e^{-k-1}\alpha = \frac{1}{20e}\alpha,
	\end{align*}
	where in the second inequality we use the fact that $e^{\alpha/2} - e^{-\alpha/2}\ge \alpha$. And when $y_t = 0$, we have 
	\begin{align*}
		\langle \bar{w}(\by), z_t(\by)\rangle - \bar{s}_t(\by) & = \frac{1}{20}e^{k}\cdot \left(\langle w((0, \by)) , v_{t+1}((0, \by))\rangle - e^{u_{t+1}((0, \by))}\left(e^{\alpha/2} + e^{-\alpha/2}\right) + 1\right)\\
		& \le \frac{1}{20}e^k\cdot\left(2e^{u_{t+1}((0, \by))}e^{-\alpha/2} - 1 - e^{u_{t+1}((0, \by))}\left(e^{\alpha/2} + e^{-\alpha/2}\right) + 1\right)\\
		& = - \frac{1}{20}e^{k}e^{u_{t+1}((0, \by))}(e^{\alpha/2} - e^{-\alpha/2})\\
		&\le - \frac{1}{20}e^{k}e^{u_{t+1}((0, \by))}\alpha\le - \frac{1}{20}e^{k}e^{-k-1}\alpha = - \frac{1}{20e}\alpha.
	\end{align*}
	Therefore, tree $\mathbf{z}\in B_2(1)$ is shattered by function class $\calG$ (defined in \pref{eq: def-calG}) at scale $1/(20e)\alpha$.
	
	According to \cite[Page 67-68]{rakhlin24}, the sequential fat shattering dimension of $\calG$ at scale $\alpha$ is upper bounded by $16/\alpha^2$. Hence we have
	$$\bard - 1\le \frac{16}{(1/(20e)\alpha)^2} = \frac{6400e^2}{\alpha^2}.$$
	This inequality, together with the fact that $\bard\ge \frac{d-1}{\lceil \log(2n)\rceil}$, implies that
	$$d\le 1 + \lceil \log(2n)\rceil\cdot \left(1 + \frac{6400e^2}{\alpha^2}\right) = \mathcal{O}\left(\frac{\log n}{\alpha^2}\right).$$
	Therefore, we have
	$$\frakd(\calF_{1/n}, 4\alpha, \alpha^2/2) = \mathcal{O}\left(\frac{\log n}{\alpha^2}\right),$$
	which verifies \pref{eq: hilbert-ball-frakd}.
\end{proof}

\section{Renewal Process and Hardness through Sequential Square-root Entropy}\label{sec: renewal-app}

We consider the following class of renewal process, originally introduced in \cite{csiszar1996redundancy}.
\begin{definition}[Renewal Process Class \cite{csiszar1996redundancy}]\label{def: renewal}
	This class $\calQ$ is defined over the alphabet $\calY = \{0, 1\}$ and parameterized by a distribution $p\in \Delta(\mathbb{Z}_+)$. Given $p$, we sample  $T_i \stackrel{iid}{\sim} p $ and set $y_t=1$ if $t=T_1+\cdots + T_i$ for some $i\ge1$ and otherwise $y_t=0$.
\end{definition}
For this class $\calQ$ the work~\cite{csiszar1996redundancy} established that log-loss regret is $\Theta(\sqrt{n})$. Their proof leveraged sophisticated estimates on the partition number by Hardy and Ramanujan. Unfortunately, as we show in this appendix, the entropic bounds that we developed in this work, as well as those that were proposed before, are not able to yield correct upper bound on regret. 

Specifically, we will verify that the sequential square-root entropy (defined in \pref{def: sequential-covering}), and also the sequential log entropy defined in \cite{cesa1999minimax,cesa2006prediction} are both $\Omega(n)$, no matter what scale we choose. Therefore, by simply applying \pref{thm: general-upper-bound} or the entropy bound in \cite{cesa1999minimax,cesa2006prediction} will only give a vacuous bound $O(n)$ on regret.

\begin{proposition}
	For any $0 < \alpha < 1/6$, we have $\calHsq(\calQ, \alpha, n)\ge n$. As for the log entropy (entropy with respect to distance \pref{eq: CBL-d}) defined in \cite{cesa1999minimax,cesa2006prediction}, we have $\calH_{\sf log}(\calQ, \alpha, n)\ge (1-\log 2)n - o(n)$.
\end{proposition}
\begin{proof}
	For any $\bepsilon\in \{-1, 1\}^n$, we construct a distribution $p^\bepsilon\in \Delta(\mathbb{Z}_+)$ as 
	$$p^\bepsilon(t) = \prod_{i=1}^{t-1} \left(\frac{1}{2} + 3\epsilon_t\cdot \alpha\right) - \prod_{i=1}^t \left(\frac{1}{2} + 3\epsilon_t\cdot \alpha\right).$$
	It is easy to see that $p^\bepsilon$ is a distribution on $\mathbb{Z}_+$. We let $q^\bepsilon$ to be the distribution in $\calQ$ which is parametrized by $p^\bepsilon$. Then we can calculate that with $\by^0 = (0, 0, \cdots, 0)\in \{0, 1\}^n$,
	$$q_t^\bepsilon(0\mid \by^0) = \frac{1}{2} + 3\epsilon_t\cdot \alpha,\quad \forall t\in [n].$$
	We first lower bound the sequential square-root entropy (defined in \pref{def: sequential-covering}). Suppoose $\calV$ is a finite cover of $\calQ$ at scale $\alpha$. Then for $\bepsilon\in \{-1, 1\}^n$, there exists $\bv^\bepsilon\in \calV$ such that 
	$$\max_{\bw}\max_{y\in \{0, 1\}}\max_{t\in [n]} \left|\sqrt{v_t^\bepsilon(y_t\mid \bw)} - \sqrt{q_t^\bepsilon(y_t\mid \bw)}\right|\le \alpha,$$
	which implies that 
	$$\max_{t\in [n]}\left|\sqrt{v_t^\bepsilon(0\mid \by^0)} - \sqrt{q_t^\bepsilon(0\mid \by^0)}\right|\le \alpha.$$
	If there exists $\bepsilon$ and $\bepsilon'$ such that $\bv^\bepsilon = \bv^{\bepsilon'}$, then 
	$$\max_{t\in [n]}\left|\sqrt{q_t^{\bepsilon'}(0\mid \by^0)} - \sqrt{q_t^\bepsilon(0\mid \by^0)}\right|\le 2\alpha.$$
	However, if $\epsilon_t\neq \epsilon_t'$, then 
	$$\left|\sqrt{q_t^{\bepsilon'}(0\mid \by^0)} - \sqrt{q_t^\bepsilon(0\mid \by^0)}\right| = \sqrt{\frac{1}{2} + 3\alpha} - \sqrt{\frac{1}{2} - 3\alpha} > 2\alpha,$$
	leading to contradiction. Therefore, for any $\bepsilon\neq \bepsilon$, we have $v^\bepsilon\neq v^{\bepsilon'}$. This implies that $|\calV|\ge 2^n$, hence 
	$$\calH(\calQ, \alpha, n)\ge n.$$

Next we lower bound the log-entropy $\calH_{\sf log}$, which is the entropy with respect to the distance defined in \pref{eq: CBL-d}. Suppose $\calV$ is a finite cover of $\calQ$ at scale $\alpha$. We first define set $E \subseteq \{-1, 1\}^n$ such that for any two distinct items $\bepsilon, \bepsilon'\in E$, we have 
	\begin{equation}\label{eq: hamming}
		\sum_{t=1}^n \mathbb{I}[\epsilon_t\neq \epsilon_t']\ge \frac{n}{4}.
	\end{equation}
	According to the lower bound of packing number under Hamming distances (see \cite[Theorem 27.5]{polyanskiy2024information}), there exists such set $E$ which satisfies 
	$$\log|E| \ge (1-\log 2)n - o(n).$$
	Next, since $\calV$ is a covering of $\calQ$, for any $\bepsilon\in E$, there exists $\bv^\bepsilon\in \calV$ such that 
	$$\sum_{t=1}^n \sup_{\by} (\log v_t^\bepsilon(y_t\mid \by) - \log q_t^\bepsilon(y_t\mid \by))^2\le n\alpha^2,$$
	which implies that 
	$$\sum_{t=1}^n (\log v_t^\bepsilon(0\mid \by^0) - \log q_t^\bepsilon(0\mid \by^0))^2\le n\alpha^2.$$
	If there exists $\bepsilon$ and $\bepsilon'$ such that $\bv^\bepsilon = \bv^{\bepsilon'}$, then 
	\begin{equation}\label{eq: log-hamming-upper-bound}
		\sum_{t=1}^n (\log q_t^{\bepsilon}(0\mid \by^0) - \log q_t^{\bepsilon'}(0\mid \by^0))^2\le 4n\alpha^2.
	\end{equation}
	However, if $\epsilon_t\neq \epsilon_t'$, then 
	$$|\log q_t^\bepsilon(0\mid \by^0) - \log q_t^{\bepsilon'}(0\mid \by^0)|\ge \left|\log\frac{1 - 6\alpha}{1 + 6\alpha}\right| > 6\alpha,$$
	which implies that 
	$$\sum_{t=1}^n (\log q_t^{\bepsilon}(0\mid \by^0) - \log q_t^{\bepsilon'}(0\mid \by^0))^2\ge 36\alpha^2\cdot \sum_{t=1}^n \mathbb{I}[\epsilon_t\neq \epsilon_t'] > 9n\alpha^2,$$
	where the last inequality follows from the construction of set $E$, i.e. \pref{eq: hamming}. This contradicts to \pref{eq: log-hamming-upper-bound}. Hence for any $\bepsilon, \bepsilon'\in E$, we have $\bv^\bepsilon\neq \bv^{\bepsilon'}$. This implies that $|\calV|\ge |E|$, hence 
	$$\calH(\calQ, \alpha, n)\ge \log |E|\ge (1-\log 2)n - o(n).$$
\end{proof}

We see that the root cause of entropies being $\Omega(n)$ is the same: both definitions of $\calHsq$ and $\calH_{\sf log}$ in \pref{def: sequential-covering} and \eqref{eq: CBL-d} take supremum over the ``true path'' $\bf{y}$ on the tree. In the example above, this corresponds to simply taking a path on the very left of the tree. The process class is so rich that already on this left-most path  the entropy is $\Omega(n)$. However, this should not concern log-loss prediction as this left-most path would not happen too-often, unless $\bf p$ in~\eqref{eq: r-n-formula} places all mass on all-0 input, in which case the $\calR_n(\calQ, \bf p) = 0$. Searching for the correct definition of entropy to handle this class is left to future work.

\end{document}